\documentclass[journal,twocolumn,10pt]{IEEEtran}
\usepackage{amsmath,epsfig}
\usepackage{amssymb}
\usepackage{booktabs}
\usepackage{amsthm}
\usepackage{graphicx}
\usepackage{caption}
\usepackage{subcaption}
\usepackage{rotating}
\usepackage{floatflt} 
\usepackage{paralist} 
\usepackage{algorithm} 
\usepackage{xcolor}
\usepackage{color}
\usepackage{epstopdf}
\usepackage[normalem]{ulem}
\usepackage{float}
\usepackage{todonotes}
\usepackage{mathtools}
\usepackage{fancyhdr}
\usepackage[version=3]{mhchem}

\theoremstyle{definition}
\newtheorem{defn}{Definition}
\newcommand{\mypar}[1]{{\bf #1.}}

\newtheorem{myThm}{Theorem}
\newtheorem{myCorollary}{Corollary}
\newtheorem{myConj}{Conjecture}

\newcommand{\R}{\ensuremath{\mathbb{R}}}

\DeclareMathOperator{\Id}{I}

\DeclarePairedDelimiter\ceil{\lceil}{\rceil}

\def\a{\mathbf{a}}

\def\x{\mathbf{x}}
\def\y{\mathbf{y}}
\def\d{\mathbf{d}}

\def\e{\mathbf{e}}

\def\t{\mathbf{t}}
\def\vv{\mathbf{v}}
\def\w{\mathbf{w}}

\def\N{\mathcal{N}}
\def\V{\mathcal{V}}
\def\E{\mathcal{E}}
\def\M{\mathcal{M}}

\DeclareMathOperator{\LL}{L}

\DeclareMathOperator{\BL}{BL}
\DeclareMathOperator{\PL}{PL}
\DeclareMathOperator{\PC}{PC}
\DeclareMathOperator{\PPL}{PPL}
\DeclareMathOperator{\PBL}{PBL}
\DeclareMathOperator{\STV}{S}
\DeclareMathOperator{\Adj}{A}

\DeclareMathOperator{\Cc}{C}
\DeclareMathOperator{\D}{D}

\DeclareMathOperator{\F}{F}

\DeclareMathOperator{\Pj}{P}

\DeclareMathOperator{\RR}{R}

\DeclareMathOperator{\Vm}{V}
\DeclareMathOperator{\X}{X}
\DeclareMathOperator{\Um}{U}

\DeclareMathOperator{\W}{W}
\DeclareMathOperator{\Ss}{S}

\begin{document}
\title{ Signal Representations on Graphs: Tools and Applications }
\author{Siheng~Chen,~\IEEEmembership{Student~Member,~IEEE}, Rohan~Varma,~\IEEEmembership{Student~Member,~IEEE}, Aarti Singh,   Jelena~Kova\v{c}evi\'c,~\IEEEmembership{Fellow,~IEEE}
  \thanks{S. Chen and R. Varma are with the Department of Electrical and Computer
    Engineering, Carnegie Mellon University, Pittsburgh, PA, 15213
    USA. Email: sihengc@andrew.cmu.edu, rohanv@andrew.cmu.edu. A. Singh is with the Department of Machine Learning, Carnegie Mellon University, Pittsburgh, PA. Email: aarti@cs.cmu.edu. J. Kova\v{c}evi\'c is with the
      Departments of Electrical and
      Computer Engineering and Biomedical Engineering, Carnegie Mellon University, Pittsburgh,
      PA. Email: jelenak@cmu.edu.  }
  \thanks{The authors gratefully acknowledge support from the NSF through awards
1130616,1421919,  the University Transportation Center grant (DTRT12-GUTC11)
from the US Department of Transportation, and the CMU Carnegie
Institute of Technology Infrastructure Award.}
}
 \maketitle

\tableofcontents

\begin{abstract}
We present a framework for representing and modeling data on graphs. Based on this framework, we study three typical classes of graph signals: smooth graph signals, piecewise-constant graph signals, and piecewise-smooth graph signals. For each class, we provide an explicit definition of the graph signals and construct a
corresponding graph dictionary with desirable properties. We then study how such graph dictionary works in two
standard tasks: approximation and sampling followed
with recovery, both from theoretical as well as algorithmic perspectives. Finally, for each class, we present a case study of a real-world problem by using the proposed methodology.
\end{abstract}

\section{Introduction}
\label{sec:intro}
Signal processing on graphs is a framework that extends classical discrete signal processing to
signals with an underlying complex and irregular
structure. The framework models that underlying structure by a graph and signals by graph signals, generalizing concepts and tools from classical discrete signal
processing to graph signal processing.  Recent work includes graph-based transforms~\cite{SandryhailaM:13,HammondVG:11,NarangSO:10}, sampling and interpolation
on graphs~\cite{Pesenson:08, AnisGO:14, ChenVSK:15a}, graph signal recovery~\cite{ChenSMK:14, WangLG:14, WangLG:15a, ChenVSK:15}, uncertainty principle on graphs~\cite{AgaskarL:13, TsitsveroBL:15},  graph
dictionary learning~\cite{DongTFV:14, ThanouSF:14},  community detection~\cite{Tremblay:14, ChenH:15, ChenH:15a} , and many others.

In this paper, we consider the signal representations on graphs. Signal representation is one of the most fundamental tasks in our discipline. For example, in classical signal processing, we use the Fourier basis to represent the sine waves; we use the wavelet  basis to represent smooth signals with transition changes. Signal representation is highly related to approximation, compression, denoising, inpainting, detection, and localization~\cite{Mallat:09, VetterliKG:12}.  Previous works along those lines consider representations based on the graph Fourier domain, which emphasize the smoothness and global behavior of a graph signal~\cite{ShumanRV:15}, as well as  the representations based on the graph vertex domain, which emphasize the connectivity and localization of a graph signal~\cite{GavishNC:10, CrovellaK:13}.

We start by proposing  a representation-based  framework, which provides a recipe to model real-world data on graphs. Based on this framework, we study three typical classes of graph signals: smooth graph signals, piecewise-constant graph signals, and piecewise-smooth graph signals. For each class, we provide an explicit definition for the graph signals and construct a corresponding graph dictionary with desirable properties. We then study how the proposed graph dictionary works in two standard tasks: approximation and sampling followed with recovery, both from theoretical as well as algorithmic perspectives. Finally, for each class, we present a case study of a real-world problem by using the proposed methodology.

\mypar{Contributions}  The main contribution of this
  paper is to build a novel and unified framework to analyze graph signals. The framework provides a general solution allowing us to study real-world data on graphs. Based on the framework, we study three typical classes of graph signals:
  
 \begin{itemize}
 \item Smooth graph signals. We explicitly define the smoothness criterion and construct corresponding representation dictionaries. We propose a generalized uncertainty principle on graphs and study the localization phenomenon of graph Fourier bases. We then investigate how the proposed graph dictionary works in approximation and sampling followed with recovery. Finally, we demonstrate a case study on a co-authorship network.
 
  \item Piecewise-constant graph signals. We explicitly define piecewise-constant graph signals and construct the multiresolution local sets, a local-set-based piecewise-constant  dictionary and local-set-based piecewise-constant wavelet basis, which provide a multiresolution analysis on graphs and promote sparsity for  such graph signals.  We then investigate how the proposed local-set-based piecewise-constant  dictionary works for approximation and sampling followed with recovery. Finally, we demonstrate a case study on epidemics processes.

   \item  Piecewise-smooth graph signals. We explicitly define  piecewise-smooth graph signals and construct a local-set-based piecewise-smooth  dictionary, which promotes sparsity for  such graph signals.  We then investigate how the proposed local-set-based piecewise-smooth dictionary works in approximation. Finally, we demonstrate a case study on environmental change detection.
\end{itemize}  

\mypar{Outline of the paper} Section~\ref{sec:spg} introduces and reviews the background on signal processing on graphs; Section~\ref{sec:foundations} proposes a representation-based framework to study graph signals, which lays the foundation for this paper; Sections~\ref{sec:R_Smooth},~\ref{sec:R_PC}, and~\ref{sec:R_PS} present representations of smooth, piecewise constant, and piecewise smooth graph signals, respectively, including their effectiveness in approximation and sampling followed by recovery, as well as validation on three different real-world problems: co-authorship network, epidemics processes, and environmental change detection. Section~\ref{sec:conclusions} concludes the paper.

\section{Signal Processing on Graphs}
\label{sec:spg}
Let $\mathcal{G} = (\V,\E, \W)$  be a directed, irregular and weighted graph, where $\V = \{v_i \}_{i=1}^N$ is the set of nodes, $\E$ is the set of weighted edges, and $\W \in \R^{N \times N}$ is the weighted adjacency matrix, whose element $\W_{i,j}$ measures the underlying relation between the $i$th and the $j$th nodes. Let $\d \in \R^N$ be a degree vector, where $d_i = \sum_j \W_{i,j}$. Given a fixed ordering of nodes, we assign a signal coefficient to each node; a~\emph{graph signal} is then defined as a vector,
$$
	\x = [x_1,x_2,\cdots,x_N]^T \in \R^N,
$$
with $x_n$ the signal coefficient corresponding to the node $v_n$. 

To represent the graph structure by a matrix, two basic approaches have been considered. The first one is based on algebraic signal processing~\cite{PueschelM:08}, and is the one we follow. We use the graph shift operator $\Adj \in \R^{N \times N}$ as a graph representation, which is an elementary filtering operation that replaces a signal coefficient at a node with a weighted linear combination of coefficients at its neighboring nodes. Some common choices of a graph shift are weighted adjacency matrix $\W$, normalized adjacency matrix $\W_{\rm norm} = {\rm diag}(\d)^{-\frac{1}{2}} \W {\rm diag}(\d)^{-\frac{1}{2}}$, and transition matrix $\Pj = {\rm diag}(\d)^{-1} \W$~\cite{ChenSMK:13}.  The second one is based on the spectral graph theory~\cite{Chung:96}, where the graph Laplacian matrix $\LL \in \R^{N \times N}$ is used  as a graph representation, which is a second-order difference operator on graphs. Some common choices of a graph Laplacian matrix are the unnormalized Laplacian  ${\rm diag}(\d) - \W$, the normalized Laplacian  $\Id - {\rm diag}(\d)^{-\frac{1}{2}} \W {\rm diag}(\d)^{-\frac{1}{2}}$, and the transition Laplacian $\Id - {\rm diag}(\d)^{-1} \W $. 

\begin{table}[htbp]
  \footnotesize
  \begin{center}
    \begin{tabular}{@{}lll@{}}
      \toprule
     & Graph shift  $\Adj$ &  Graph Laplacian $\LL$ \\
      \midrule \addlinespace[1mm]
 Unnormalized &   $\W$ & ${\rm diag}(\d) - \W$ \\

Normalized & $\W_{\rm norm} = {\rm diag}(\d)^{-\frac{1}{2}} \W {\rm diag}(\d)^{-\frac{1}{2}}$ & $\Id - \W_{\rm norm}$ \\

Transition  & $\Pj = {\rm diag}(\d)^{-1} \W$  & $\Id - \Pj$ \\

\bottomrule
\end{tabular} 
\caption{\label{tab:graph_structure_matrix}. Graph structure matrix $\RR$ can be either a graph shift $\Adj $ or a graph Laplacian. $\LL$. }
\end{center}
\end{table}

A graph shift emphasizes the similarity while a graph Laplacian matrix emphasizes the difference between each pair of nodes. The graph shift and graph Laplacian matrix often appear in pairs; see Table~\ref{tab:graph_structure_matrix}. Another overview is presented in~\cite{AnisAO:15}. We use $\RR \in \R^{N \times N}$ to represent a graph structure matrix, which can be either a graph shift or a graph Laplacian matrix.  Based on this graph structure matrix, we are able to generalize many tools from traditional signal processing to graphs, including filtering~\cite{NarangO:12, ChenCRBGK:13}, Fourier transform~\cite{ShumanNFOV:13, SandryhailaM:13}, wavelet transforms~\cite{CrovellaK:03,HammondVG:11}, and many others.  We now briefly review the graph Fourier transform.

The graph Fourier basis generalizes the traditional Fourier basis and is used to represent the graph signal in the graph spectral domain.  The~\textit{graph Fourier basis} $\Vm \in \R^{N \times N}$ is defined to be the eigenvector matrix of $\RR$, that is,
\begin{equation*}
	\RR =  \Vm \Lambda \Vm^{-1},
\end{equation*}
where the $i$th column vector of $\Vm$ is the~\textit{graph Fourier basis vector} $\vv_i$ corresponding to the eigenvalue $\lambda_i$ as the $i$th diagonal element in $\Lambda$. The graph Fourier transform \rm{of $\x \in \mathbb{R}^N$} is 
$
	\widehat{\x} = \Um \x,
$
where $\Um = \Vm^{-1}$ is called the graph Fourier transform matrix. 
When $\RR$ is symmetric, then $\Um = \Vm^T$ is orthornormal; the graph Fourier basis vector $\vv_i$ is the $i$th row vector of $\Um$. The~\textit{inverse graph Fourier transform} is
$
	\x =  \Vm  \widehat{\x}.
$
The vector $\widehat{\x}$ represents the frequency coefficients corresponding to the graph signal $\x$, and the graph Fourier basis vectors can be regarded as graph frequency components. In this paper, we use $\Vm$, $\Um$ to denote the inverse graph Fourier transform matrix and graph Fourier transform matrix for a general graph structure matrix, which can be an adjacency matrix, a graph Laplacian, or a transition matrix. When we emphasize that $\Vm$ and $\Um$ is generated from a certain graph structure matrix, we add a subscript. For example, $\Um_{\W}$ is the graph Fourier transform matrix of the weighted adjacency matrix $\W$.

The ordering of graph Fourier basis vectors depends on their variations. The variations of a graph signal $\x$ are defined in different ways. When the graph representation matrix is the graph shift, the variation of a graph signal $\x$ is defined as,
\begin{equation*} 	\STV_{\Adj} ( \x ) = \left\| \x - \frac{1}{\left\lvert\lambda_{\max}(\Adj)\right\rvert} \Adj \x \right\|_2^2, \end{equation*}
 where $\lambda_{\max}(\Adj)$ is the eigenvalue of $\Adj$ with the largest magnitude. We can show that when the eigenvalues of the graph shift $\Adj $ are sorted in a nonincreasing order $\lambda^{(\Adj)}_1 \geq \lambda^{(\Adj)}_2 \geq \ldots \geq \lambda^{(\Adj)}_N$, the variations of the corresponding eigenvectors follow a nondecreasing order $\STV_{\Adj}(\vv^{(\Adj)}_1) \leq \STV_{(\Adj)}(\vv^{(\Adj)}_2) \leq \ldots \leq \STV_{(\Adj)}(\vv^{(\Adj)}_N)$.

 When the graph representation matrix is the graph Laplacian matrix, the variation of a graph signal $\x$ is defined as,
 \begin{equation*} 	\STV_{\LL} (\x) = \sum_{i,j=1}^N \Adj_{i,j}(x_i-x_j)^2 = \x^T \LL \x. \end{equation*}
Similarly, we can show that when the eigenvalues of the graph Laplacian $\LL$ are sorted in a nondecreasing order $\lambda^{(\LL)}_1 \leq \lambda^{(\LL)}_2 \leq \ldots \leq \lambda^{(\LL)}_N$, the variations of the corresponding eigenvectors follows a nondecreasing order $\STV_{\LL}(\vv^{(\LL)}_1) \leq \STV_{\LL}(\vv^{(\LL)}_2) \leq \ldots \leq \STV_{\LL}(\vv^{(\LL)}_N)$.

 The variations of graph Fourier basis vectors thus allow us to provide the ordering:  the Fourier basis vectors with small variations are considered as \textit{low-frequency} components while the vectors with large variations are considered as \textit{high-frequency} components~\cite{SandryhailaM:131}. We will discuss the difference between these two variations in Section~\ref{sec:R_Smooth}.

 Based on the above discussion, the eigenvectors associated with large eigenvalues of the graph shift  (small eigenvalues of the graph Laplacian) represent low-frequency components and the eigenvectors associated with small eigenvalues of the graph shift (large eigenvalues of graph Laplacian) represent high-frequency components. In the following discussion, we assume all graph Fourier bases are ordered from low frequencies to high frequencies.

\section{Foundations}
\label{sec:foundations}
In this section, we introduce a representation-based framework with three components, including graph signal models, representation dictionaries and the related tasks, in a general and abstract level. This lays a foundation for the following sections, which are essentially special cases that follow this general framework.

\begin{figure}[htb]
  \begin{center}
    \begin{tabular}{cc}
 \includegraphics[width=0.8\columnwidth]{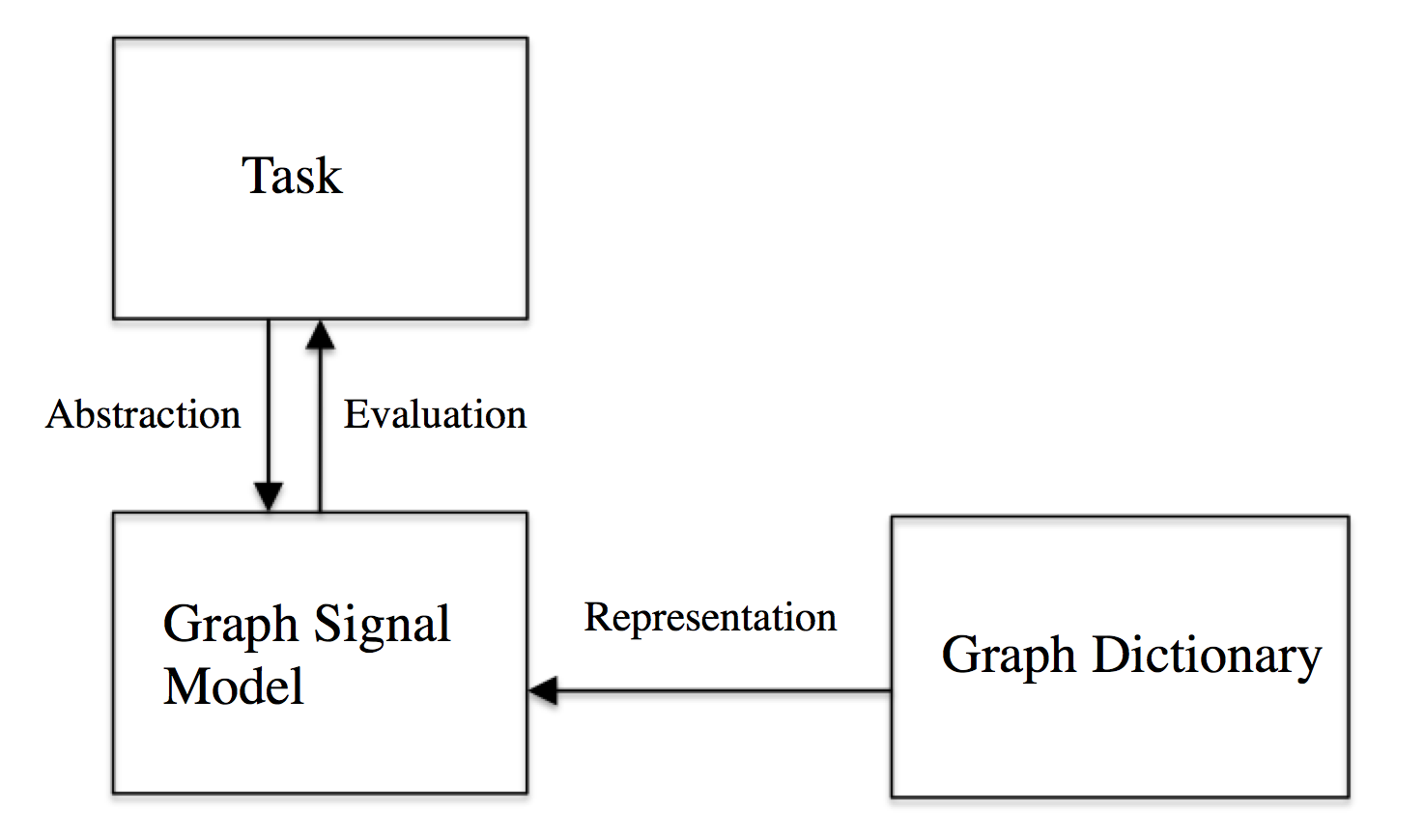}
\end{tabular}
  \end{center}
   \caption{\label{fig:framework} The central concept here is the graph signal model, which is abstracted from given data and is represented by some dictionary.
}
\end{figure}

As shown in Figure~\ref{fig:framework}, when studying a task with a graph, we first model the given data with some graph signal model. The model describes data by capturing its important properties. Those properties can be obtained from observations, domain knowledge, or statistical learning algorithms. We then use a representation dictionary to represent the graph signal model. In the following discussion, we go through each component one by one.

\subsection{Graph Signal Model}
In classical signal processing, people often work with signals with some specific properties, instead of arbitrary signals. For example, smooth signals have been studied extensively over decades; sparse signals are intensively studied recently. Here we also need a graph signal model to describe a class of graph signals with specific properties. In general, there are two approaches to mathematically model a graph signal, including a descriptive approach and a generative approach.

For the descriptive approach, we describe the properties of a graph signal by bounding the output of some operator. Let $\x$ be a graph signal, $f(\cdot)$ be a function operating on $\x$, we define a class of graph signals by restricting
\begin{equation}
\label{eq:descriptive}
 f(\x)  \leq C,
\end{equation}
where $C$ is some constant. For example, we define smooth graph signals by restricting $\x^T \LL \x$ be small~\cite{BelkinNS:06}. 

For the generative approach, we describe the properties of a graph signal by using a graph dictionary. Let $\D$ be a graph dictionary, we define a class of graph signals by restricting
\begin{equation*}
 \x \ = \ \D \a,
\end{equation*}
where $\a$ is a vector of expansion coefficients. For the descriptive approach, we do not need to know everything about a graph signal, instead, we just need to its output of some operator; for the generative approach, we need to reconstruct a graph signal, which requires to know everything about a graph signal.

\subsection{Graph Dictionary}
For a certain graph signal model, we aim to find some dictionary to provide accurate and concise representations.

\subsubsection{Design}
In general, there are two approaches to design a graph dictionary, including a passive approach and a active one.

For the passive approach, the graph dictionary is designed only based on the graph structure; that is,
\begin{equation*}
 \D  = g(\RR),
\end{equation*}
where $g(\cdot)$ is some operator on the graph structure matrix $\RR$. For example, $\D$ can be the eigenvector matrix of $\RR$, which is the graph Fourier basis. In classical signal processing, the Fourier basis, wavelet bases, wavelet frames, and Gabor frames are all constructed using this approach, where the graph is a line graph or a lattice graph~\cite{PueschelM:08}.

For the active approach, the graph dictionary is designed based on both graph structure and a set of given graph signals; that is,
\begin{equation*}
 \D  = g(\Adj, \X),
\end{equation*}
where $\X$ is a matrix representation of a set of graph signals. We can fit to those given graph signals and provide a specialized dictionary. Some related works see~\cite{DongTFV:14, ThanouSF:14}.

\subsubsection{Properties}
\label{sec:properties}
The same class of graph signals can be modeled by various dictionaries. For example, whenever $\D$ is an identity matrix, it can represent arbitrary graph signals, but may not be appealing to represent a non-sparse signals. Depending on the application, we may have different requirements for the constructed graph dictionary. Here are some standard properties of a graph dictionary $\D$, we aim to study.
\begin{itemize}
\item Frame bounds. For any $\x$ in  a certain graph signal model,
\begin{equation*}
\alpha_1  \left\|  \x  \right\|_2 \leq  \left\| \D \x  \right\|_2 \leq  \alpha_2 \left\| \x  \right\|_2,
\end{equation*}
where $\alpha_1, \alpha_2$ are some constants;

\item Sparse representations. For any $\x$ in a certain graph signal model, there exists a sparse coefficient $\a$ with $\left\| \a \right\|_0 \leq C$, which satisfies 
\begin{equation*}
\left\| \x - \D \a \right\|_2^2  \leq \epsilon,
\end{equation*}
where $C, \epsilon$ are some constants;

\item Uncertainty principles. For  any $\x$ in a certain graph signal model,  the following is satisfied 
\begin{equation*}
\left\| \a_1 \right\|_0 + \left\| \a_2 \right\|_0   \geq C,
\end{equation*}
where $\left\| \x - \D_1 \a_1 \right\|_2^2  \leq \epsilon$ and
$\left\| \x - \D_2 \a_2 \right\|_2^2 \leq \epsilon$, and $\D = \begin{bmatrix}
\D_1 & \D_2
\end{bmatrix}$.
\end{itemize}

\subsection{Graph Signal Processing Tasks}
\label{sec:template_tasks}
We mainly consider two standard tasks in signal processing, approximation and sampling followed with recovery.

\subsubsection{Approximation}
Approximation is a standard task to evaluate a representation. The goal is to use a few expansion coefficients to approximate a graph signal.  We consider approximating a graph signal by using a linear combination of a few atoms from $\D$ and solving the following sparse coding problem,
\begin{eqnarray}
\label{eq:sparse_coding}
   \x^*, \a^*=  &  \arg \min_{\x, \a} & d(\x', \x ),
   \\
   \nonumber
   & {\rm subject~to:~} & \x' = \D \a, 
   \\
   \nonumber
   && \left\| \a \right\|_0 \leq K.
\end{eqnarray}
where $d(\cdot, \cdot)$ is some evaluation metric. The objective function measures the difference between the original signal and the approximated one, which evaluates how well the designed graph dictionary represents a given graph signal. The same formulation can also be used for denoising graph signals.

\subsubsection{Sampling and Recovery} 
\label{sec:samplingandrecovery}
The goal is to recover an original graph signal from a few samples. We consider a general sampling and recovery setting. We consider any decrease in dimension via a linear operator as sampling, and, conversely, any increase in dimension via a linear operator as recovery~\cite{VetterliKG:12}.  Let $\F \in \R^{N \times N}$ be a sampling pattern matrix, which is constrained by a given application and the sampling operator is 
\begin{equation}
\label{eq:Psi}
\Psi = \Cc \F \in \R^{M \times N},
\end{equation}
where $\Cc \in \R^{M \times N}$ selects rows from $\F$. For example, when we choose the $k$th row of $\F$ as the $i$th sample, the $i$th row of $\Cc$ is
\begin{equation*}
\label{eq:C}
 \Cc_{i,j} = 
  \left\{ 
    \begin{array}{rl}
      1, & j = k;\\
      0, & \mbox{otherwise}.
  \end{array} \right.
\end{equation*}

There are three sampling strategies: (1) uniform sampling when designing $\Cc$, , where row indices are chosen from from $\{0, 1, \cdots, N-1\}$ independently and uniformly; and~\emph{experimentally design sampling}, where row indices can be chosen beforehand; and~\emph{active sampling}, where we will use feedback as samples are sequentially collected to decide the next row to be sampled. Each sampling strategy can be implemented by two approaches, including a random approach and a deterministic one. The sampling pattern matrix $\F$ constraints the following sampling patterns: when $\F$ is an identity matrix, $\Psi$ is a subsampling operator; when $\F$ is a Gaussian random matrix,  $\Psi$ is a compressed sampling operator.

In the sampling phase, we take
samples with the sampling operator $\Psi$,
\begin{equation*}
\x_{\Psi} = \Psi \y = \Psi ( \x  + \epsilon),
\end{equation*}
 is a vector of samples and $\epsilon$ is
noise with zero mean and $\sigma^2$ as variance. In the recovery phase, we
reconstruct the graph signal by using a recovery operator $\Phi$, 
\begin{equation*}
\x' = \Phi \x_{\Psi},
\end{equation*}
where $\x'$ recovers $\x$ either
exactly or approximately. The evaluation metric can be the mean square error or other metrics.  Without any property of $\x$, it is hopeless to design an efficient sampling and recovery strategy. Here we focus on a special graph signal model, which can be described by a graph dictionary.  The prototype of designing sampling and recovery strategies is
\begin{eqnarray*}
\Psi^*(\D), \Phi^*(\D) = \min_{\Psi, \Phi} \max_{\a} && d(\x', \x),
\\
 {\rm subject~to} && \x' = \Phi \Psi (\x + \epsilon),
\\
&& \x = \D \a,
\end{eqnarray*}
where $d(\cdot, \cdot)$ is some evaluation metric. The optimal sampling and recovery strategies $\Psi^*, \Phi^* $ are influenced by the given graph dictionary $\D$. We often consider fixing either the sampling strategy or the recovery strategy and optimizing over the other one.

\section{Representations of Smooth Graph Signals}
\label{sec:R_Smooth}
Smooth graph signals are mostly studied in the previous literature; however, many works only provide a heuristic. Here we rigorously define graph signal models and design the corresponding graph dictionaries.

\begin{figure*}[htb]
  \begin{center}
    \begin{tabular}{cccc}
 \includegraphics[width=0.4\columnwidth]{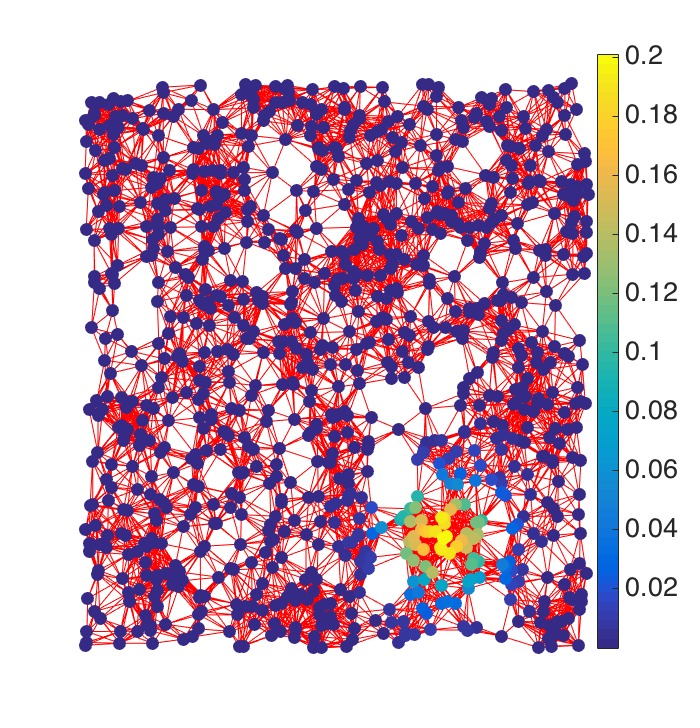} &
  \includegraphics[width=0.4\columnwidth]{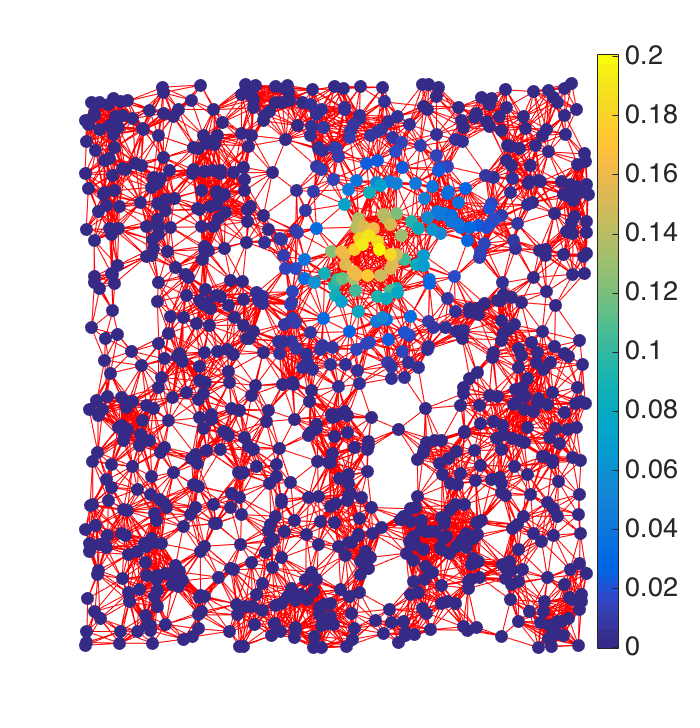} 
&
 \includegraphics[width=0.4\columnwidth]{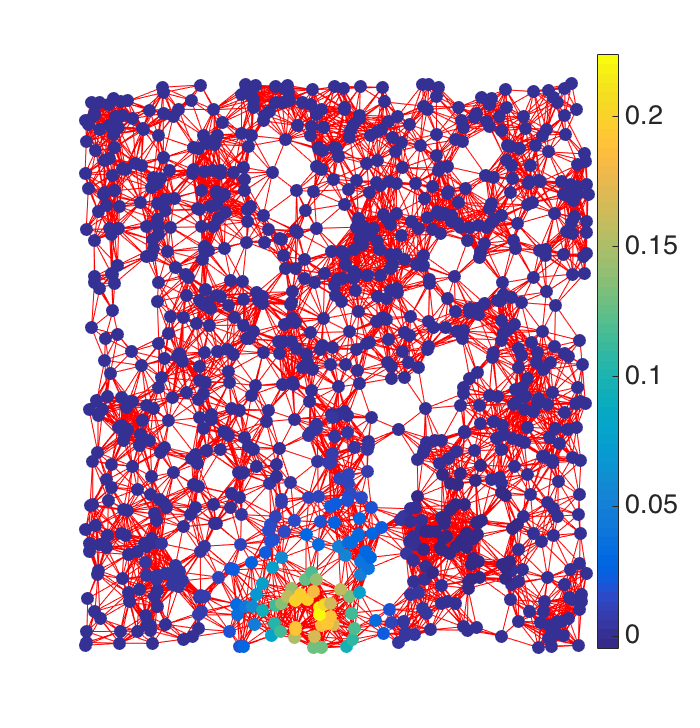}
 &
  \includegraphics[width=0.4\columnwidth]{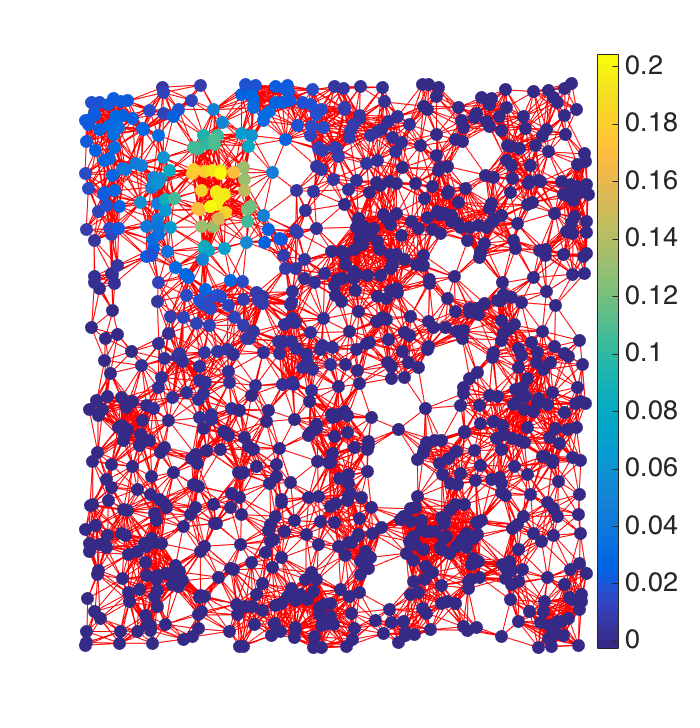}
\\
{\small (a) $\vv_1^{(\W)}$.} & {\small (b) $\vv_2^{(\W)}$.}  &
{\small (c) $\vv_3^{(\W)}$.} &
{\small (d) $\vv_4^{(\W)}$.}
  \\
\includegraphics[width=0.4\columnwidth]{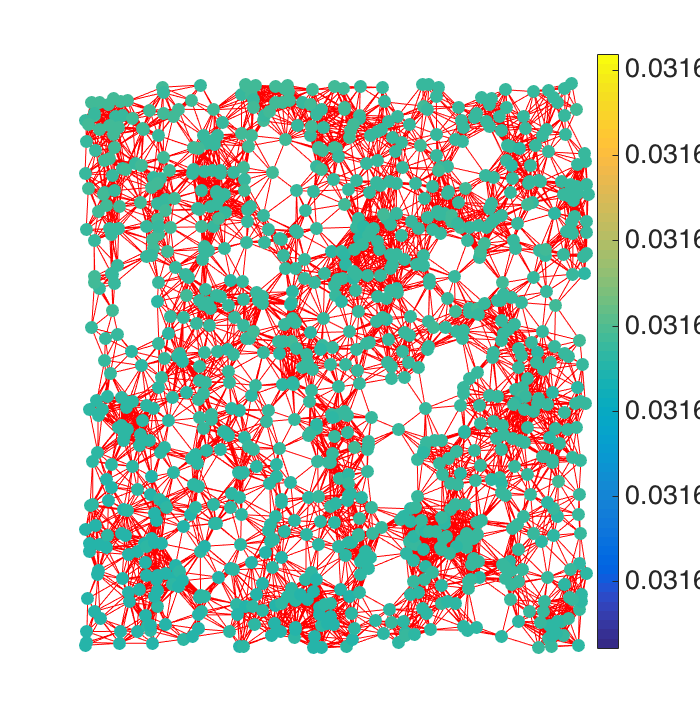} &
  \includegraphics[width=0.4\columnwidth]{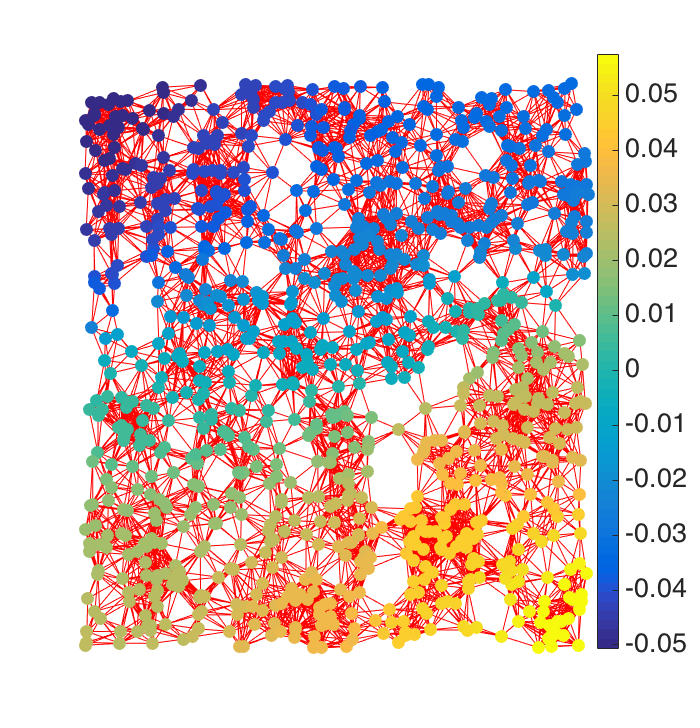} 
&
 \includegraphics[width=0.4\columnwidth]{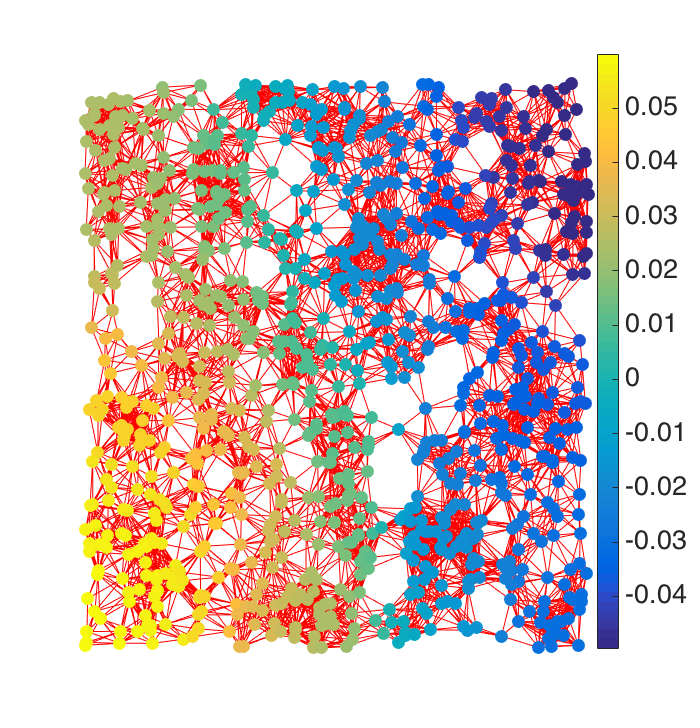}
 &
  \includegraphics[width=0.4\columnwidth]{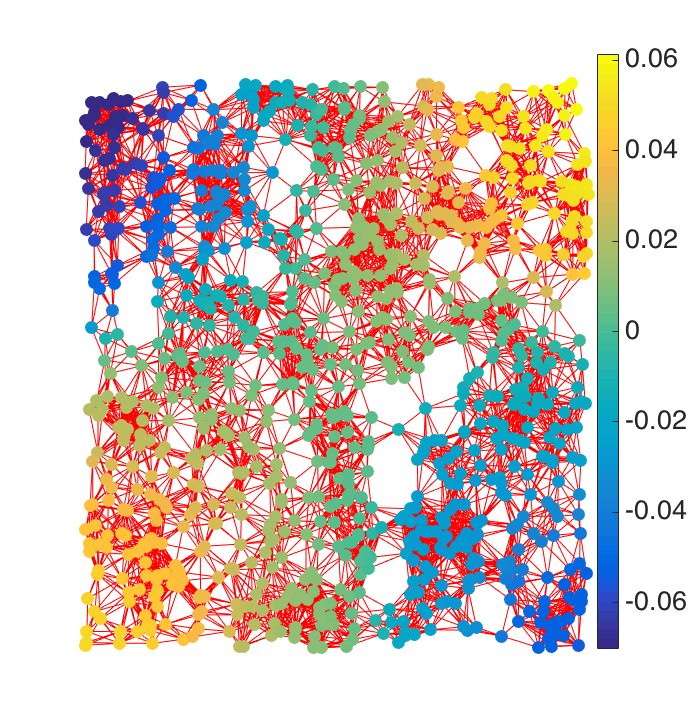}
\\
{\small (e) $\vv_1^{(\LL)}$.} & {\small (f) $\vv_2^{(\LL)}$.}  &
{\small (g) $\vv_3^{(\LL)}$.} &
{\small (h) $\vv_4^{(\LL)}$.}
  \\
  \includegraphics[width=0.4\columnwidth]{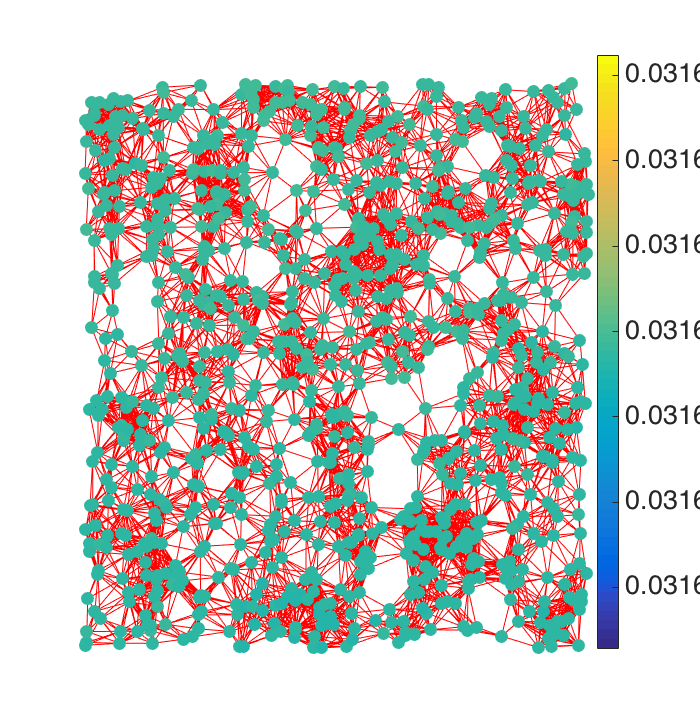} &
  \includegraphics[width=0.4\columnwidth]{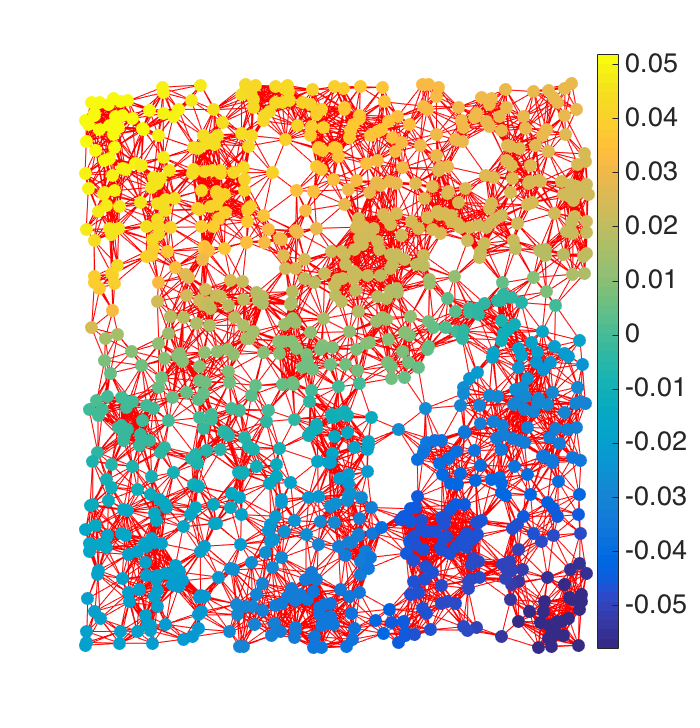} 
&
 \includegraphics[width=0.4\columnwidth]{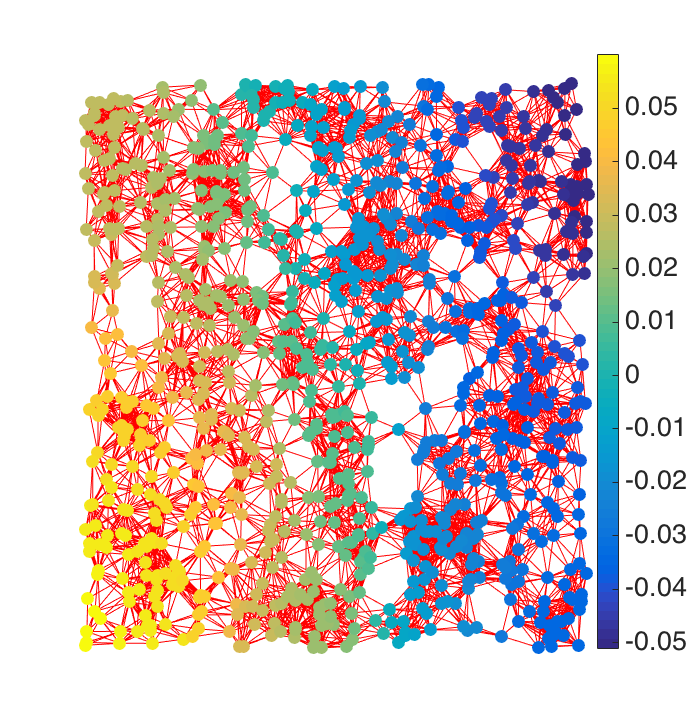}
 &
  \includegraphics[width=0.4\columnwidth]{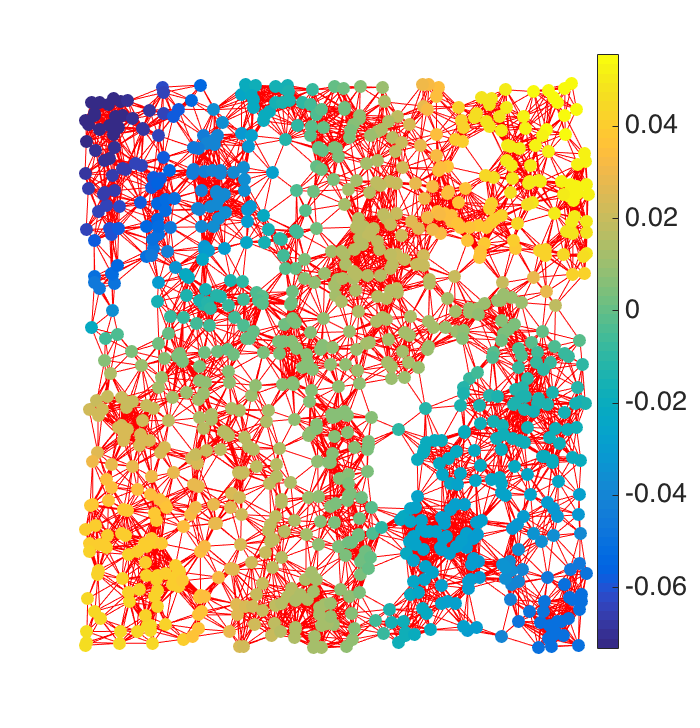}
\\
{\small (i) $\vv_1^{(\Pj)}$.} & {\small (j) $\vv_2^{(\Pj)}$.}  &
{\small (k) $\vv_3^{(\Pj)}$.} &
{\small (l) $\vv_4^{(\Pj)}$.}
  \\
	\end{tabular}
  \end{center}
   \caption{\label{fig:geo_fourier} Graph Fourier bases of a geometric graph. $\Vm_{\W}$ localizes in some small regions;  $\Vm_{\LL}$ and $\Vm_{\Pj}$ have similar behaviors.}
\end{figure*}

\subsection{Graph Signal Models}
We introduce four smoothness criteria for graph signals; while they have been implicitly mentioned previously, none have been rigorously defined. The goal here is not to conclude which criterion or representation approach works best; instead, we aim to study the properties of various smoothness criteria and model a smooth graph signal with a proper criterion.

We start with the pairwise Lipschitz smooth criterion.
\begin{defn}
\label{df:pairwise_Lip}
A graph signal $\x$ with unit norm is~\emph{pairwise Lipschitz smooth} with parameter $C$ when it satisfies
\begin{equation*}
	| x_i - x_j |   \ \leq \  C \ d(v_i, v_j),~{\rm for~all~} i,j = 0, 1, \ldots, N-1,
\end{equation*}
with $d(v_i, v_j)$ the distance between the $i$th and the $j$th nodes.
\end{defn}
\noindent  We can choose the geodesic distance, the diffusion
distance~\cite{CoifmanL:06}, or some other distance metric for
$d(\cdot, \cdot)$.  Similarly to the traditional Lipschitz
criterion~\cite{Mallat:09}, the pairwise Lipschitz smoothness criterion
emphasizes pairwise smoothness, which zooms into the difference
between each pair of adjacent nodes.

\begin{defn}
\label{df:global_Lip}
A graph signal $\x$ with unit norm is~\emph{total Lipschitz smooth} with parameter $C$ when it satisfies
\begin{equation*}
	\sum_{(i,j) \in \E} \W_{i,j} ( x_i - x_j )^2   \ \leq \  C.
\end{equation*}
\end{defn}
\noindent The total Lipschitz smoothness criterion generalizes the pairwise Lipschitz smoothness criterion while still emphasizing pairwise smoothness, but in a less restricted manner; it is also known as the Laplacian smoothness criterion~\cite{BelkinN:03}.

\begin{defn}
\label{df:local_neighboring_smooth}
A graph signal $\x$ with unit norm  is~\emph{local normalized neighboring smooth  } with parameter $C$ when it satisfies
\begin{equation*}
	  \sum_i \left(  x_i -  \frac{1}{\sum_{j: (i,j) \in \E} \W_{i,j} }  \sum_{j: (i,j) \in \E} \W_{i,j} x_j  \right)^2  \ \leq \  C.
\end{equation*}
\end{defn}
\noindent The local normalized neighboring smoothness criterion compares each node to the local normalized average of its immediate neighbors.

\begin{defn}
\label{df:global_neighboring_smooth}
A graph signal $\x$ with unit norm  is~\emph{global normalized neighboring smooth} with parameter $C$ when it satisfies
\begin{equation*}
	  \sum_i \left(  x_i -  \frac{1}{ |\lambda_{\max}(\W)| }  \sum_{j: (i,j) \in \E} \W_{i,j} x_j  \right)^2  \ \leq \  C.
\end{equation*}
\end{defn}
\noindent The global normalized neighboring smoothness criterion compares each node to the global normalized  average of its immediate neighbors. The difference between the local normalized neighboring smoothness criterion and the global normalized neighboring smoothness criterion is the normalization factor. For the local normalized neighboring smoothness criterion, each node has its own normalization factor;
for the global normalized neighboring smoothness criterion, all nodes have the same normalization factor.

The four criteria quantify smoothness in different ways: the pairwise
and the total Lipschitz ones focus on the variation of two signal
coefficients connected by an edge with the pairwise Lipschitz one more
restricted, while the local and global neighboring smoothness criterion focuses on
comparing a node to the average of its neighbors.

\subsection{Graph Dictionary}
As shown in~\eqref{eq:descriptive}, The graph signal models in Definitions~\ref{df:pairwise_Lip},~\ref{df:global_Lip},~\ref{df:local_neighboring_smooth},~\ref{df:global_neighboring_smooth} are introduced in a descriptive approach. Following these, we are going to translate the descriptive approach into a generative approach; that is,  we  represent
the corresponding signal classes satisfying each of the four criteria by some representation graph dictionary.

\subsubsection{Design}
 We first construct polynomial graph signals that satisfy the Lipschitz smoothness criterion.

\begin{defn}
\label{df:polynomial_sig}
A graph signal $\x$ is polynomial with degree $K$ when
\begin{equation*}
	\x \ = \  \D_{\rm poly(K)}  \a \ = \ \begin{bmatrix}
	{\bf 1}  &  \D^{(1)}  & \D^{(2)}  &  \ldots &  \D^{(K)} 
	\end{bmatrix} 
	\a \in \R^N,
\end{equation*}
where $\a \in \R^{KN+1}$ and $\D_{\rm poly(K)} $ is a graph polynomial dictionary with $\D^{(k)}_{i,j} = d^k(v_i, v_j)$. Denote this class by $\PL(K)$. 
\end{defn}

\begin{figure}[htb]
  \begin{center}
    \begin{tabular}{cc}
 \includegraphics[width=0.4\columnwidth]{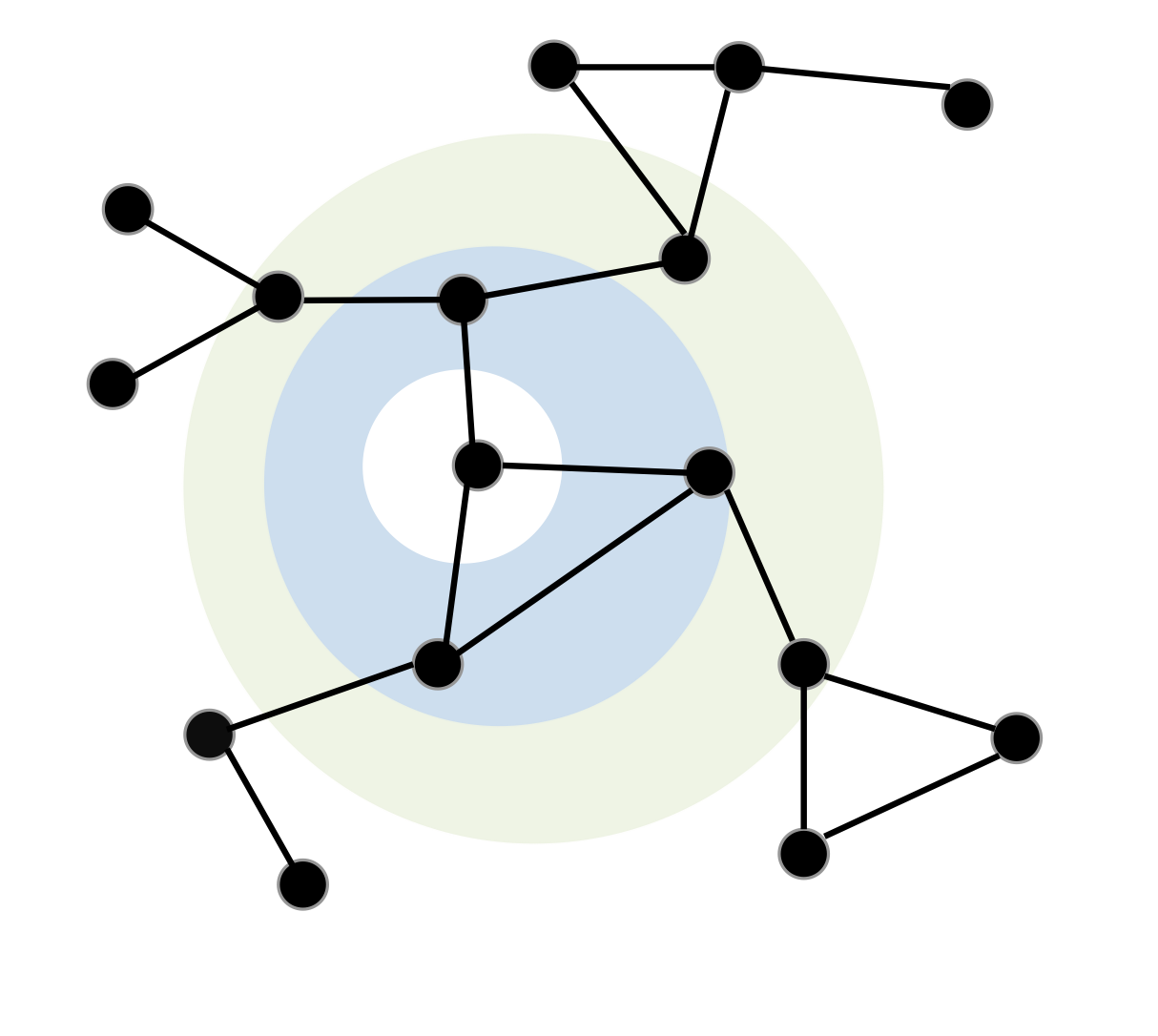} &
 \includegraphics[width=0.4\columnwidth]{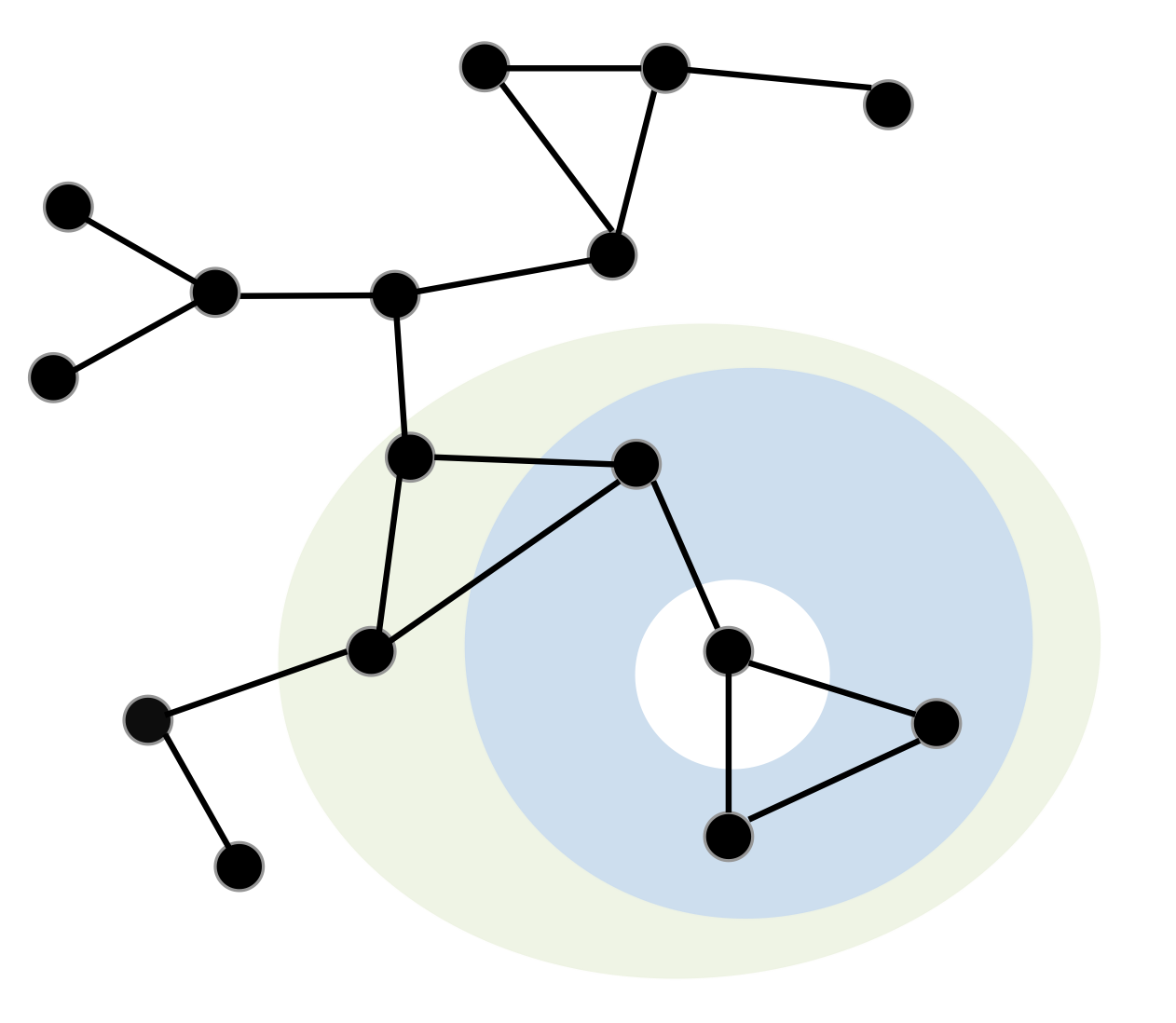}
\end{tabular}
  \end{center}
   \caption{\label{fig:origin} Different origins lead to different  coordinate systems; white, blue, and green denote the origin, nodes with geodesic distance $1$ from the origin, and nodes with geodesic distance $2$ from the origin, respectively. }
\end{figure}

\noindent  
In classical signal processing, polynomial time signals can be
expressed as $x_n = \sum_{k=0}^K a_k n^k$, $n = 1, \ldots, N$; we
can rewrite this as in the above definition as $\x = \D_{K} \a$, with
$(D_K)_{n,k} = n^k$. The columns of $\D_{K}$ are denoted as $\D^{(k)}$, $k
= 0, \ldots, K$, and called \emph{atoms}; the elements of each atom
$\D^{(k)}$ are $n^k$.  Since polynomial time signals are
shift-invariant, we can set any time point as the origin; such signals
are thus characterized by $K+1$ degrees of freedom $a_k$, $k = 0,
\ldots, K$. This is not true for graph signals, however; they are not
shift-invariant and any node can serve as the origin (see
Figure~\ref{fig:origin}). In the above definition, $\D^{(k)}$ are now
matrices with the number of atoms equal to the number of nodes $N$
(with each atom corresponding to the node serving as the origin). The
dictionary $\D_{K}$ thus contains $KN+1$ atoms.

\begin{myThm} ({\bf Graph polynomial dictionary represents the pairwise Lipschiz smooth signals})
\label{thm:BLvsGS}
$\PL(1)$ is a subset of the pairwise Lipschitz smooth with some parameter $C$.
\end{myThm}
\begin{proof}
Let $\x \in \PL(1)$, that is,
$$
\x = \begin{bmatrix}
	{\bf 1}  &  \D^{(1)}  
	\end{bmatrix} \a,
$$
Then, we write the pairwise Lipschitz smooth criterion as
\begin{eqnarray*}
|x_i - x_j| 
& = &  | \sum_{k} \left( d(v_k, v_i) - d(v_k, v_j) \right)   a_k |
\\
& \leq &  \sum_{k} | d(v_k, v_i) - d(v_k, v_j) |  |a_k |
\\
\\
& \leq &  \sum_{k} |a_k |  | d(v_i, v_j) | \ = \ \left\| \a \right\|_1  | d(v_i, v_j) |.
\end{eqnarray*}
The parameter $C = \left\| \a \right\|_1$, which corresponds to the energy of the original graph signal.
\end{proof}

We now construct bandlimited signals that satisfy the total Lipschitz, local normalized neighboring and global normalized neighboring smoothness smoothness criteria.
\begin{defn}
A graph signal $\x$ is bandlimited with respect to a graph Fourier basis $\Vm$ with bandwidth $K$ when
\begin{equation*}
	\x = \Vm_{(K)} \a,
\end{equation*}
where $\a \in \R^{K}$ and $\Vm_{(K)}$ is a submatrix containing the first $K$ columns of $\Vm$. Denote this class by $\BL_{\Vm}(K)$~\cite{ChenVSK:15}.
\end{defn}
When $\Vm$ is the eigenvector matrix of the unnormalized graph Laplacian matrix, we denote it as $\Vm_{\LL}$ and can show that signals in $\BL_{\Vm_{\LL}}(K)$ are total Lipschitz smooth; when $\Vm$ is the eigenvector matrix of the transition matrix, we denote it as $\Vm_{\Pj}$ and can show that signals in $\BL_{\Vm_{\Pj}}(K)$ are local normalized neighboring smooth; when $\Vm$ is the eigenvector matrix of the weighted adjacency matrix, we denote it as $\Vm_{\W}$ and can show that signals in $\BL_{\Vm_{\W}}(K)$ are global normalized neighboring smooth.

\begin{myThm} ({\bf The graph Fourier basis of the graph Laplacian represents the total Lipschiz smooth signals})
\label{thm:BLvsGS}
For any $K \in \{1, \cdots, N \}$, $\BL_{\Vm_{\LL}}(K)$ is a subset of the total Lipschitz smooth with parameter $C$, when $C \geq \lambda^{(\LL)}_k$.
\end{myThm}
\begin{proof}
Let $\x$ be a graph signal with bandwidth $K$, that is,
$$
\x = \sum_{k=1}^{K} \widehat{x}_k \vv^{(\LL)}_k,
$$
Then, we write the total Lipschitz smooth criterion as
\begin{eqnarray*}
&& \sum_{(i,j) \in \E} \W_{i,j} ( x_i - x_j )^2 
\ = \ \x^T \LL \x 
\\
& =  &  (\sum_{k=1}^{K} \widehat{x}_k \vv^{(\LL)}_k)^T (\sum_{k=0}^{K-1} \widehat{x}_k \lambda_k \vv^{(\LL)}_k)
\ = \   \sum_{k=1}^{K} \lambda^{(\LL)}_k \widehat{x}_k^2 
\\
& \leq & \lambda^{(\LL)}_K   \sum_{k=1}^{K} \widehat{x}_k^2  
\ = \   \lambda^{(\LL)}_K.
\end{eqnarray*}
\end{proof}

\begin{myThm}({\bf The graph Fourier basis of the transition matrix represents the local normalized neighboring smooth signals})
\label{thm:LNNS}
For any $K \in \{1, \cdots, N \}$, $\BL_{\Vm_{\Pj}}(K)$ is a subset of the local normalized neighboring smooth with parameter $C$, when $C \geq  (1 - \lambda^{(\Pj)}_{K})^2$.
\end{myThm}
\begin{proof}
Let $\x$ be a graph signal with bandwidth $K$, that is,
$$
\x = \sum_{k=1}^{K} \widehat{x}_k \vv^{(\Pj)}_k,
$$
Then, we write the local normalized neighboring smooth criterion as

\begin{eqnarray*}
&& \left| x_i - \frac{1}{\sum_{j \in \mathcal{N}_i} \W_{i,j} }  \sum_{j \in \mathcal{N}_i} \W_{i,j} x_j  \right|
\\
&= &   \left| \left(\sum_{k=1}^{K} \widehat{x}_k \vv^{(\Pj)}_k \right)_i - \sum_{j  \in \mathcal{N}_i} \Pj_{i,j} \left(\sum_{k=1}^{K} \widehat{x}_k \vv^{(\Pj)}_k \right)_j \right|
\\
& = &   \left| \sum_{k=1}^{K} \widehat{x}_k \left(  (\vv^{(\Pj)}_k)_i - \sum_{j \in \mathcal{N}_i} \Pj_{i,j} (\vv^{(\Pj)}_k)_j \right) \right|
\\
& = &   \left| \sum_{k=1}^{K} \widehat{x}_k  (1 - \lambda_k)  (\vv^{(\Pj)}_k)_i \right|
\\
& \leq &  (1 - \lambda^{(\Pj)}_{K})  \left| \sum_{k=0}^{K-1} \widehat{x}_k  (\vv^{(\Pj)}_k)_i  \right| \ = \  (1 - \lambda^{(\Pj)}_{K})  |x_i|.
\end{eqnarray*}
The last equality follows from the fact that $\vv^{(\Pj)}_k$ and $\lambda^{(\Pj)}_k$ are eigenvectors and eigenvalues of $\Pj$.
\begin{eqnarray*}
&& \sum_i \left(  x_i -  \frac{1}{\sum_{j  \in \mathcal{N}_i }}  \sum_{j  \in \mathcal{N}_i } \W_{i,j} x_j  \right)^2 
\\
& = & \sum_{i=1}^{N} |x_i - \frac{1}{\sum_{j \in \mathcal{N}_i} \W_{i,j} }  \sum_{j \in \mathcal{N}_i} \W_{i,j} x_j |^2
\\
& \leq & \sum_{i=0}^{N-1}   (1 - \lambda^{(\Pj)}_{K} )^2  |x_i|^2 \ = \   (1 - \lambda^{(\Pj)}_{K})^2 .
\end{eqnarray*}
\end{proof}

\begin{myThm} ({\bf The graph Fourier basis of the adjacency matrix represents the global normalized neighboring smooth signals})
\label{thm:GNNS}
For any $K \in \{1, \cdots, N \}$, $\BL_{\Vm_{\W}}(K)$ is a subset of the global normalized neighboring smooth with parameter $C$, when $C \geq  (1 - \lambda^{(\W)}_{K} / |\lambda_{\max}{(\W)}| )^2$.
\end{myThm}
The proof is similar to Theorem~\ref{thm:LNNS}. Note that for graph Laplacian, the eigenvalues are sorted in an ascending order; for the transition matrix and the adjacency matrix, the eigenvalues are sorted in a descending order.

Each of these three models generates smooth graph signals according to one of the four criteria in Definitions~\ref{df:pairwise_Lip},~\ref{df:global_Lip},~\ref{df:local_neighboring_smooth} and~\ref{df:global_neighboring_smooth}: $\PL(K)$ models Lipschitz smooth signals; $\BL_{\Vm_{\LL}}(K)$ models total Lipschitz smooth signals; $\BL_{\Vm_{\Pj}}(K)$ with  models the local normalized neighboring smooth signals; and $\BL_{\Vm_{\W}}(K)$ models the global normalized neighboring smooth signals; the corresponding graph representation dictionaries are $\D_{\rm poly(K)}$, $\Vm_{\LL}$, $\Vm_{\Pj}$, and $\Vm_{\W}$.

\subsubsection{Properties}
 We next study the properties of graph representation dictionaries for smooth graph signals, especially for graph Fourier bases $\Vm_{\LL}$, $\Vm_{\Pj}$, and $\Vm_{\W}$. We first visualize them in Figures~\ref{fig:geo_fourier}. Figure~\ref{fig:geo_fourier} compares the first four graph Fourier basis vectors of $\Vm_{\LL}$, $\Vm_{\Pj}$ and $\Vm_{\W}$ in a geometric graph. We see that $\Vm_{\W}$ tends to localize in some small regions; $\Vm_{\LL}$ and $\Vm_{\Pj}$ have similar behaviors. 
 
 We then check the properties mentioned in Section~\ref{sec:properties}.
 
 \mypar{Frame Bound} Graph polynomial dictionary is highly redundant and the frame bound is loose. When the graph is undirected, the adjacency matrix is symmetric, then $\Vm_{\LL}$ and $\Vm_{\W}$ are orthonormal. It is hard to draw any meaningful conclusion when the graph is directed; we leave it for the future work.
  
 \mypar{Sparse Representations} Graph polynomial dictionary provides sparse representations for polynomial graph signals. On the other hand, the graph Fourier bases provide sparse representations for the bandlimited graph signals. For approximately bandlimited graph signals, there are some residuals coming from the high-frequency components.
 
 \mypar{Uncertainty Principles} In classical signal processing, it is well known that signals cannot localize in both time and frequency domains at the same time~\cite{VetterliKG:12, DonohoS:89}. Some previous works extend this uncertainty principle to graphs by studying how well a graph signal exactly localize in both the graph vertex and graph spectrum domain~\cite{AgaskarL:13, TsitsveroBL:15}. Here we study how well a graph signal approximately localize in both the graph vertex and graph spectrum domain. We will see that the localization depends on the graph Fourier basis.
   
\begin{defn}
\label{df:concentrate_v}
A graph signal $\x$ is~\emph{$\epsilon$-vertex concentrated} on a graph vertex set $\Gamma$ when it satisfies
\begin{equation*}
	\left\| \x - \Id_{\Gamma} \x \right\|_2^2 \leq \epsilon,
\end{equation*}
where $\Id_{\Gamma} \in \R^{N \times N}$ is a diagonal matrix, with $(\Id_{\Gamma})_{i,i} = 1$ when $i \in \Gamma$ and $0$, otherwise.
\end{defn}

The vertex set $\Gamma$ represents a region that supports the main energy of signals. When $|\Gamma|$ is small, a $\epsilon$-vertex concentrated signal is approximately sparse.

 \begin{defn}
\label{df:concentrate_s}
A graph signal $\x$ is~\emph{$\epsilon$-spectrum concentrated} on a graph spectrum band $\Omega$ when it satisfies
\begin{equation*}
	\left\| \x - \Vm_{\Omega} \Um_{\Omega} \x \right\|_2^2 \leq \epsilon,
\end{equation*}
where $ \Vm_{\Omega} \in \R^{N \times |\Omega|}$ is a submatrix of $\Vm$ with columns selected by $\Omega$ and $ \Um_{\Omega} \in \R^{|\Omega| \times N}$ is a submatrix of $\Vm$ with rows selected by $\Omega$.
\end{defn}

The graph spectrum band $\Omega$ provides a bandlimited space that supports the main energy of signals. An equivalent formulation is $\left\| \widehat{\x} - \Id_{\Omega} \widehat{\x} \right\|_2^2 \leq \epsilon$. Definition~\ref{df:concentrate_s} is a simpler version of the approximately bandlimited space in~\cite{ChenVSK:15a}.

\begin{myThm}({\bf Uncertainty principle of the graph vertex and spectrum domains})
\label{thm:GNNS}
Let a unit norm signal $\x$ supported on an undirected graph  be $\epsilon_{\Gamma}$-vertex concentrated and $\epsilon_{\Omega}$-spectrum concentrated at the same time. Then,
\begin{equation*}
|\Gamma| \cdot |\Omega| \geq  \frac{(1 - (\epsilon_{\Omega}  +  \epsilon_{\Gamma} ))^2 }{ \left\| \Um_{\Omega} \right\|_{\infty}^2 }.
\end{equation*}

\end{myThm}

\begin{proof}
 We first show $\left\| \Vm_{\Omega} \Um_{\Omega}  \Id_{\Gamma} \right\|_2 \leq \left\| \Um_{\Omega} \right\|_{\infty}  \sqrt{ |\Gamma|\cdot |\Omega|}$.
\begin{eqnarray*}
&&	( \Vm_{\Omega} \Um_{\Omega} \Id_{\Gamma}  \x )_s
\ = \ \sum_{k \in \Omega} \Vm_{s,k} \left( \sum_{i \in \Gamma} \Um_{k,i} x_i \right)
\\
& = &  \sum_{i \in \Gamma}  \left(   \sum_{k \in \Omega} \Vm_{s,k} \Um_{k,i} \right) x_i
\ = \ \sum_i q(s,i) x_i,
\end{eqnarray*}
   where
   \begin{equation*}
	\label{eq:Delta}
	q(s,i) = 
  \left\{ 
    \begin{array}{rl}
      \sum_{k \in \Omega} \Vm_{s,k} \Um_{k,i}, & i \in \Gamma;\\
      0, & \mbox{otherwise}.
  \end{array} \right.
\end{equation*}
Let $\y^{(i)}$ be a graph signal with $y^{(i)}_s = q(s,i)$. Then, $( \widehat{\y^{(i)}} )_k = {\bf 1}_{k \in \Omega} \Um_{k,i}$. We then have
\begin{eqnarray*}
&&	\left\| \Vm_{\Omega} \Um_{\Omega} \Id_{\Gamma}  \right\|_{2}^2
 \ \leq \  \left\| \Vm_{\Omega} \Um_{\Omega} \Id_{\Gamma}  \right\|_{\rm HS}^2
\\
& = & \sum_{i \in \Gamma} \sum_{s} |q(s,i)|^2 
\ =  \sum_{i \in \Gamma} \left\| \y^{(i)} \right\|_2^2
\\
& = & \sum_{i \in \Gamma} \left\| \widehat{ \y^{(i)} } \right\|_2^2
\ = \  \sum_{i \in \Gamma} \sum_{k} \left( {\bf 1}_{k \in \Omega} \Um_{k,i} \right)^2
\\
& \leq & \left\| \Um_{\Omega} \right\|_{\infty}^2 \sum_{i \in \Gamma} \sum_{k} {\bf 1}_{k \in \Omega} 
\ = \   \left\| \Um_{\Omega} \right\|_{\infty}^2 |\Gamma| \cdot |\Omega|.
\end{eqnarray*}

We then show that $\left\|  \Vm_{\Omega} \Um_{\Omega} \Id_{\Gamma} \x  \right\|_2 \geq 1 - (\epsilon_{\Omega}  +  \epsilon_{\Gamma} )$.
Based on the assumption, we have
   \begin{eqnarray*}
 &&  \left\| \x - \Vm_{\Omega} \Um_{\Omega} \Id_{\Gamma}  \x  \right\|_2
   \\
& = &   \left\| \x - \Id_{\Gamma} \x \right\|_2  + \left\| \Id_{\Gamma} \x -  \Vm_{\Omega} \Um_{\Omega} \Id_{\Gamma} \x  \right\|_2 
	\\
& \leq & \epsilon_{\Omega}  +  \epsilon_{\Gamma}. 
   \end{eqnarray*}
Since $\x$ has a unit norm, by the triangle inequality,  we have
   \begin{equation*}
  \left\|  \Vm_{\Omega} \Um_{\Omega} \Id_{\Gamma}  \right\|_2 \geq 1 - (\epsilon_{\Omega}  +  \epsilon_{\Gamma} ).
   \end{equation*}
Finally, we combine two results and obtain  
\begin{equation*}
|\Gamma|\cdot |\Omega|  \geq  \frac{\left\| \Vm_{\Omega} \Um_{\Omega}  \Id_{\Gamma} \right\|_2^2}{\left\| \Um_{\Omega} \right\|_{\infty}^2 } \ > \ \frac{  (1 - (\epsilon_{\Omega}  +  \epsilon_{\Gamma} ))^2 }{ \left\| \Um_{\Omega} \right\|_{\infty}^2}
\end{equation*}
\end{proof}
We see that the lower bound involves with the maximum magnitude of $\Um_{\Omega}$. In classical signal processing, $\Um$ is the discrete Fourier transform matrix, so $\left\| \Um_{\Omega} \right\|_{\infty}  = 1/\sqrt{N}$; the lower bound is $O(N)$ and signals cannot localize in both time and frequency domain. However, for complex and irregular graphs, the energy of a graph Fourier basis vector may concentrate on a few elements, that is, $\left\| \Um_{\Omega} \right\|_{\infty} = O(1)$, as shown in Figures~\ref{fig:geo_fourier}(a)(b)(c)(d). It is thus possible that  graph signals can be localized in both the vertex and spectrum domain. We now illustrate this localization phenomenon
   
 \mypar{Localization Phenomenon}
 A part of this section has been shown in~\cite{LiangCK:16}. We show it here for the completeness. The \textit{localization} of a graph signal means that most elements in a graph signal are zeros, or have small values; only a small number of elements have large magnitudes and the corresponding nodes are clustered in one subgraph with a small diameter. 


Prior work uses~\emph{inverse participation ratio} (IPR) to quantify localization~\cite{CucuringuM:11}. The IPR of a graph signal $\x \in \R^N$ is
\begin{equation*}
	{\rm IPR} = \frac{\sum_{i=1}^N x_i^{4}}{(\sum_{i=1}^N x_i^{2})^2}.
\end{equation*}
 A large IPR indicates that $\x$ is localized, while a small IPR indicates that $\x$ is not localized. The range of IPR is from $0$ to $1$. For example, $\x=[1/\sqrt{N},1/\sqrt{N},\cdots,1/\sqrt{N}]^T$ is the most delocalized vector with $\text{IPR}=1/N$, while $\x=[1,0,\cdots,0]^T$ is the most localized vector with $\text{IPR}=1$. IPR has some shortcomings: a) IPR only promotes sparsity,  that is, a high IPR does not necessarily mean that the nonzero elements concentrate in a clustered subgraph, which is the essence of localization (Figure ~\ref{fig:cluster}); b) IPR does not work well for large-scale datasets. When $N$ is large, even if only a small set of elements are non zero, IPR tends to be small.

\begin{figure}[htb]
  \begin{center}
    \begin{tabular}{cc}
   \includegraphics[width=0.4\columnwidth]{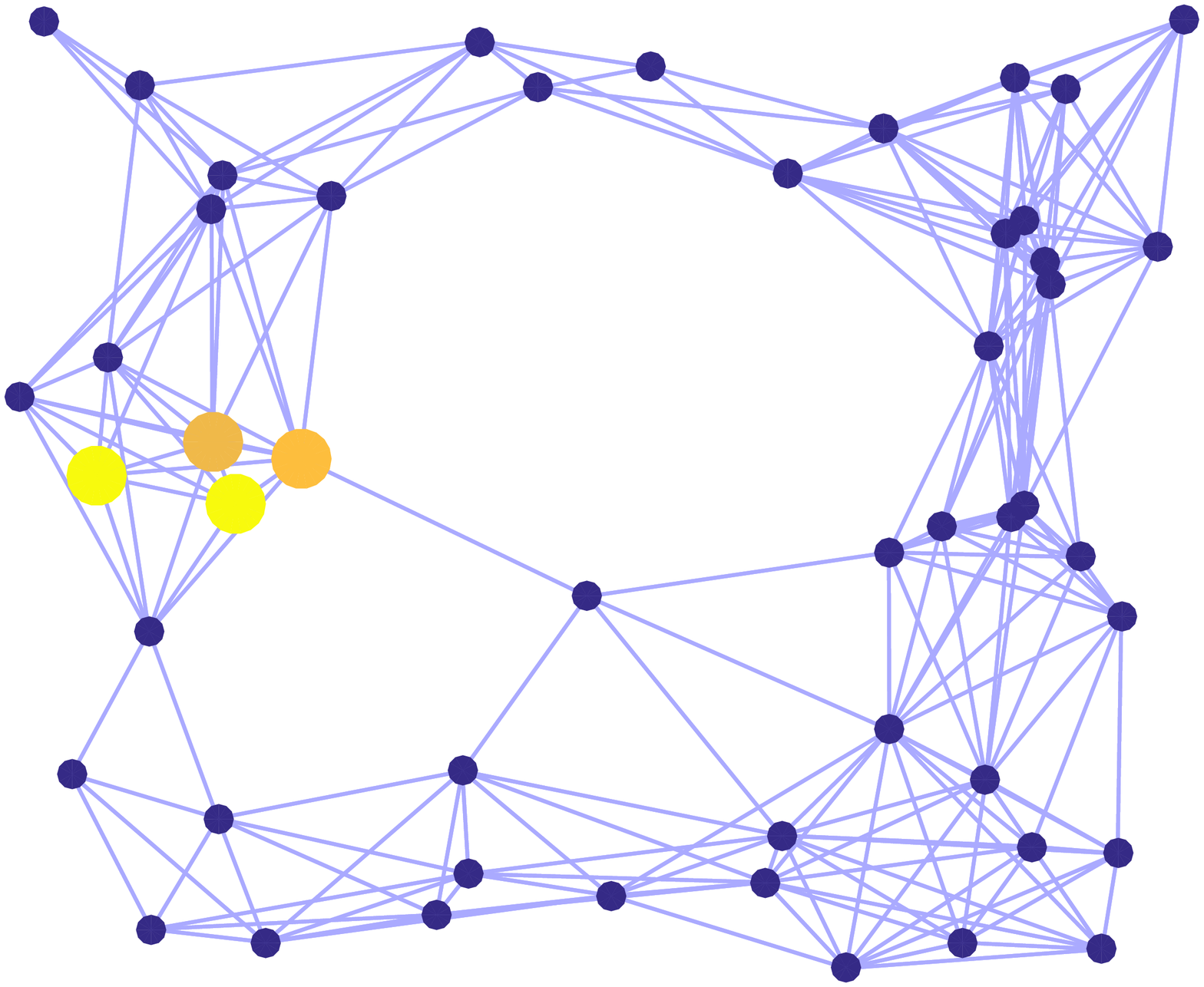} & \includegraphics[width=0.4\columnwidth]{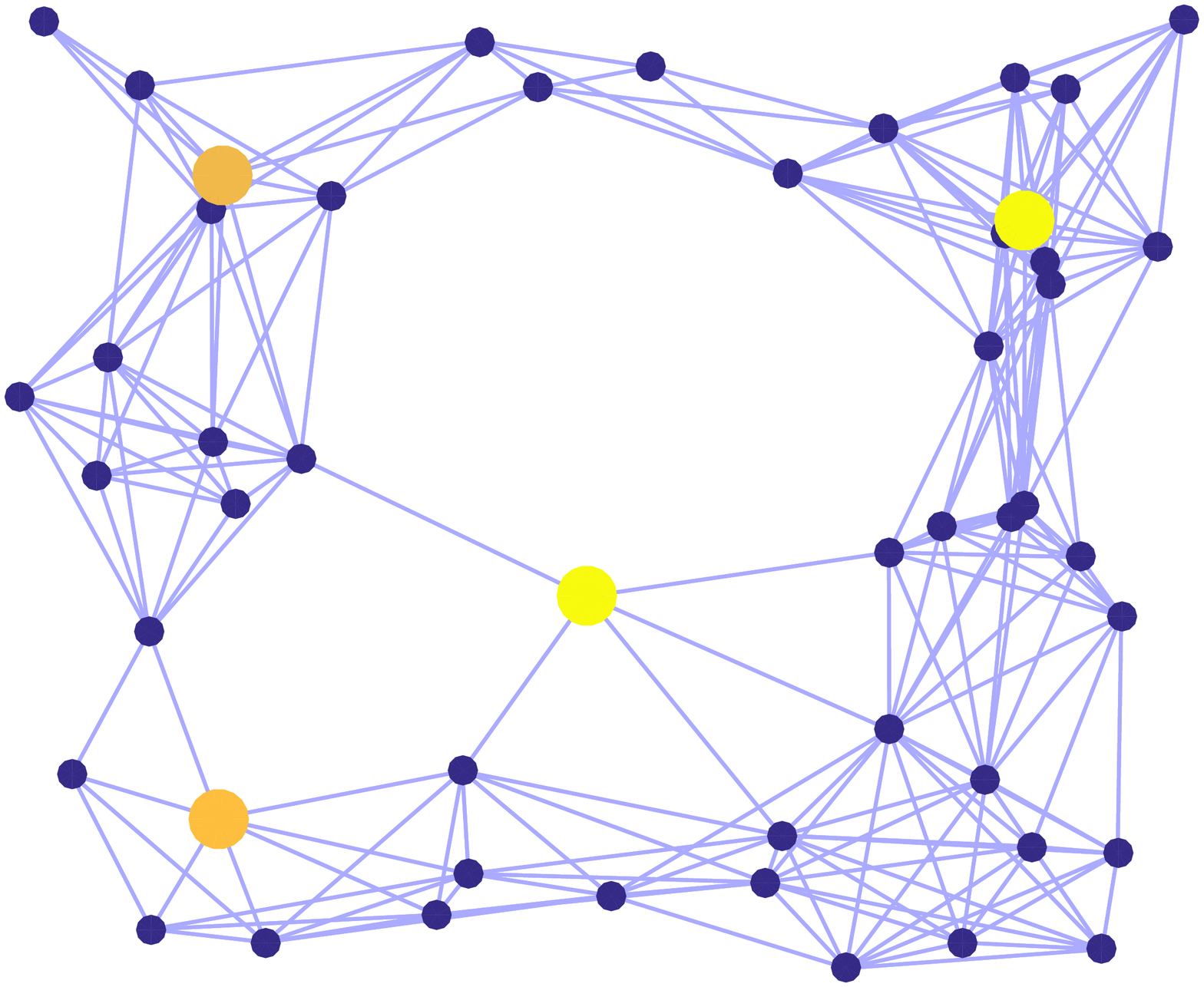}
\\
 {\small (a) clustered graph signal.} & {\small (b) unclustered graph signal.} 
 	\end{tabular}
  \end{center}
   \caption{\label{fig:cluster} Sparse graph signal. Colored nodes indicate large nonzero elements. }
\end{figure}
To solve this, we propose a novel measure to quantify the localization of a graph signal. We use \textit{energy concentration ratio} (ECR) to quantify the energy concentration property. The ECR is defined as 
\begin{equation*}
\rm{ECR} = \frac{S^*}{N},
\end{equation*}
where $S^*$ is the smallest $S$ that satisfies
$
\left\| \x_S \right\|_2^2 \geq 95\%  \left\| \x \right\|_2^2,
$
with $\x_S$ the first $S$ elements with the largest magnitude in $\x$. This indicates that 95\% energy of a graph signal is concentrated in the first $S^*$ elements with largest magnitude. ECR ranges from 0 to 1: when the signal is energy-concentrated, the ECR is small; when the energy of the signal is evenly distributed, the ECR is 1.

We next use \textit{normalized geodesic distance} (NGD) to quantify the clustered property.
Let $\M$ be the set of nodes that possesses 95\% energy of the whole signal. The normalized geodesic distance is defined as:
\begin{equation*}
	{\rm{NGD}} = \frac{1}{D}\frac{\sum_{ i, j \in \M, i \neq j}d(v_i,v_j)}{n(n-1)/2},
\end{equation*}
where $D$ is the diameter of the graph, $d(v_i,v_j)$ is the geodesic distance between nodes $i$ and $j$. Here we use the normalized average geodesic distance as a measure to determine whether the nodes are highly connected. We use the average geodesic distance instead of the largest geodesic distance to avoid the influence of outliers. The NGD ranges from 0 to 1: when the nodes are clustered in a small subgraph, the NGD is small; when the nodes are dispersive, the NGD is large.

We use ECR and NGD together to determine the localization of graph signals. When the two measures are small, the energy of the signal is concentrated in a small set of highly connected nodes, which can be interpreted as localization.

Each graph Fourier basis vector is regarded as a graph signal. When most basis vectors in the graph Fourier basis are localized, that is, the corresponding ECRs and NGDs are small, we call that graph Fourier basis localized.

We now investigate the localization phenomenon of graph Fourier bases of several real-world networks, including the arXiv general relativity and quantum cosmology (GrQc) collaboration network~\cite{LeskovecKF:07}, arXiv High Energy Physics - Theory (Hep-Th) collaboration network~\cite{LeskovecKF:07} and the Facebook `friend circles' network~\cite{LeskovecKF:07}. We find similar localization phenomena among different datasets. Due to the limited space, we only show the result of arXiv GrQc collaboration network. The arXiv GrQc network represents the collaborations between authors based on the submitted papers in general relativity and quantum cosmology category of arXiv. When the author $i$ and author $j$ coauthored a paper, the graph contains an undirected edge between node $i$ and $j$. The graph contains 5242 nodes and 14496 undirected edges. Since the graph is not connected, we choose a connect component with 4158 nodes and 13422 edges.

We investigate the Fourier bases of the weighted adjacency matrix, the transition matrix  and the unnormalized  graph Laplacian matrix in the arXiv GrQc network. 
 Figure~\ref{fig:GRQC_localization} illustrates the ECRs and NGDs of the first 50 graph Fourier basis vectors (low-frequency components) and the last 50 graph Fourier basis vectors (high-frequency components) of the three graph representation matrices, where the ECR and NGD are plotted as a function of the index of the corresponding graph Fourier basis vectors. We find that a large number of graph Fourier basis vectors has small ECRs and NGDs, which indicates that graph Fourier basis vectors of various graph representation matrices are localized. Among various graph representation matrices, the graph Fourier basis vectors of graph Laplacian matrix tend to be more localized, especially in high-frequency components. In low-frequency components, the graph Fourier basis vectors of adjacency matrix are more localized.

To combine the uncertainty principle previously, when the graph Fourier basis shows localization phenomenon, it is possible that a graph signal can be localized in both graph vertex and spectrum domain. Based the graph Fourier basis of the graph Laplacian, a high-frequency bandlimited signals can be well localized in the graph vertex domain; based on the graph Fourier basis of adjacency matrix, a low-frequency bandlimited signals can be well localized in the graph vertex domain.  The Fourier transform is famous for capturing the global behaviors and works as a counterpart of the delta functions; however, this may not be true on graphs. The study of new role of the graph Fourier transform will be an interesting future work.

\begin{figure}[htb]
  \begin{center}
    \begin{tabular}{cc}
     {\small low-frequency components} & {\small high-frequency components} 
     \\
 \includegraphics[width=0.45\columnwidth]{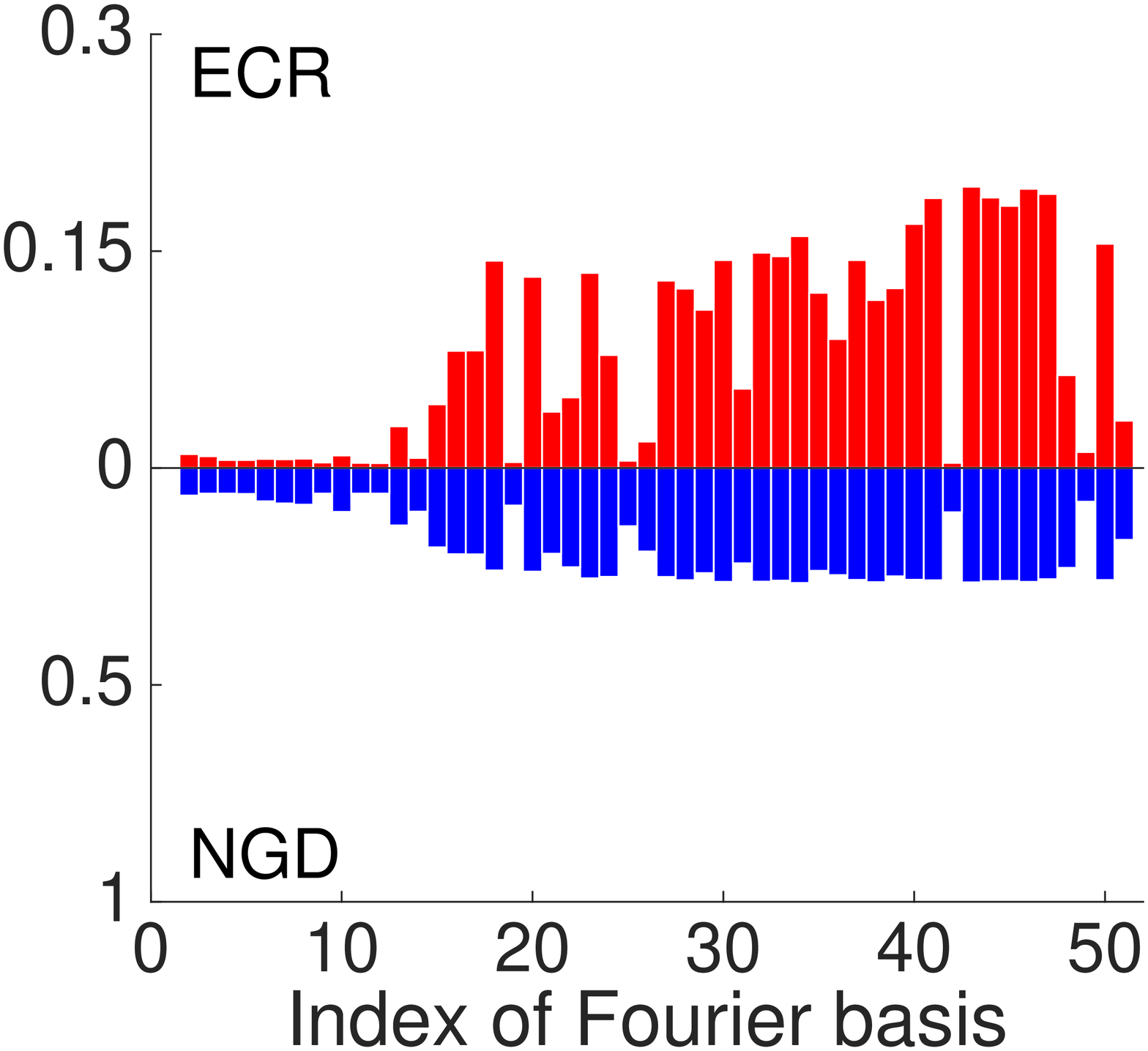} &
 \includegraphics[width=0.45\columnwidth]{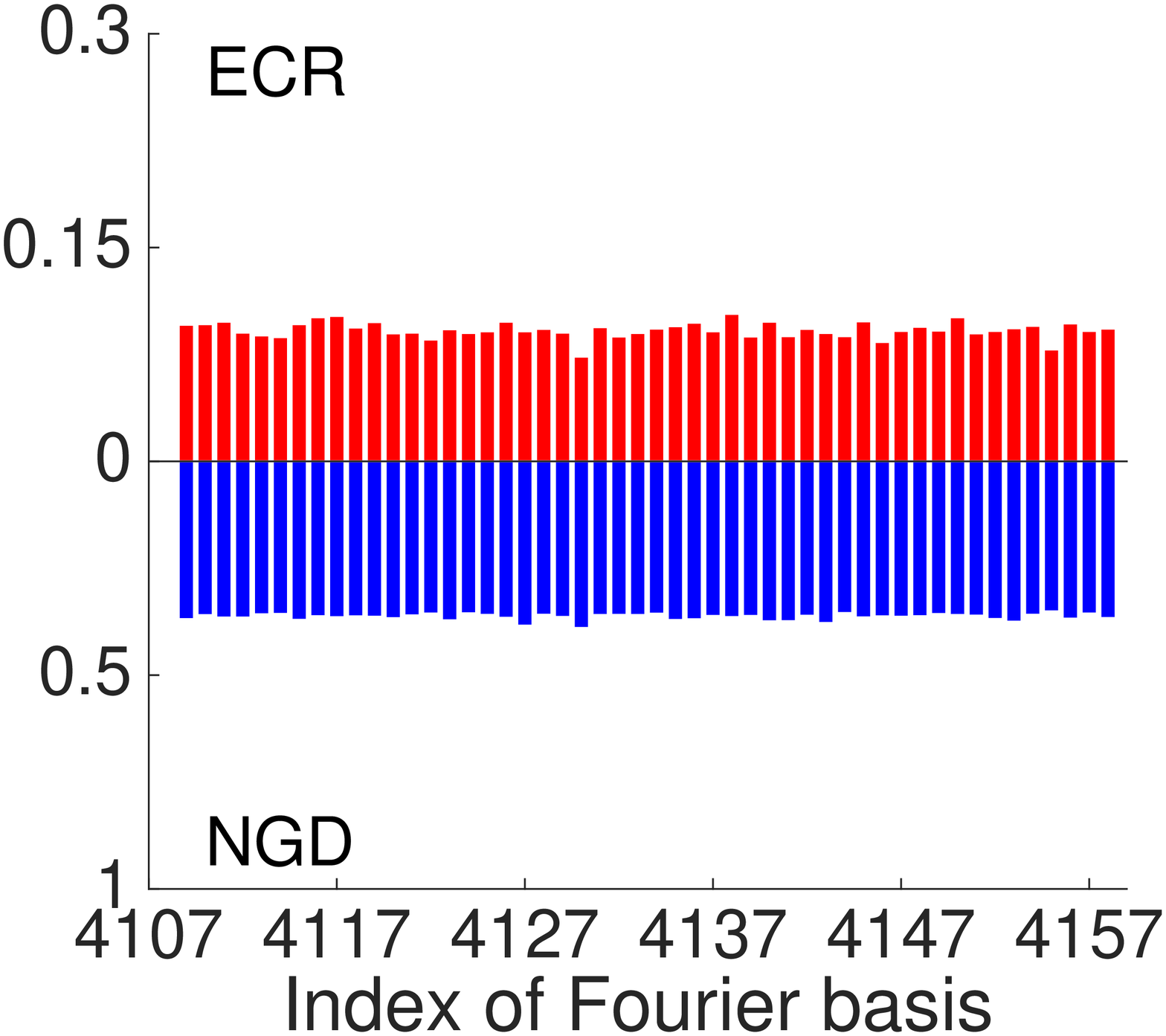}
\\
 {\small (a) $\W$.} & {\small (b) $\W$.} 
 \\
  \includegraphics[width=0.45\columnwidth]{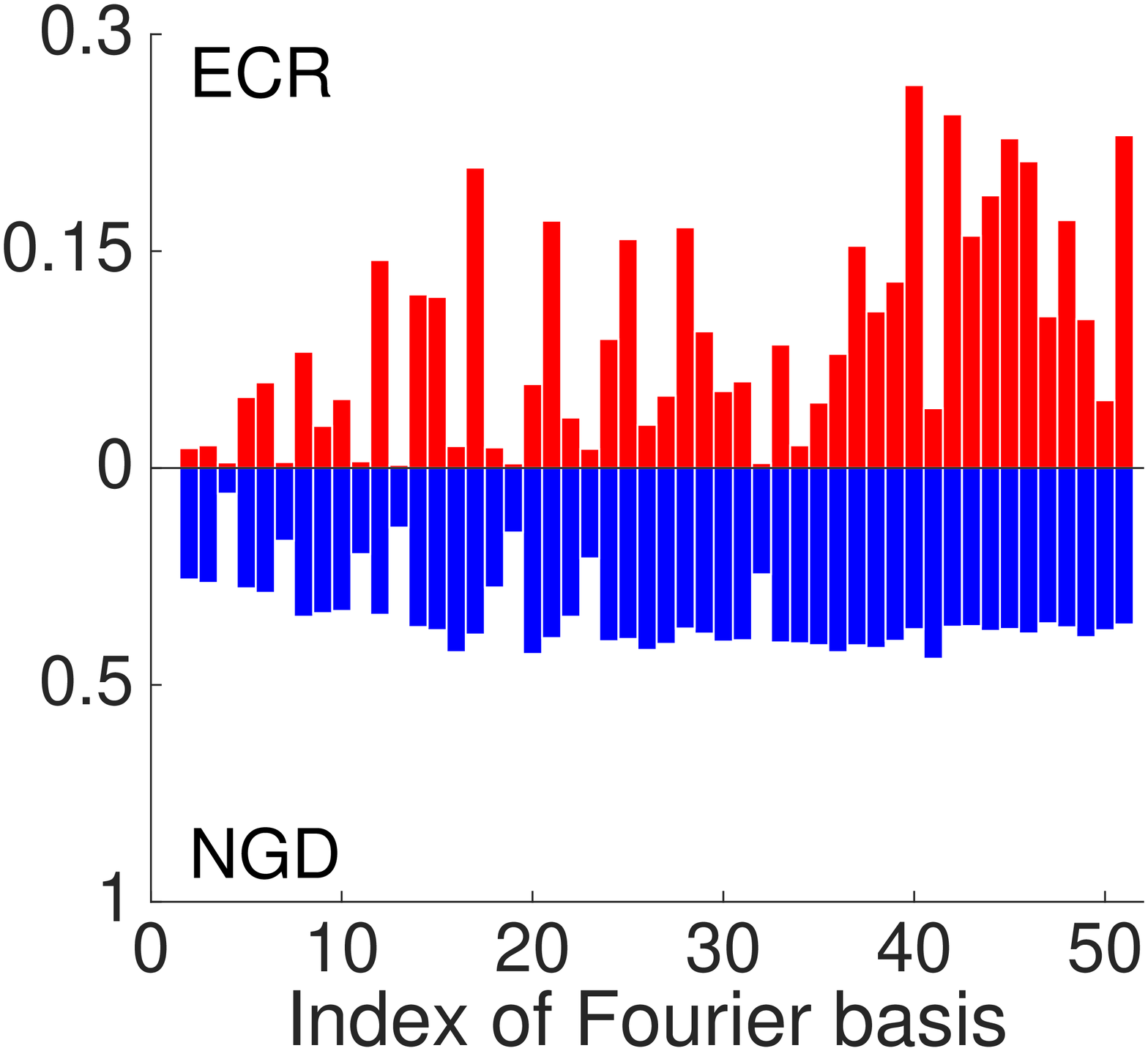} 
&
 \includegraphics[width=0.45\columnwidth]{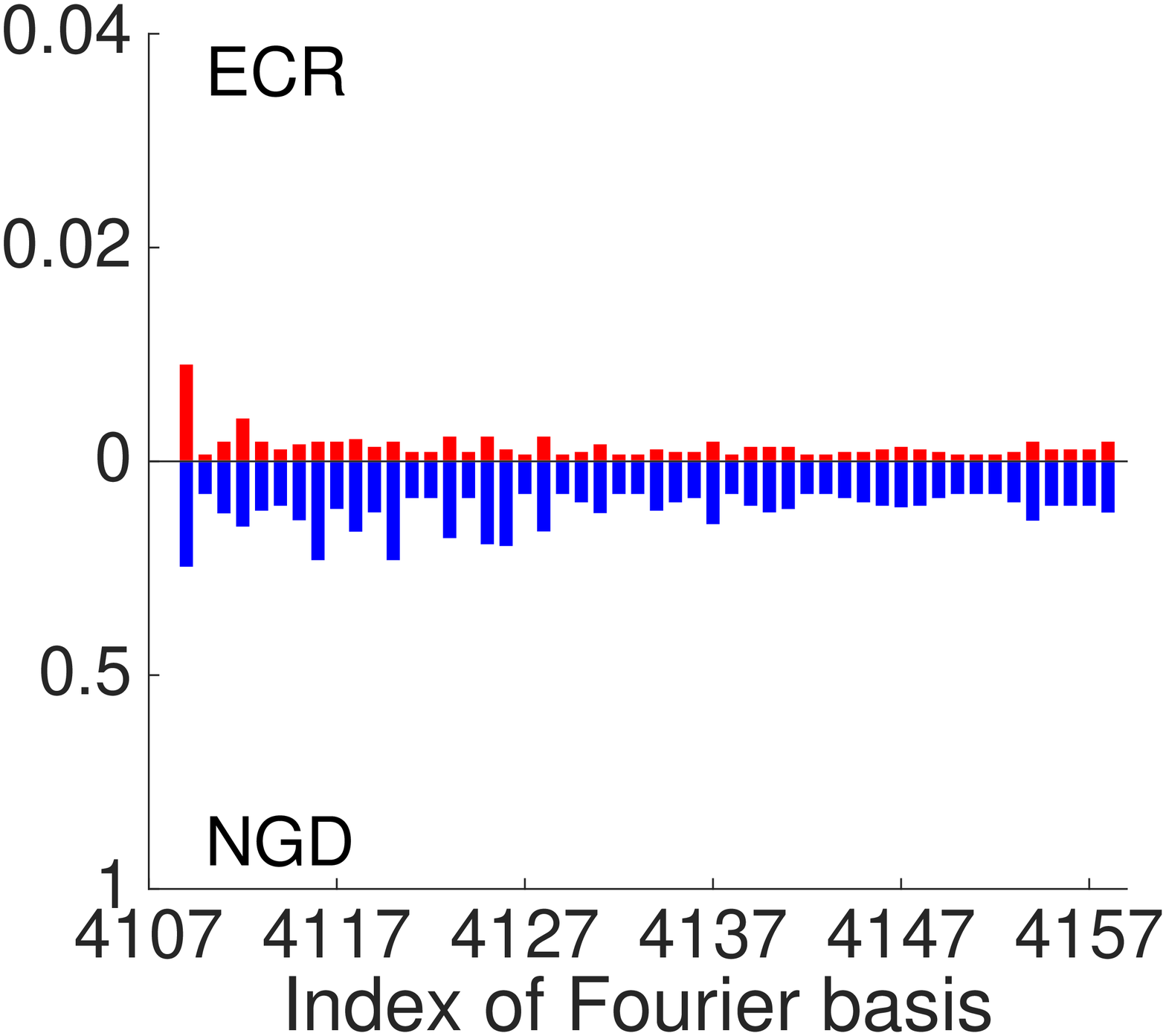}
\\
 {\small (c) $\Pj$.} & {\small (d) $\Pj$.}  
  \\
  \includegraphics[width=0.45\columnwidth]{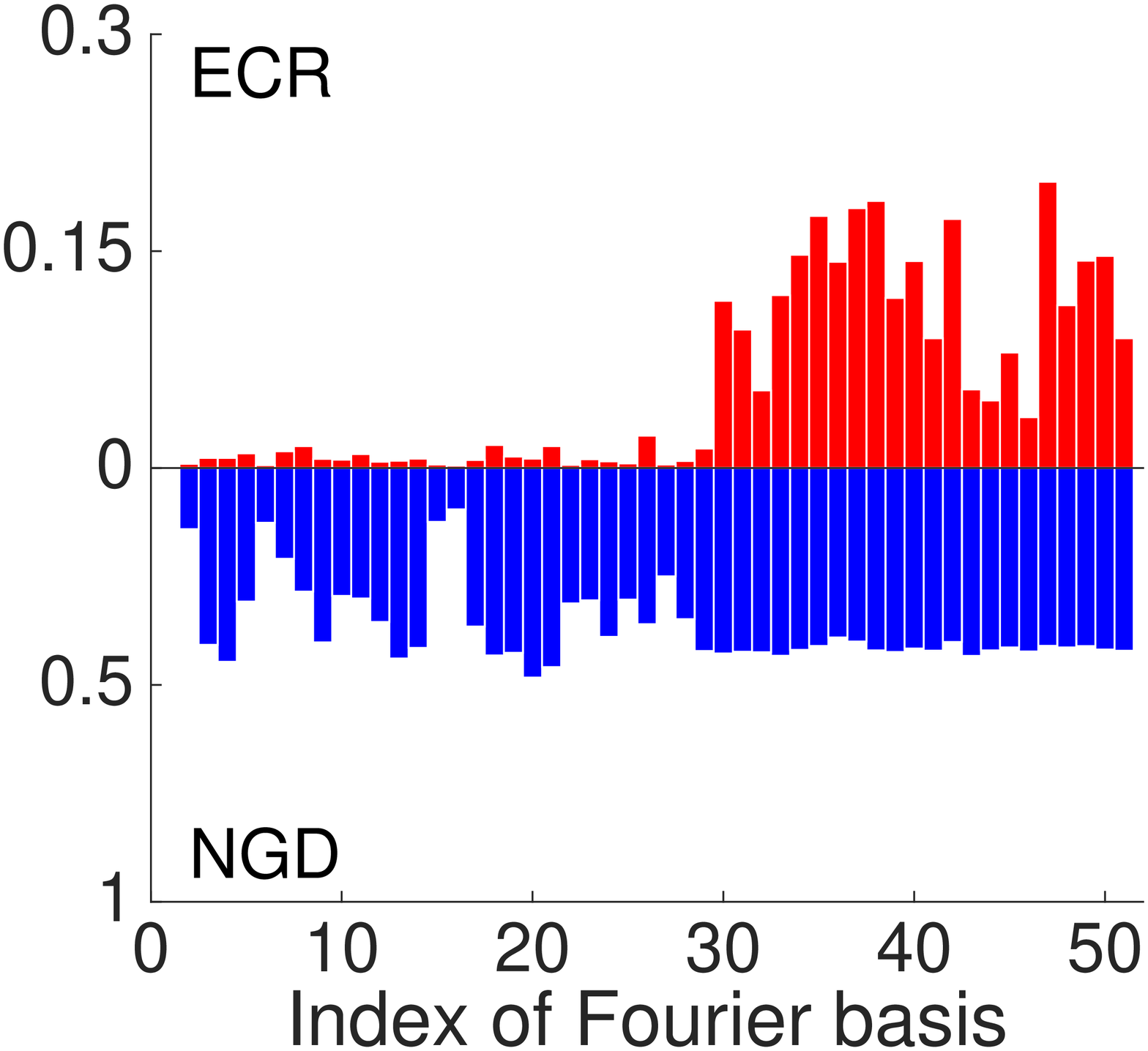} 
&
 \includegraphics[width=0.45\columnwidth]{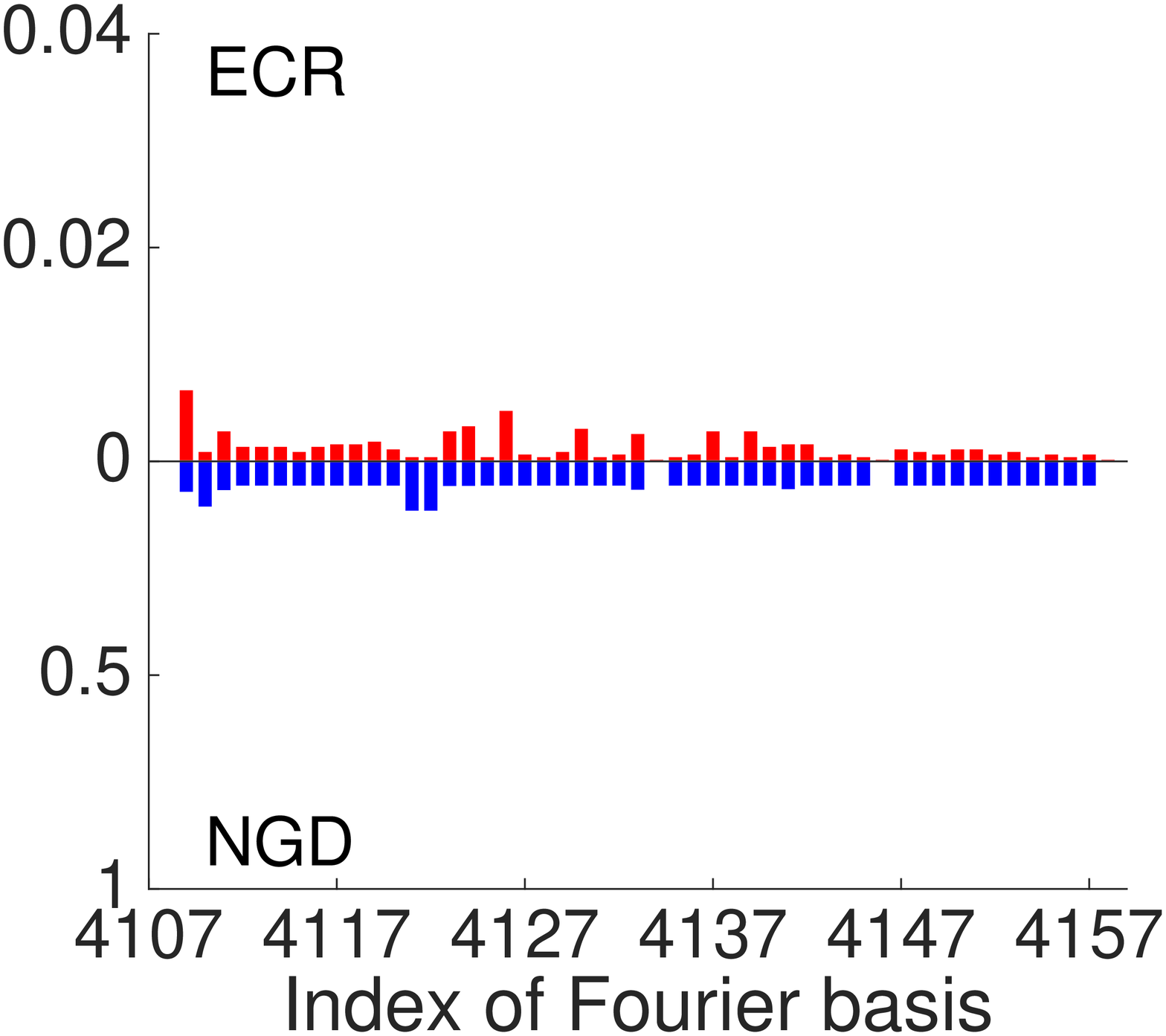}
\\
 {\small (e) $\LL$.} & {\small (f) $\LL$.} 
	\end{tabular}
  \end{center}
   \caption{\label{fig:GRQC_localization} Localization of graph Fourier bases of various graph representation matrices in the arXiv GrQc network. In low-frequency components, Fourier basis vectors of the adjacency matrix are localized; in high-frequency components, Fourier basis vectors of graph Laplacian matrix are localized.} 
\end{figure}

\subsection{Graph Signal Processing Tasks}
As mentioned in Section~\ref{sec:template_tasks}, we  focus on two tasks: approximation and sampling following with recovery.

\subsubsection{Approximation}
\label{sec:smooth_app}
We compare the graph Fourier bases based on different graph structure matrices.

\mypar{Algorithm}
We consider nonlinear approximation for the graph Fourier bases, that is, after expanding with a representation, we should choose the $K$ largest-magnitude expansion coefficients so as to minimize the approximation error. Let $\{ \phi_k \in \R^N \}_{k=1}^N$ and  $\{ \widehat{\phi}_k \in \R^N \}_{k=1}^N$ be a pair of biorthonormal basis and $\x \in  \R^N $ be a signal. Here the graph Fourier transform matrix $\Um = \{ \widehat{\phi}_k \}_{k=1}^N$ and the graph Fourier basis $\Vm = \{ \phi_k \}_{k=1}^N$. The nonlinear approximation to $\x$ is 
\begin{eqnarray}
\label{eq:nonlinear_approx}
   \x^* = \sum_{k \in \mathcal{I}_K} \left\langle {\x, \widehat{\phi}_k} \right\rangle \phi_k,
\end{eqnarray}
where $\mathcal{I}_K$ is the index set of the $K$ largest-magnitude expansion coefficients. When a basis promotes sparsity for $\x$, only a few expansion coefficients are needed to obtain a small approximation error. Note that~\eqref{eq:nonlinear_approx} is a special case of~\eqref{eq:sparse_coding} when the distance metric $d(\cdot, \cdot)$ is the $\ell_2$ norm and $\D$ is a basis.

Since the the graph polynomial dictionary is redundant, we solve the following sparse coding problem,
\begin{eqnarray}
\label{eq:l2_sparse_coding}
   \x^* =  &  \arg \min_{\a} & \left\| \x - \D_{\rm poly(2)} \a  \right\|_2^2,
   \\
   \nonumber
   & {\rm subject~to:~} &  \left\| \a \right\|_0 \leq K,
\end{eqnarray}
where $\D_{\rm poly(2)}$ is the graph polynomial dictionary with order $2$ and $\a$ are expansion coefficients. The idea is to use a linear combination of a few atoms from $\D_{\rm poly(2)}$ to approximate the original signal. When $\D$ is an orthonormal basis, the closed-form solution is exactly~\eqref{eq:nonlinear_approx}. We solve~\eqref{eq:sparse_coding} by using the orthogonal matching pursuit, which is a greedy algorithm~\cite{PatiRK:93}. Note that~\eqref{eq:l2_sparse_coding} is a special case of~\eqref{eq:sparse_coding} when the distance metric $d(\cdot, \cdot)$ is the $\ell_2$ norm.

\mypar{Experiments}
We test the four representations on two datasets, including the Minnesota road graph~\cite{MinnesotaGraph} and the U.S city graph~\cite{ChenSMK:14}.

\begin{figure}[htb]
  \begin{center}
    \begin{tabular}{cc}
     \includegraphics[width=0.4\columnwidth]{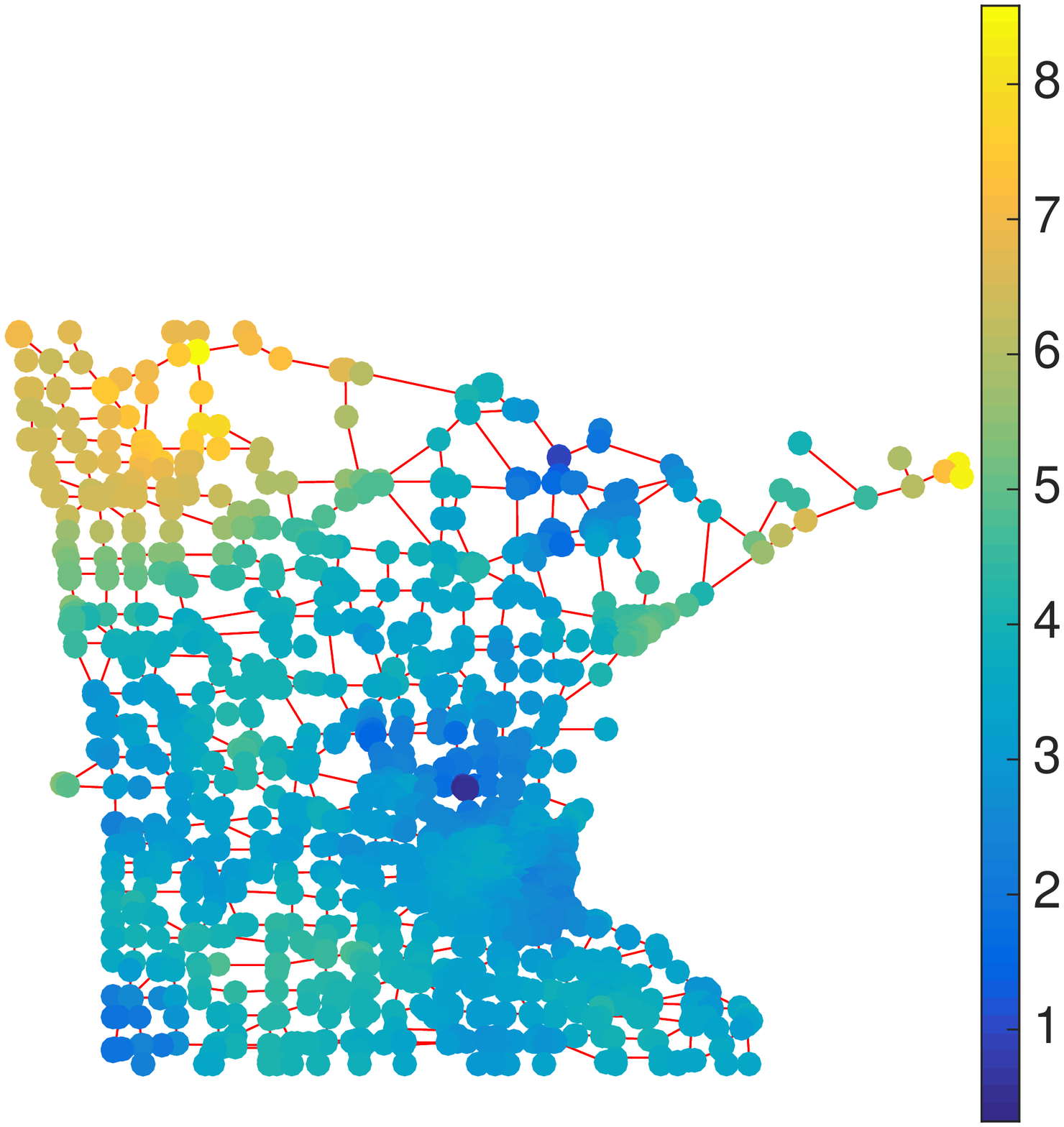} &
 \includegraphics[width=0.4\columnwidth]{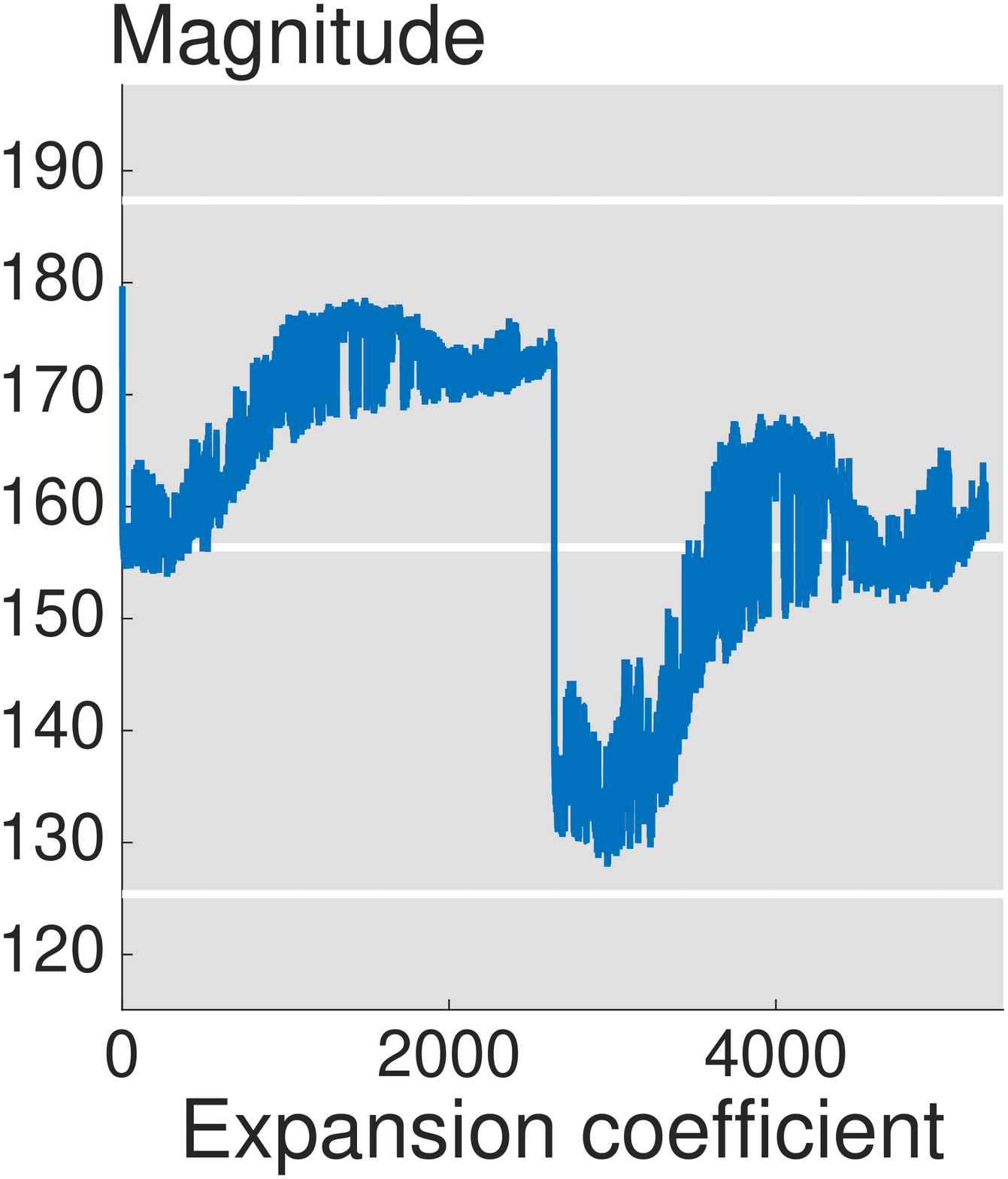} 
 \\
  {\small (a) Wind speed.} & {\small (b) Coefficients of $\D_{\rm poly(2)}$.} 
  \\
  \\
 \includegraphics[width=0.4\columnwidth]{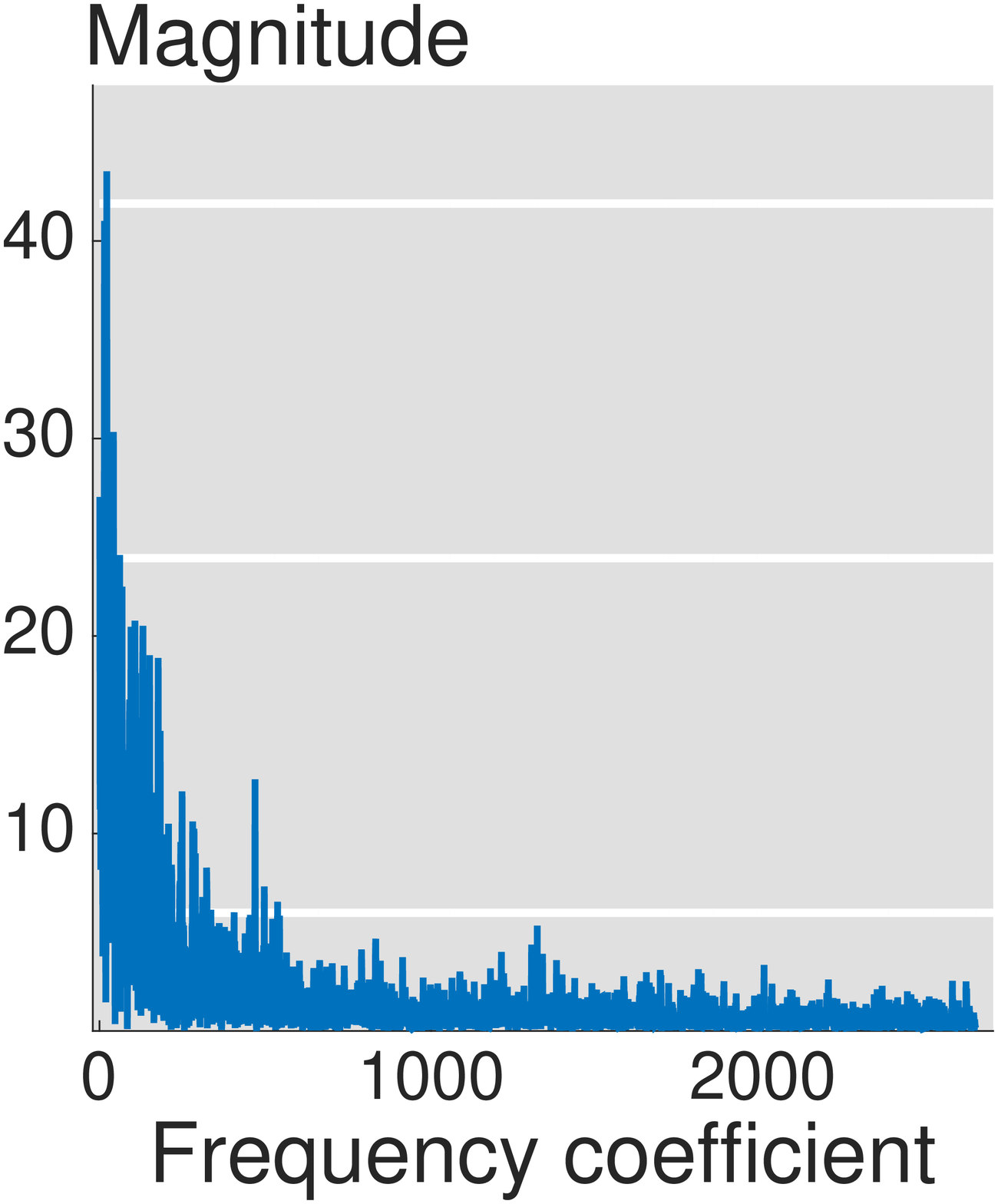} &
 \includegraphics[width=0.4\columnwidth]{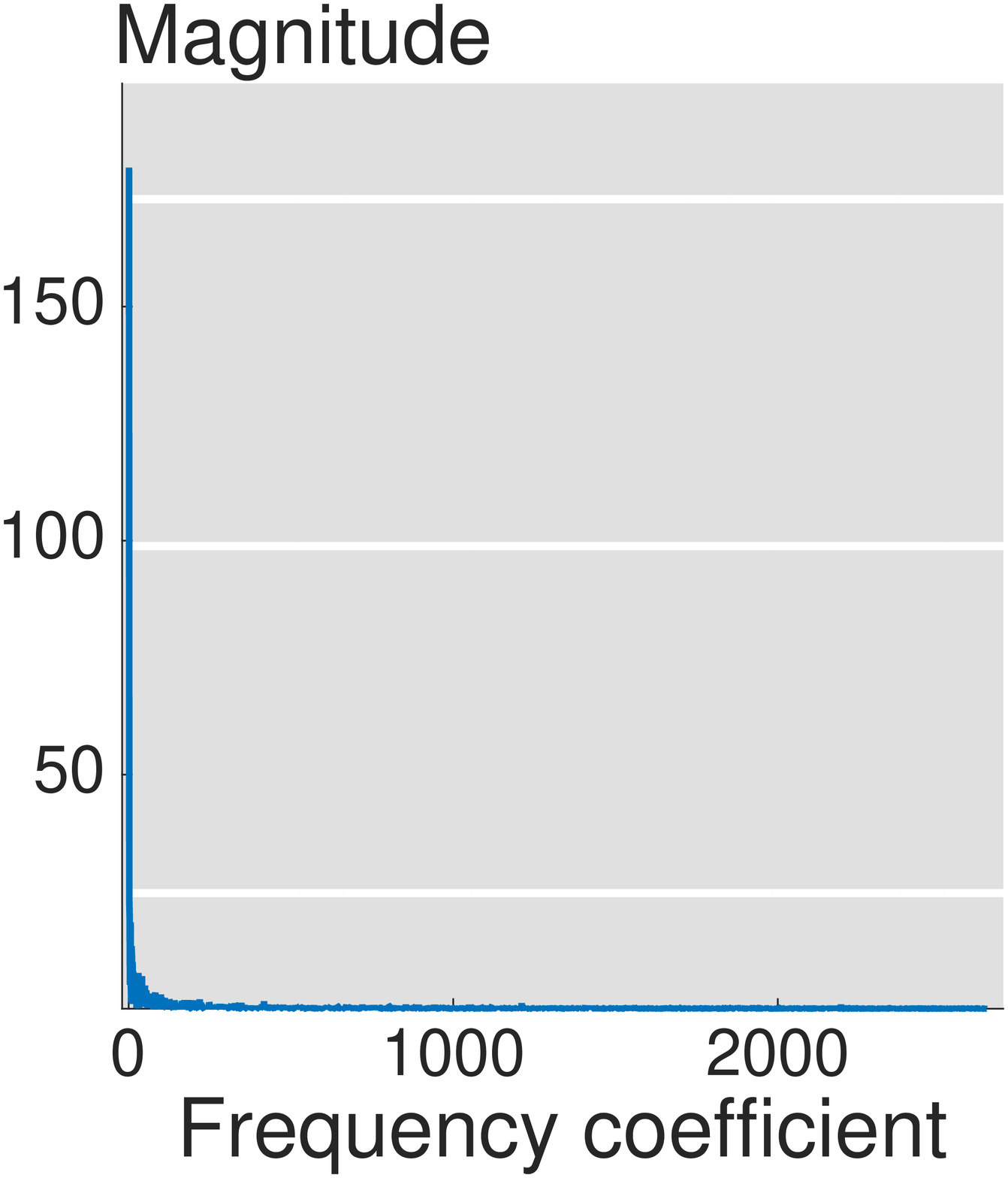} 
 \\
  {\small (c) Coefficients of $\Vm_{\W}$.} & {\small (d) Coefficients of $\Vm_{\LL}$.} 
 \\
 \\
 \includegraphics[width=0.4\columnwidth]{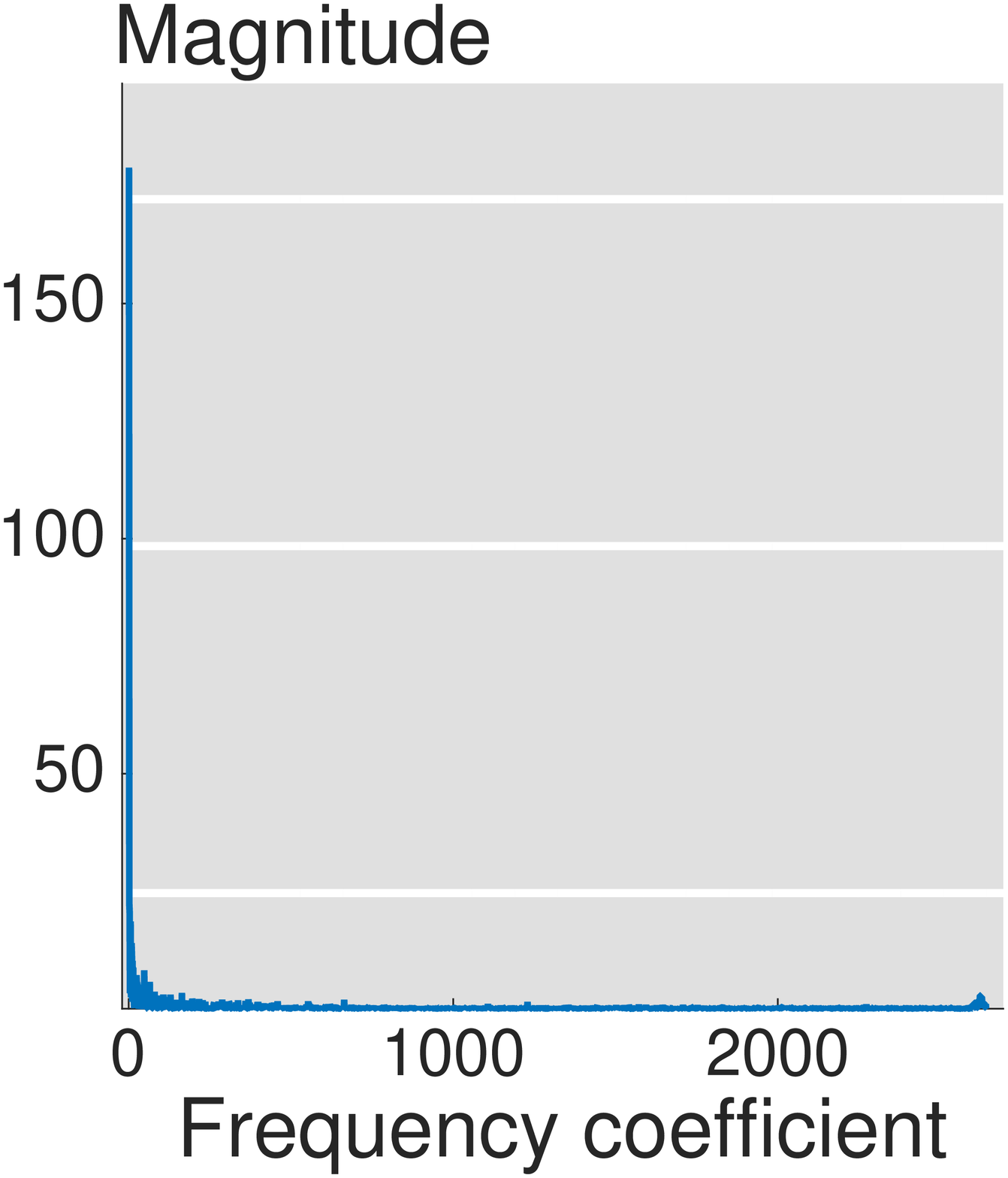} &
 \includegraphics[width=0.4\columnwidth]{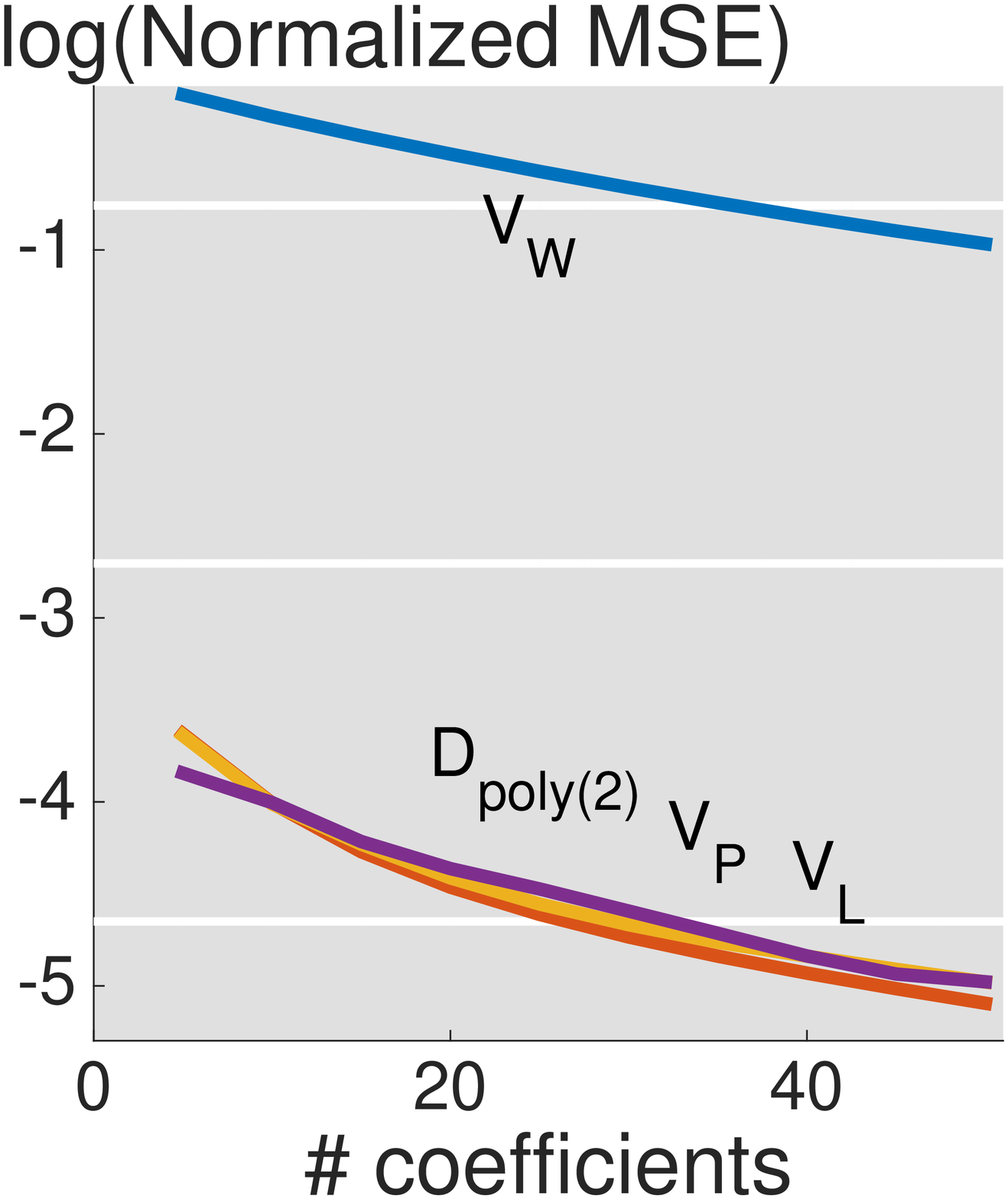}  
 \\
  {\small (e) Coefficients of $\Vm_{\Pj}$.} & {\small (f) Approximation error.} 
\end{tabular}
  \end{center}
   \caption{\label{fig:smooth_app_wind} Approximation of wind speed. $\Vm_{\LL}$ and $\D_{\rm poly(2)}$ tie $\Vm_{\Pj}$; all of them are much better than $\Vm_{\W}$. }
\end{figure}

The Minnesota road graph is a standard dataset including 2642 intersections and 3304 roads
 ~\cite{MinnesotaGraph}. we construct a  graph by modeling the intersections as nodes and the roads as undirected edges. We collect a dataset recording the wind speeds at those 2642 intersections~\cite{JiCVK:15}. The data records the hourly measurement of the wind speed and direction at each intersection. In this paper, we present the data of wind speed on January 1st, 2015. Figure~\ref{fig:smooth_app_wind}(a) shows a snapshot of the wind speeds on the entire Minnesota road. The expansion coefficients obtained by using four representations are shown in Figure~\ref{fig:smooth_app_wind}(b), (c), (d) and (e). The energies of the frequency coefficients of $\Vm_{\W}$, $\Vm_{\LL}$ and $\Vm_{\Pj}$ mainly concentrate on the low-frequency bands; $\Vm_{\LL}$ and $\Vm_{\Pj}$ are more concentrated; $\D_{\rm poly(2)}$ is redundant and the corresponding expansion coefficients $\D_{\rm poly(2)}^T \x$ are not sparse.

To make a more serious comparison, we evaluate the approximation error by using the normalized mean square error, that is,
\begin{equation}
\label{eq:NMSE} 
 {\rm Normalized~MSE} = \frac{ \left\|  \x^* - \x \right\|_2^2 }{ \left\| \x \right\|_2^2 },
\end{equation}
where $\x^*$ is the approximation signal and $\x$ is the original signal. Figure~\ref{fig:smooth_app_wind}(f) shows the approximation errors given by the four representations. The x-axis is the number of coefficients used in approximation, which is $K$ in~\eqref{eq:nonlinear_approx} and~\eqref{eq:sparse_coding} and the y-axis is the approximation error, where lower means better. We see that $\Vm_{\LL}$ and $\D_{\rm poly(2)}$ tie $\Vm_{\Pj}$; all of them are much better than $\Vm_{\W}$. This means that the wind speeds on the Minnesota road graph are well modeled by pairwise Lipschitz smooth, total Lipschitz smooth and local normalized neighboring smooth criteria. The global normalized neighboring smooth criterion is not appropriate for this dataset.

\begin{figure}[htb]
  \begin{center}
    \begin{tabular}{cc}
\includegraphics[width=0.4\columnwidth]{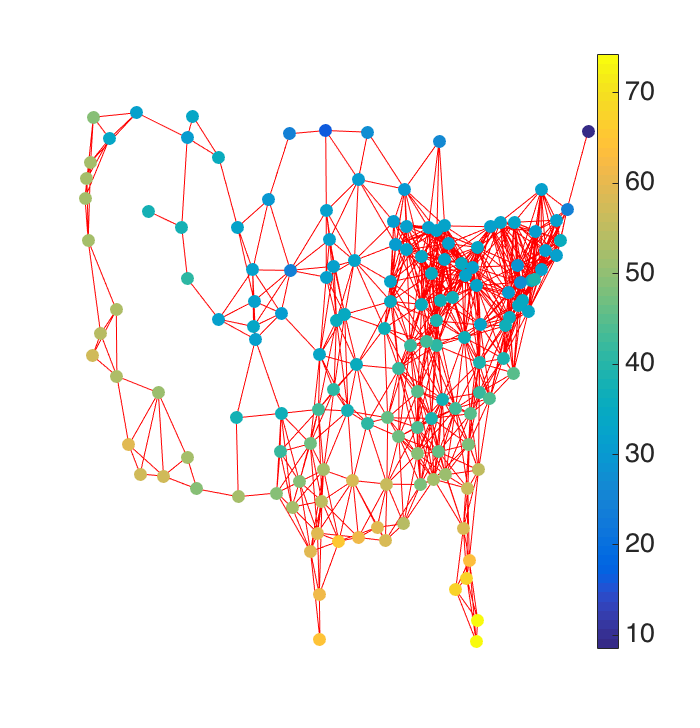} &
 \includegraphics[width=0.4\columnwidth]{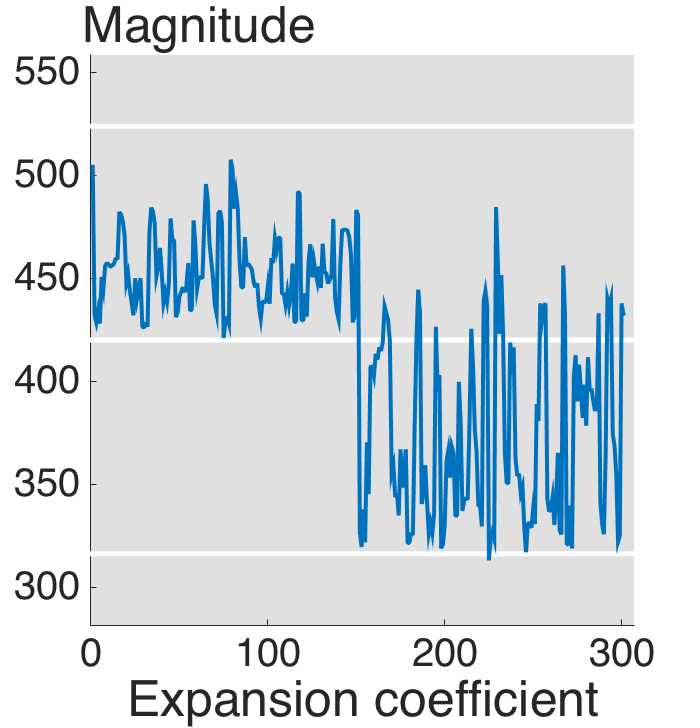} 
 \\
  {\small (a) Temperature.} & {\small (b) Coefficients of $\D_{\rm poly(2)}$.} 
 \\
 \\
 \includegraphics[width=0.4\columnwidth]{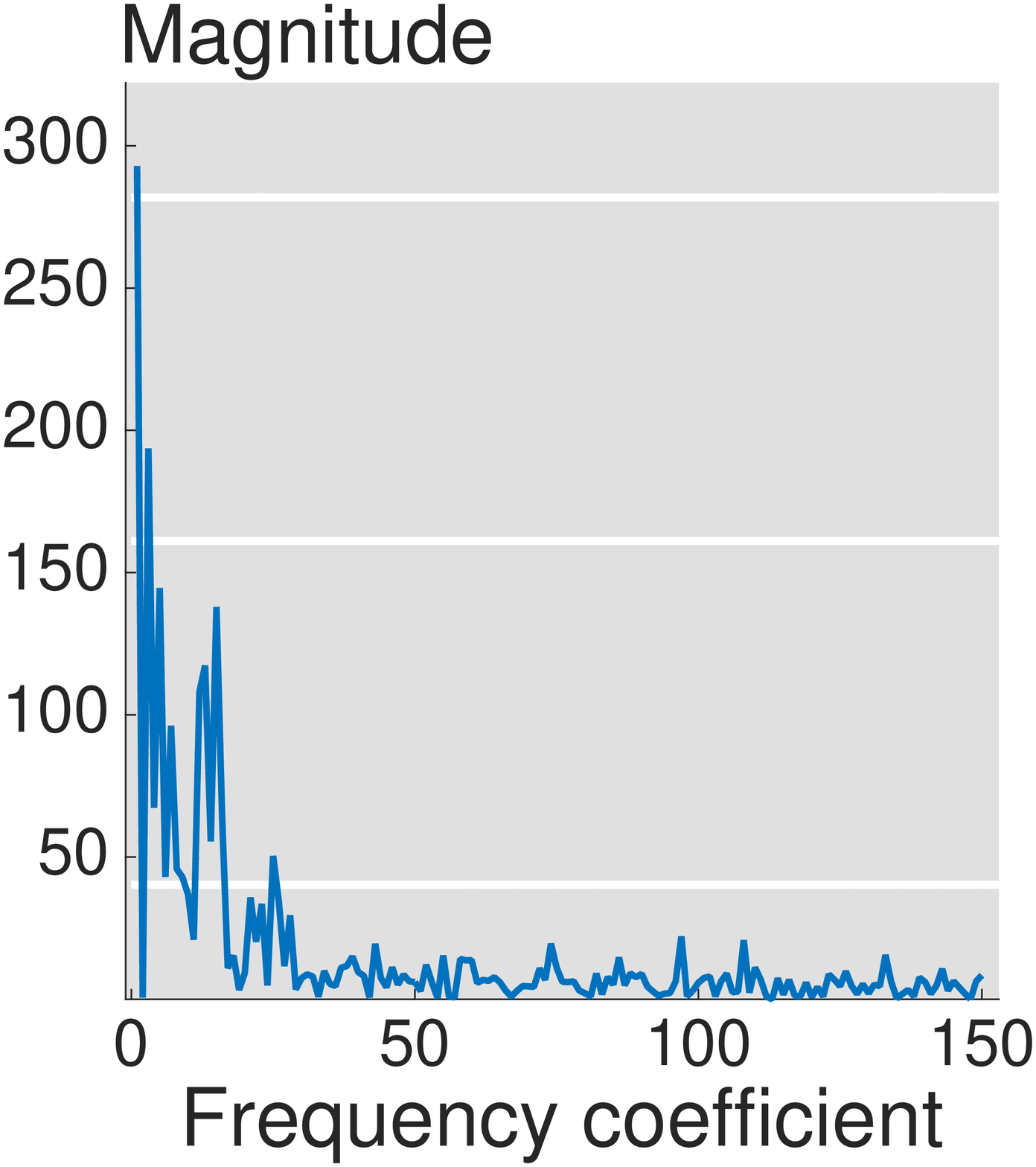} &
 \includegraphics[width=0.4\columnwidth]{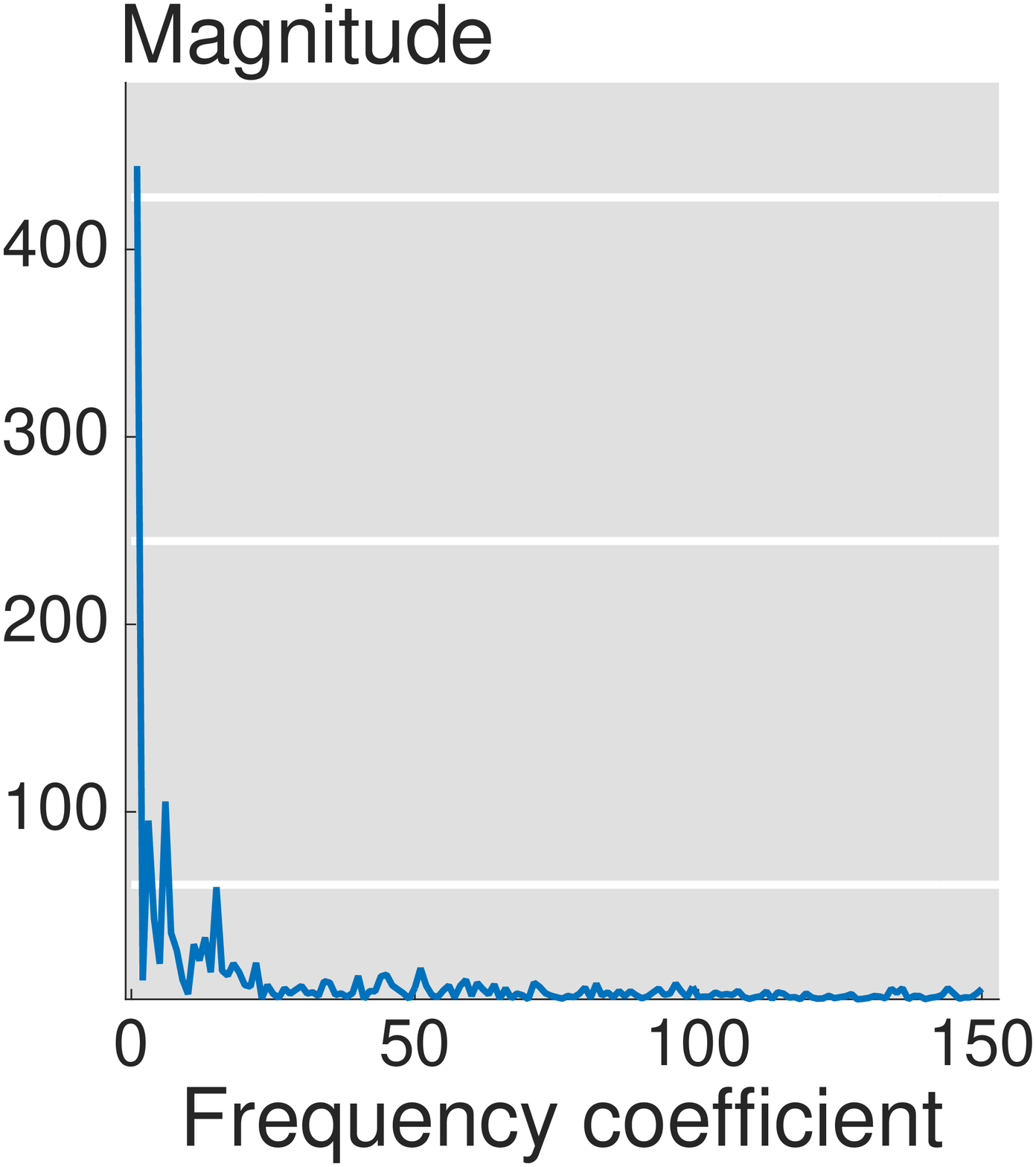} 
 \\
  {\small (c) Coefficients of $\Vm_{\W}$.} & {\small (d) Coefficients of $\Vm_{\LL}$.} 
 \\
 \\
 \includegraphics[width=0.4\columnwidth]{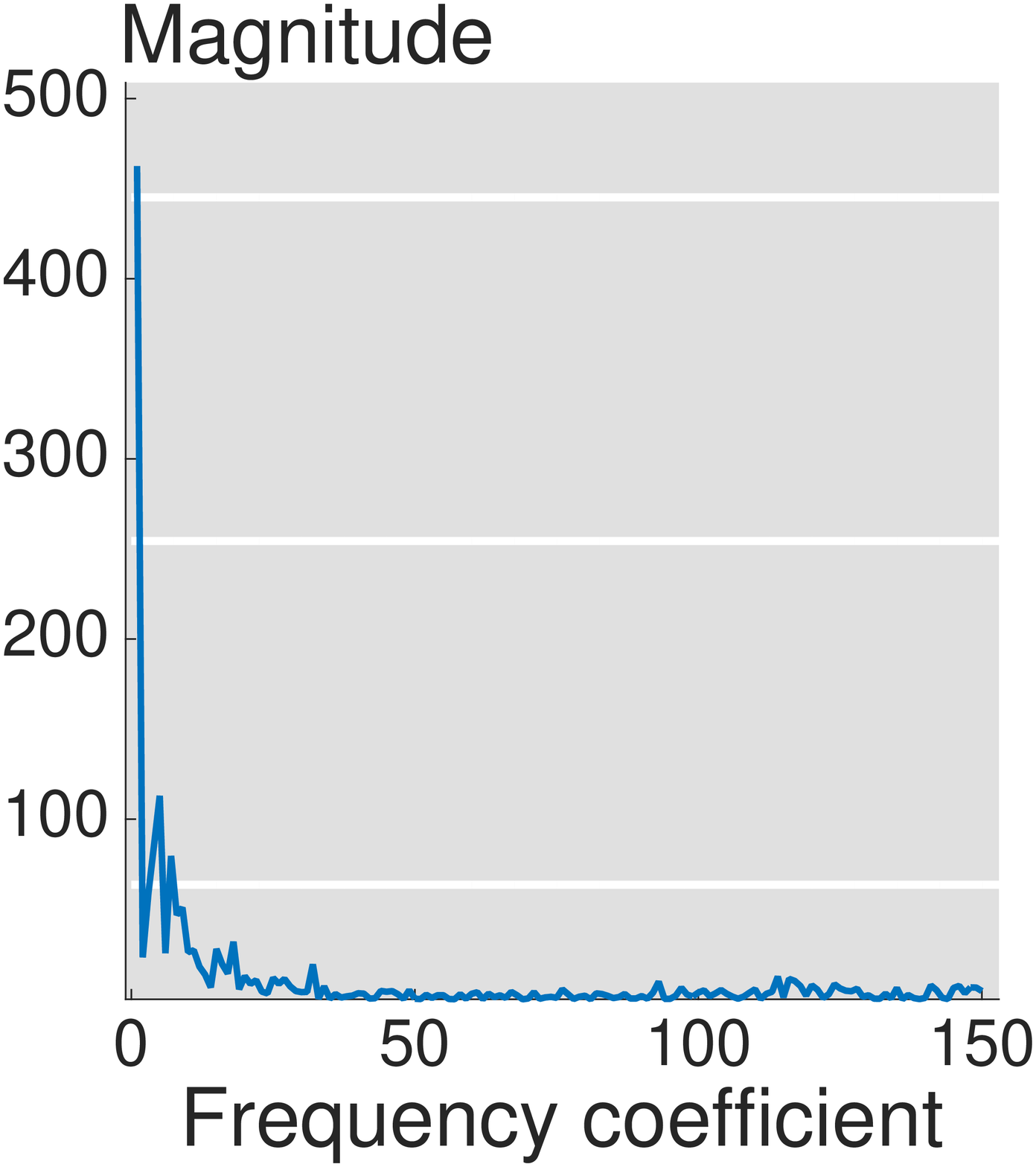} &
 \includegraphics[width=0.4\columnwidth]{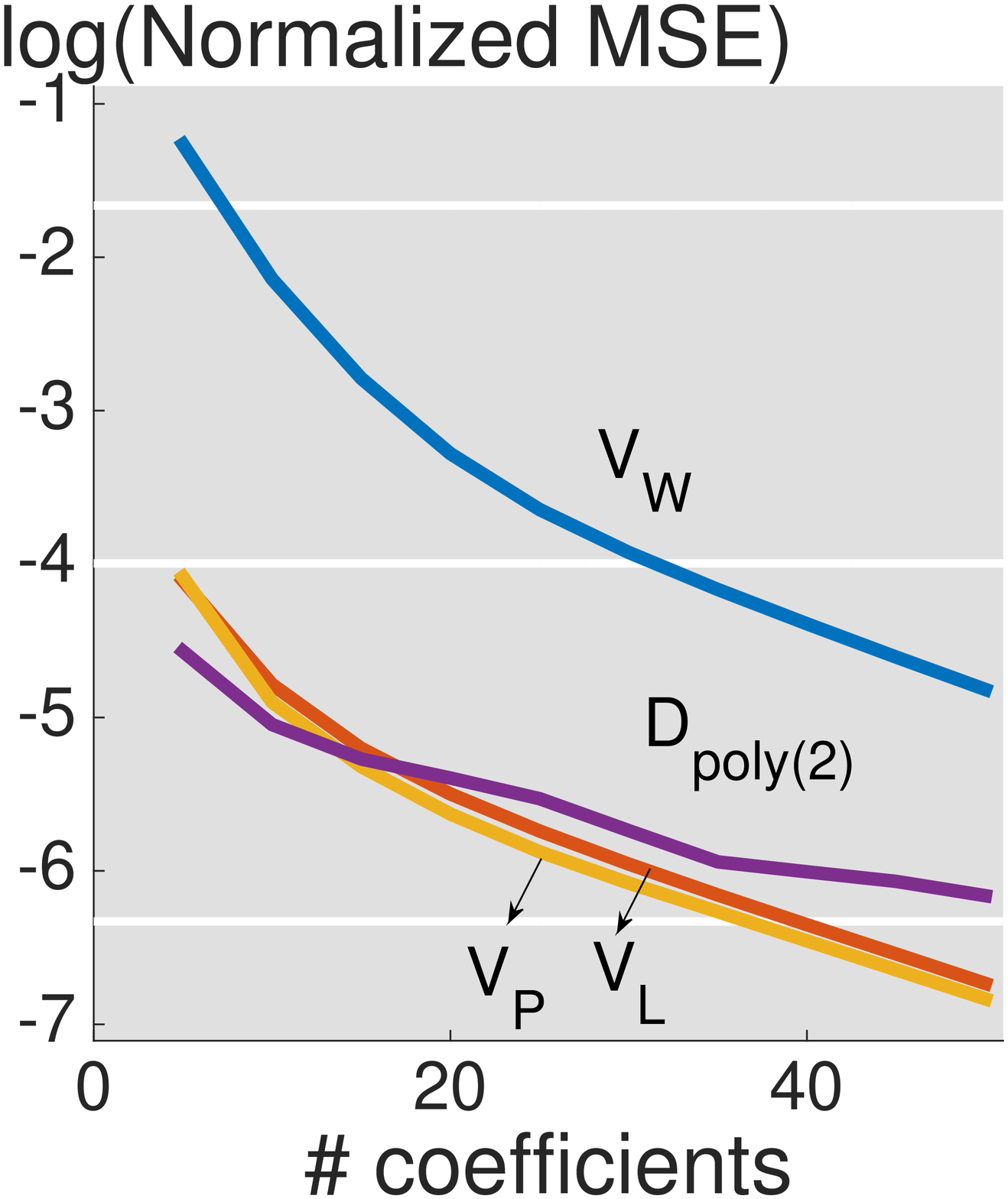}  
 \\
  {\small (e) Coefficients of $\Vm_{\Pj}$.} & {\small (f) Approximation err.} 
\end{tabular}
  \end{center}
   \caption{\label{fig:smooth_app_temp} Approximation of temperature. $\Vm_{\LL}$ ties $\Vm_{\Pj}$; both are slightly better than $\D_{\rm poly(2)}$ and are much better than $\Vm_{\W}$. }
\end{figure}

The U.S weather station graph is a network representation of $150$ weather stations across the U.S. We assign an edge when two 
weather stations are within 500 miles. The graph includes $150$ nodes and $1033$ undirected, unweighted edges. Each weather station has $365$ days of recordings (one recording per day), for a total of 365 graph signals. As an example, see Figure~\ref{fig:smooth_app_temp}(a). The expansion coefficients obtained by using four representations are shown in Figure~\ref{fig:smooth_app_temp}(b), (c), (d) and (e). 
Similarly to the wind speed dataset, the energies of the frequency coefficients of $\Vm_{\W}$, $\Vm_{\LL}$ and $\Vm_{\Pj}$ are mainly concentrated on the low-frequency bands; $\Vm_{\LL}$ and $\Vm_{\Pj}$ are more concentrated; $\D_{\rm poly(2)}$ is redundant and the corresponding expansion coefficients $\D_{\rm poly(2)}^T \x$ are not sparse.

The evaluation metric of the approximation error is also the normalized mean square error.  Figure~\ref{fig:smooth_app_temp}(f) shows the approximation errors given by the four representations. The results are averages over 365 graph signals. Again, we see that $\Vm_{\LL}$, $\D_{\rm poly(2)}$ and $\Vm_{\Pj}$ perform similarly; all of them are much better than $\Vm_{\W}$. This means that the wind speeds on the Minnesota road graph are well modeled by pairwise Lipschitz smooth, total Lipschitz smooth and local normalized neighboring smooth criteria. The global normalized neighboring smooth criterion is not appropriate for this dataset.

The results from two real-world datasets suggest that we should consider using pairwise Lipschitz smooth, total Lipschitz smooth and local normalized neighboring smooth criteria to model real-world smooth graph signals. In terms of the representation dictionary, among $\Vm_{\LL}$, $\D_{\rm poly(2)}$ and $\Vm_{\Pj}$, $\D_{\rm poly(2)}$ is redundant; $\Vm_{\Pj}$ is not orthonormal. We thus prefer using $\Vm_{\LL}$ because it is an orthonormal basis.

\subsubsection{Sampling and Recovery}
\label{sec:smooth_samplingandrecovery}
The goal is to recover a smooth graph signal from a few subsamples. A subsample is collected from an individual node each time; that is, we constraint the sampling pattern matrix $\F$ in~\eqref{eq:Psi} be an identity matrix. Previous works show that in this senario, experimentally designed sampling is equivalent to active sampling and is better than uniform sampling asymptotically~\cite{ChenVSK:15c}. Here we compare various sampling strategies based on experimentally designed sampling, which are implemented in a deterministic approach. The random approach sees~\cite{ChenVSK:15a}.

\mypar{Algorithm}
We follow the sampling and recovery framework in Section~\ref{sec:samplingandrecovery}.
Let a smooth graph signal be $\x = \D \a$. For example, $\D = \Vm_{\LL}$, then $\a$ are the frequency coefficients. In general, we assume that the energy of the expansion coefficient $\a$ is concentrated in a few known supports, that is, $\left\| \a_{\Omega} \right\|_2 \gg \left\| \a_{\Omega^c} \right\|_2$, where $|\Omega| \ll N $ and $\Omega$ is the known. We aim to recover $\a_{\Omega}$ and further approximate $\x$ by using $\D_{\Omega}\a_{\Omega}$. We consider the partial least squares (PLS), $x^*_{\rm PLS}  \ = \  \D_{\Omega}  \a^*_{\rm PLS} $, where
\begin{eqnarray}
\label{eq:PLS}
  \a^*_{\rm PLS}  & = & 	
	  \arg \min_{ \a}  \left\|  \Psi^T \x_{\Psi} -  \Psi^T \Psi  \D_{\Omega} \a_{\Omega} \right\|_2^2
	\nonumber
	\\
	& = &   \left(  \D_{\Omega}^T \Psi^T \Psi \D_{\Omega} \right)^{-1}  \D_{\Omega} \Psi^T \Psi \y
	\nonumber
	\\
	\nonumber
	& = & (\Psi \D_{\Omega} )^{\dagger} \Psi ( \D_{\Omega} \a_{\Omega} + \D_{\Omega^c} \a_{\Omega^c}  + \epsilon ),
\end{eqnarray}
where $\y = \x + \epsilon$ is the noisy version of the graph signal with $\epsilon$ is noise, $(\Psi \D_{\Omega} )^{\dagger} \Psi = \left(  \D_{\Omega}^T \Ss  \D_{\Omega} \right)^{-1}  \D_{\Omega}^T \Ss$ and $\Ss =  \Psi^T \Psi$, which is a diagonal matrix with $\Ss_{i,i}= 1$, when the $i$th node is sampled, and 0, otherwise. The recovery error is then
\begin{eqnarray*}
&&   \x^*_{\rm PLS}  - \x
\\
& = & \D_{\Omega}  (\Psi \D_{\Omega} )^{\dagger} \Psi ( \D_{\Omega^c} \a_{\Omega^c}  + \epsilon ),
\end{eqnarray*}
where the first term is bias and the second term is the variance from noise. 

We aim to optimize the sampling strategy by minimizing the recovery error and there are six cases to be considered: 
\begin{enumerate}[(a)]
	\item minimizing the bias in the worst case; that is,	
	$$
	\min_{\Psi}  \left\| (\Psi \D_{\Omega} )^{\dagger} \Psi  \D_{\Omega^c} \right\|_2,
	$$
	where $ \sigma_{\max}$ is the largest singular value;
	\item minimizing the bias in expectation; that is,	
	$$
	\min_{S}   \left\| (\Psi \D_{\Omega} )^{\dagger} \Psi   \D_{\Omega^c} \right\|_{F}^2,
	$$
	where $\left\| \cdot \right\|_F$ is the Frobenius norm;
	\item minimizing the variance from noise in the worst case; that is,	
	$$
	\min_{\Psi}  \left\| (\Psi \D_{\Omega} )^{\dagger} \Psi  \right\|_2;
	$$
	\item minimizing the variance from noise in expectation; that is,	
	$$
	\min_{\Psi}  \left\| (\Psi \D_{\Omega} )^{\dagger} \Psi  \right\|_F;
	$$
	\item minimizing the recovery error in worst case, that is,
$$
	\min_{\Psi}  \left\| (\Psi \D_{\Omega} )^{\dagger} \Psi  \D_{\Omega^c} \right\|_2	+ c \left\| (\Psi \D_{\Omega} )^{\dagger} \Psi  \right\|_2,
$$
	where $c$ is related to the signal-to-noise ratio;
	\item minimizing the recovery error in expectation; that is,
$$
	\min_{\Psi}  \left\| (\Psi \D_{\Omega} )^{\dagger} \Psi  \D_{\Omega^c} \right\|_F + c \left\| (\Psi \D_{\Omega} )^{\dagger} \Psi  \right\|_F,
$$
	where $c$ is related to the signal-to-noise ratio.
\end{enumerate}

We call the resulting sampling operator $\Psi$ of each setting as an~\emph{optimal sampling operator with respect to setting ($\cdot$)}.  As shown in the previous work, (c) can be solved by a heuristic greedy search algorithm as shown in~\cite{ChenVSK:15} and a dual set algorithm as shown in~\cite{BoutsidisDM:14}; (d) can be solved by a convex relaxation as shown in~\cite{DavenportMNW:15}. When we deal with a huge graph, $\D_{\Omega^c}$ can be a huge matrix, solving (a), (b), (e), and (f) can lead to computational issues. When we only want to avoid the interference from a small subset of  $\D_{\Omega^c}$, we can use a similar heuristic greedy search algorithm to solve (a) and (e) and a similar convex relaxation to solve (b) and (f). Note that the sampling strategies are designed based the recovery strategy of partial least squares; it is not necessary to be optimal in general. Some other recovery strategies are variation minimization algorithms~\cite{ZhuGL:03, ZhouS:04, ChenSMK:14, AnisAO:15}. Comparing to the efficient sampling strategies~\cite{AnisAO:15}, the optimal sampling  operator is more general because it does not require $\D_{\Omega}$ to have any specific property. For example, $\D_{\Omega}$ can contain high-frequency atoms.

\mypar{Relations to Matrix Approximation} 
We show that the nodes sampled by the optimal sampling operator are actually the prototype nodes that maximally preserve the connectivity information in the graph structure matrix. In the task of matrix approximation, we aim to find a subset of rows to minimize the reconstruction error. When we specify the matrix to be a graph structure matrix, we solve
\begin{eqnarray*}
 && \min_{\Psi \in \R^{M \times N} } \left\| \RR - \RR (\Psi \RR)^{\dagger} (\Psi \RR) \right\|_{\xi},
\end{eqnarray*}
where $\Psi$ is the subsampling operator in~\eqref{eq:Psi}
 and $\xi = 2, F$.~\cite{BoutsidisDM:14} shows that 
\begin{eqnarray*}
\left\| \RR - \RR (\Psi \RR)^{\dagger} (\Psi \RR)  \right\|_{\xi}^2  \leq   
\left\| \RR  - \RR_K \right\|_{\xi}^2  \left\| \left( \Psi \Vm_{(K)} \right)^{\dagger} \Psi \right\|_2^2,
\end{eqnarray*}
where $\RR_K$ is the best rank $K$ approximation of $\RR$ and $\Vm_{(K)}$ is the first $K$ columns of the graph Fourier transform matrix of $\RR$. We see that when $\D_{\Omega} = \Vm_{(K)}$, the rows sampled by the optimal sampling operator in (c) minimize the reconstruction error of the graph structure matrix. When sampling nodes, we always lose information, but the optimal sampling operator selects most representative nodes that maximally preserves the connectivity information in the graph structure matrix.

\mypar{Experiments}
The main goal is to compare sampling strategies based on various recovery strategies and datasets.

We compare $6$ sampling strategies in total. We consider the sampling strategies implementing the optimal sampling operator in three ways: solving the setting (c) by the greedy search algorithm (Opt(G)), solving the setting (c) by the dual set algorithm (Opt(D)) and solving the setting (d) by solving the convex relaxation algorithm (Opt(C)). We also consider the efficient sampling strategies with three parameter settings (Eff($k$), where $k$  is the connection order, varying as $1, 2, 3$)~\cite{AnisAO:15}, where the efficient sampling strategies provide fast implementations by taking the advantages of the properties of the graph Laplacian.

We test on $4$ recovery strategies in total, including the partial least squares (PLS)~\eqref{eq:PLS}, harmonic functions (HF)~\cite{ZhuGL:03}, total variation minimization (TVM)~\cite{ChenSMK:14} and graph Laplacian based variation minimization with connection order 1 (LVM($1$))~\cite{AnisAO:15}. Note that the proposed optimal sampling operator is based on PLS and  the compared efficient sampling strategy is based on LVM.

We test the sampling strategies on two real-world datasets, including the wind speeds on the Minnesota road graph and the temperature measurements on the U.S city graph. As shown in Figures~\ref{fig:smooth_app_wind} and~\ref{fig:smooth_app_temp}, these signals are smooth on the corresponding graphs, but are not bandlimited.

In Section~\ref{sec:smooth_app}, we conclude that the graph Fourier basis based on graph Laplacian  $\Vm_{\LL}$ models the graph signals in these two datasets well. We thus specify the first $K$ columns of $\Vm_{\LL}$ as $\D_{\Omega}$ in~\eqref{eq:PLS}, where $K = 0.65 M$ and $M$ is the sample size. The physical meaning of $K$ is the bandwidth of a bandlimited space, which means that we design the samples based on a small bandlimited space. We use the same bandwidths in the recovery strategies of the partial least squares and the iterative projection between two convex sets.

\begin{figure}[htb]
  \begin{center}
    \begin{tabular}{cc}
 \includegraphics[width=0.4\columnwidth]{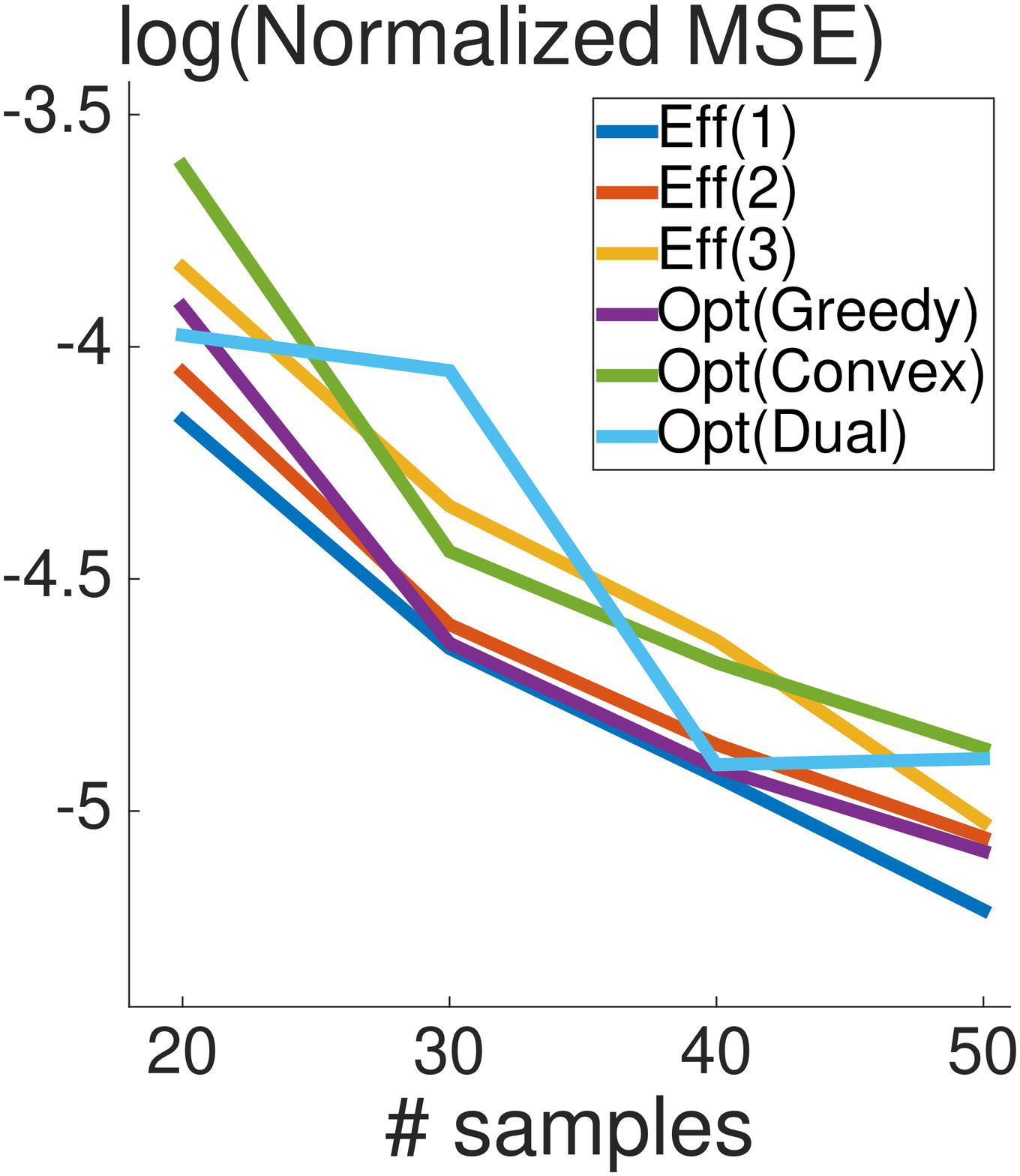}  & \includegraphics[width=0.4\columnwidth]{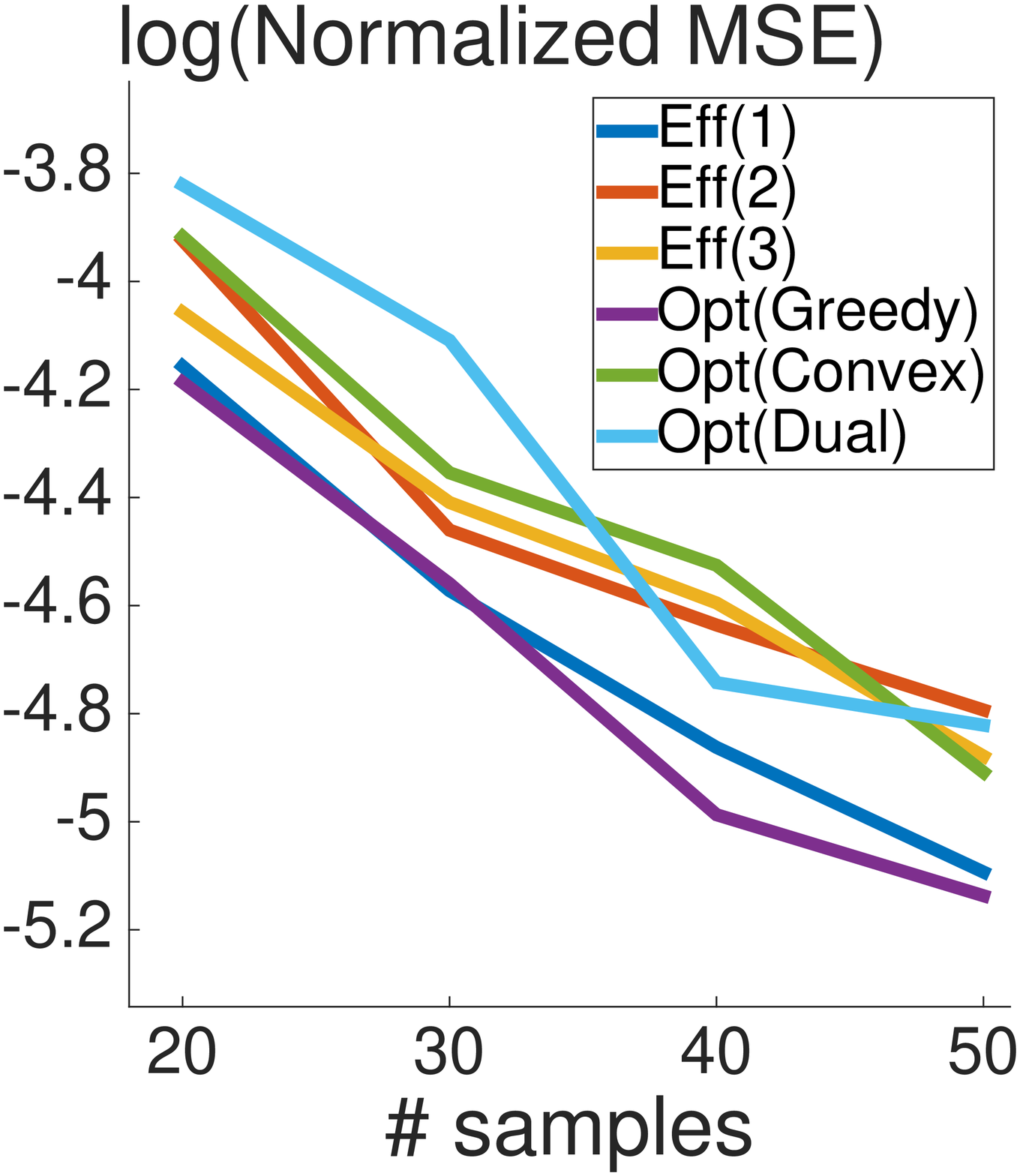}
\\
 {\small (a) PLS. } & {\small (b) HF. } 
 \\
 \includegraphics[width=0.4\columnwidth]{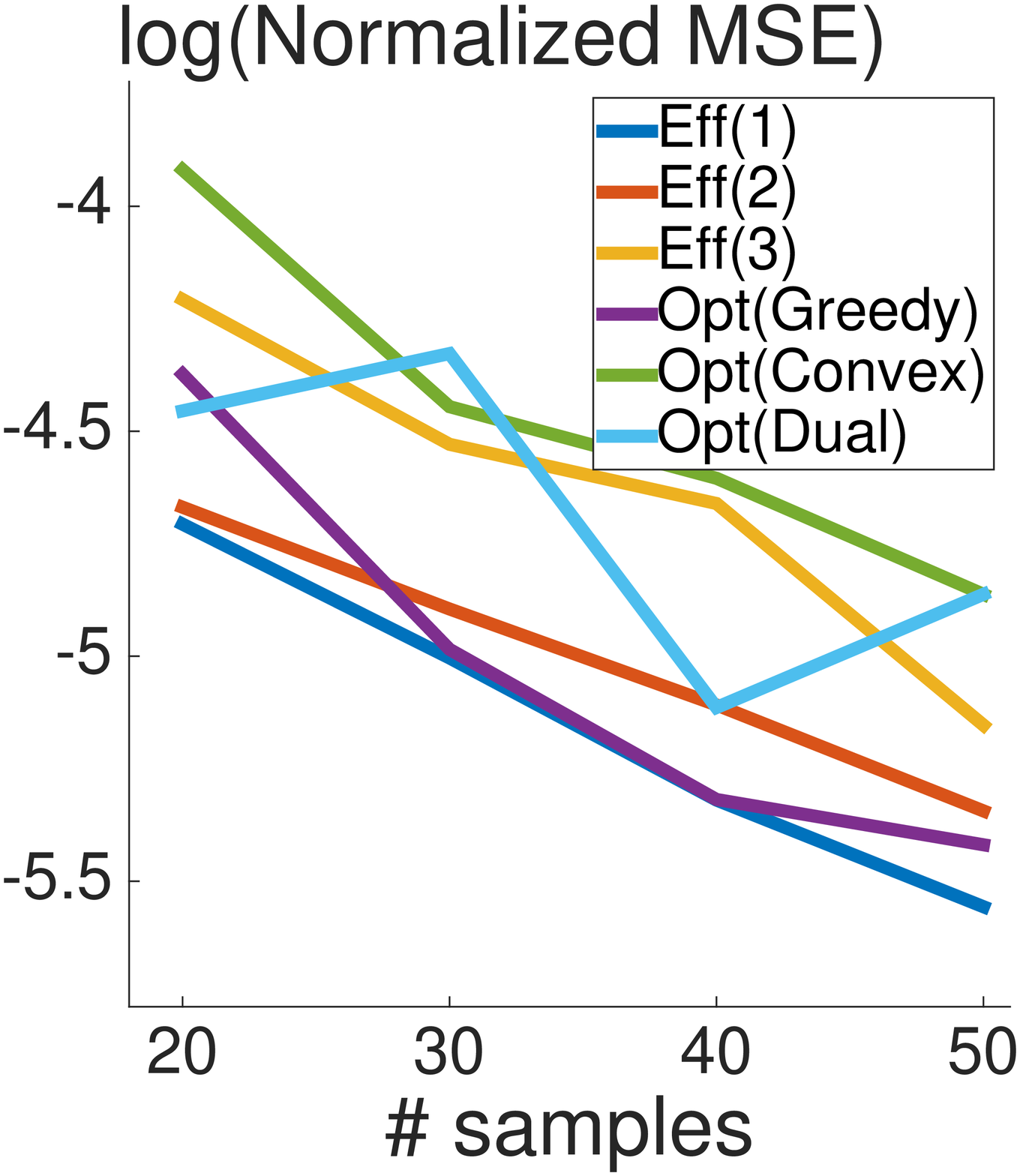}  & \includegraphics[width=0.4\columnwidth]{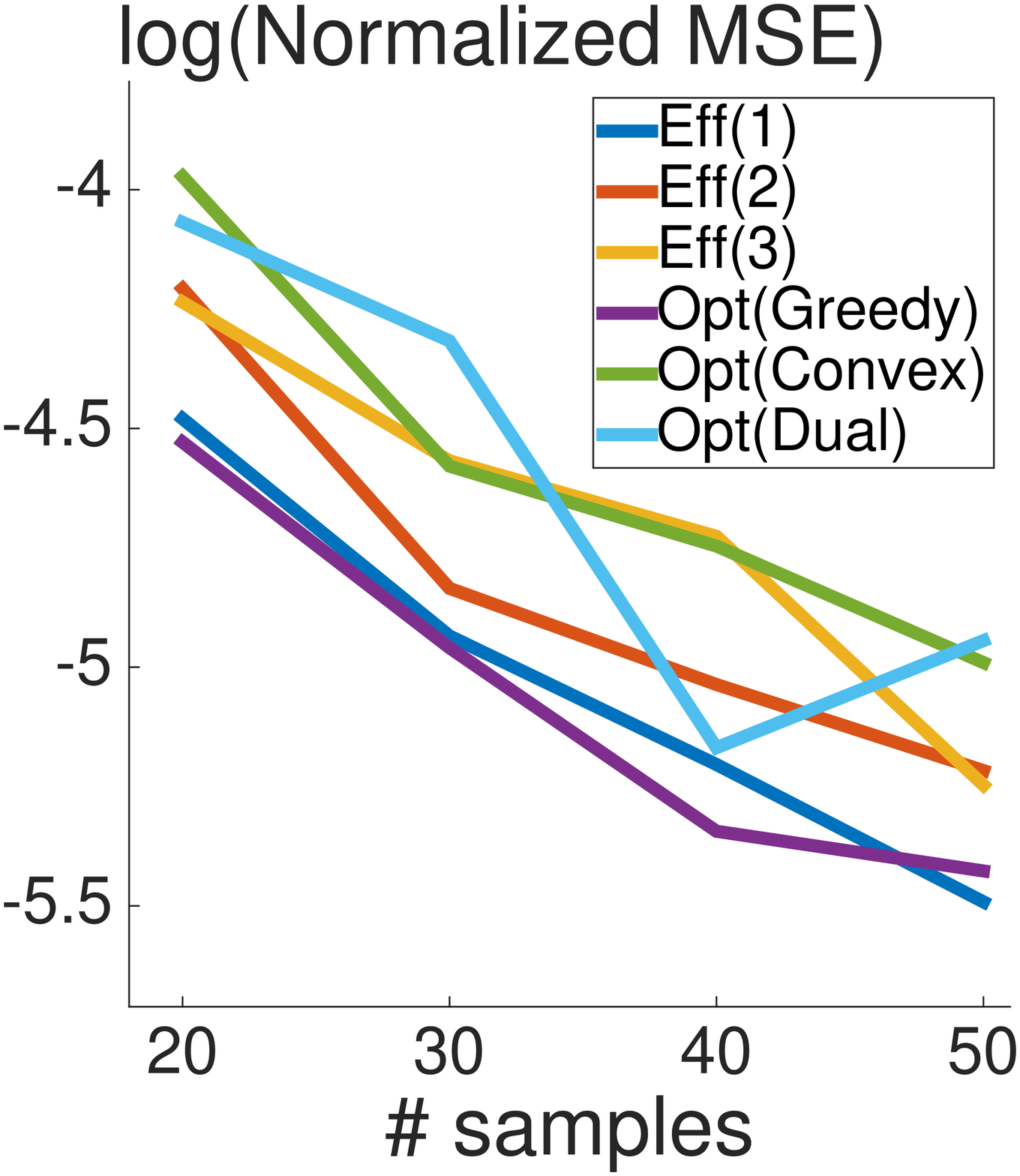}
\\
 {\small (c) TVM. } & {\small (d) LVM(1).} 
  \\
\end{tabular}
  \end{center}
   \caption{\label{fig:recovery_smooth_temp} Sampling and recovery of temperature measurements.  Lower means better. }
\end{figure}

Figures~\ref{fig:recovery_smooth_temp} and~\ref{fig:recovery_smooth_wind} shows the comparison of $6$ sampling strategies based on $4$ recovery strategies on the temperature dataset and the wind dataset, respectively. For each figure, the $x$-axis is the number of samples and the $y$-axis is the recovery error, which is evaluated by Normalized MSE~\eqref{eq:NMSE} in a logarithm scale. For both datasets, three optimal sampling operators are competitive with three efficient sampling strategies under each of $6$ recovery strategies. Since efficient sampling strategies need to compute the eigenvectors corresponds to the small eigenvalues, the computation is sometimes unstable when the graph is not well connnected. Within three optimal sampling operators, the greedy search algorithm provides the best performance. As shown in~\cite{AnisAO:15}, the computational complexity of greedy search algorithm is $O(N M^4)$, where $M$ is the number of samples, which is computationally inefficient; however, since we are under the experimentally designed sampling, all the algorithms are designed offline. When the sample size is not too huge, we prefer using a slower, but more accurate sampling strategy. The dual set algorithm also provides competitive performance and the computational complexity of dual set algorithm is $O(N M^3)$, which is more efficient than the greedy search algorithm. Thus, when one needs a small number of samples, we recommend the optimal sampling operator implemented by the greedy search algorithm; when one needs a large number of samples, we recommend the optimal sampling operator implemented by the dual set algorithm.

In~\cite{JiCVK:15}, the sampling followed with recovery of wind speeds is used to plan routes for autonomous aerial vehicles.

\begin{figure}[htb]
  \begin{center}
    \begin{tabular}{cc}
 \includegraphics[width=0.4\columnwidth]{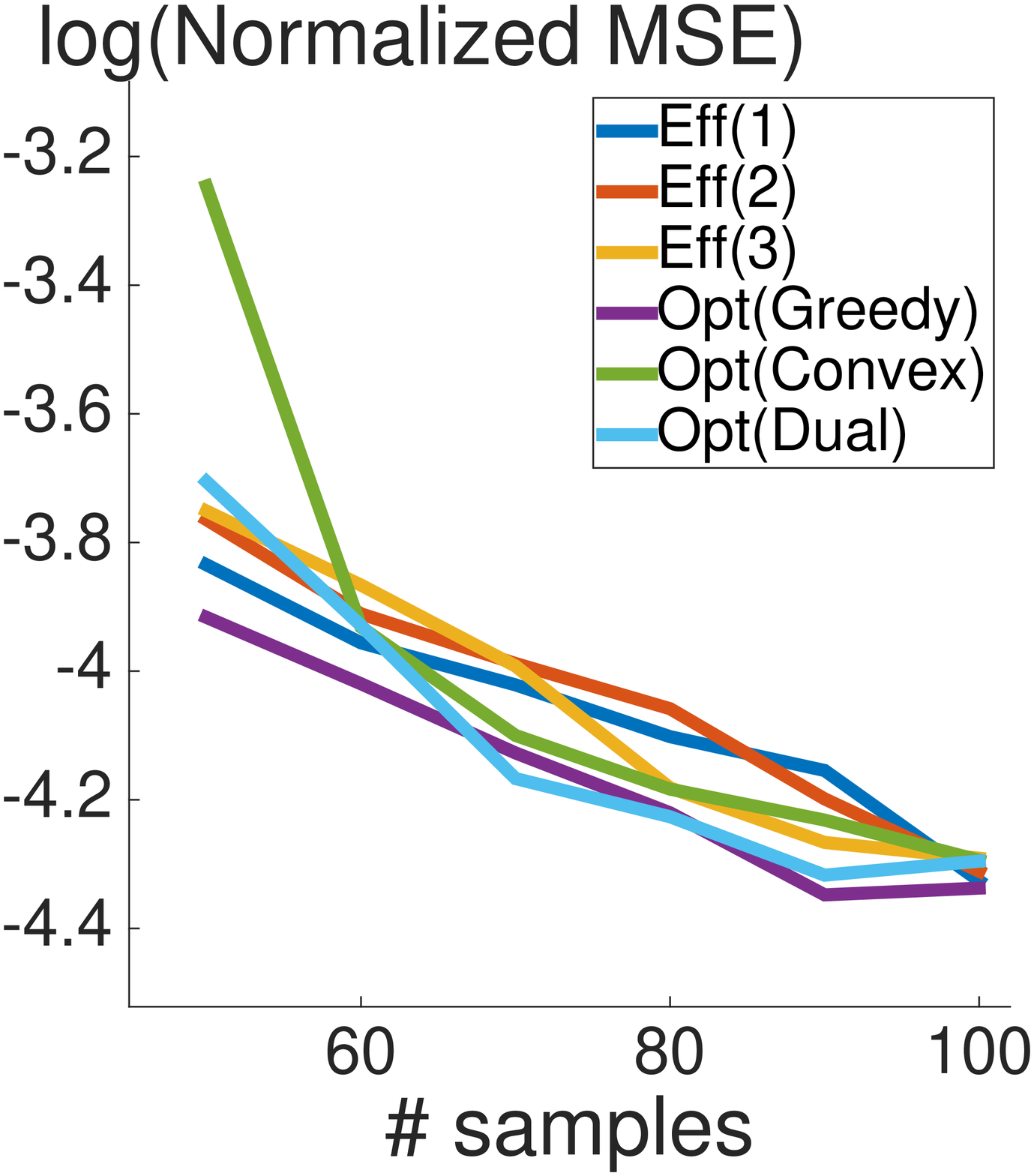}  & \includegraphics[width=0.4\columnwidth]{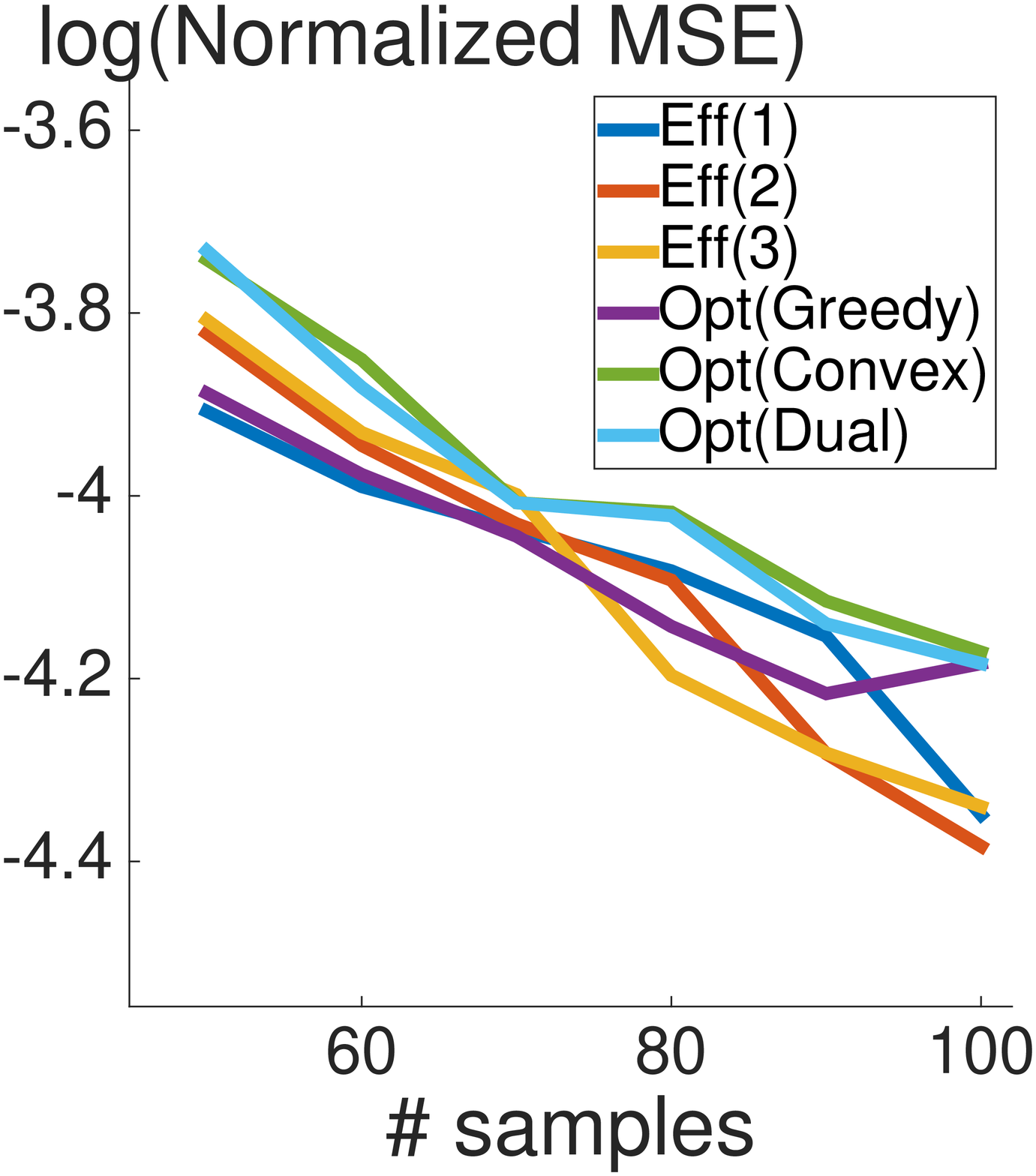}
\\
 {\small (a) PLS. } & {\small (b) HF. } 
 \\
 \includegraphics[width=0.4\columnwidth]{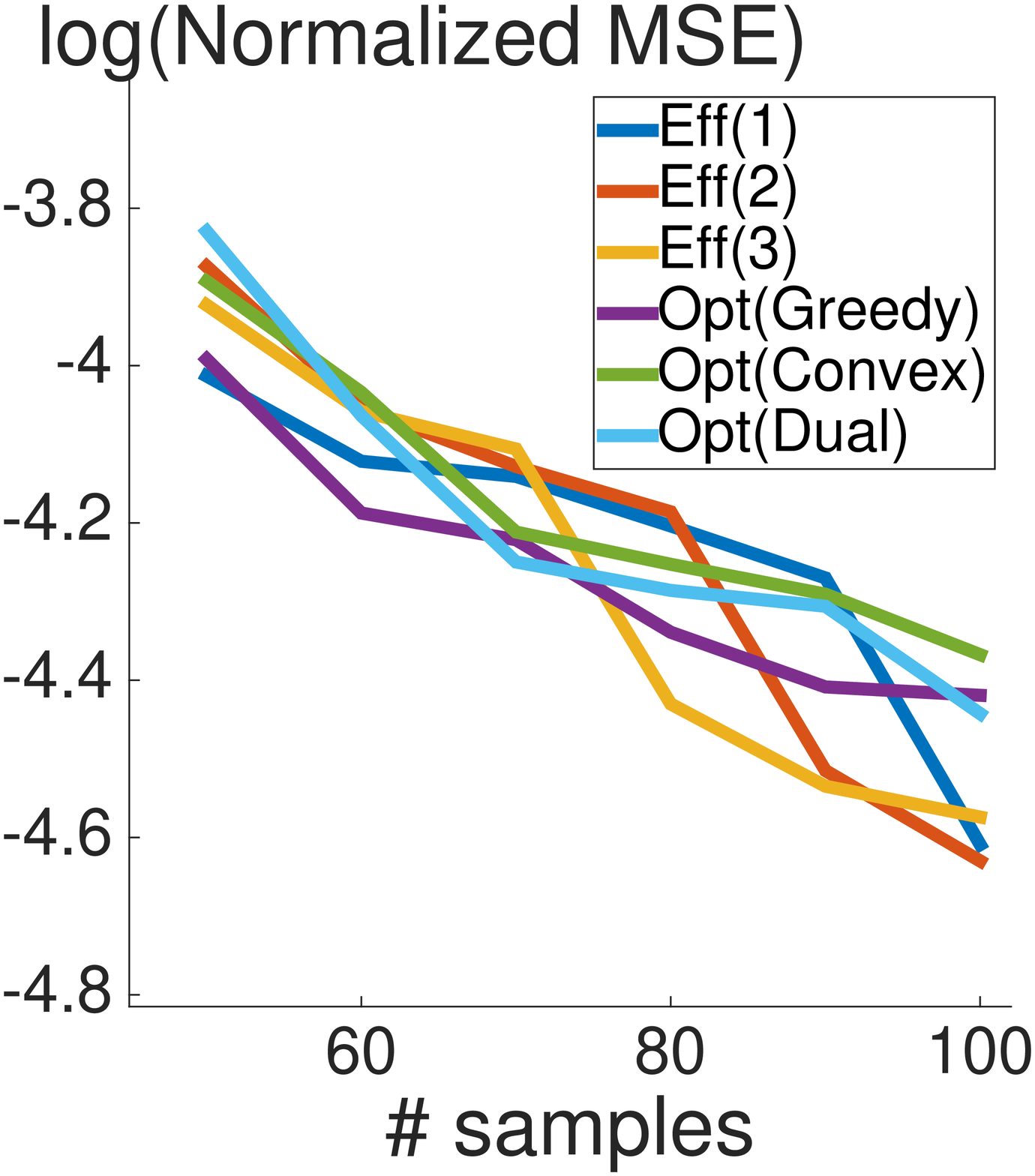}  & \includegraphics[width=0.4\columnwidth]{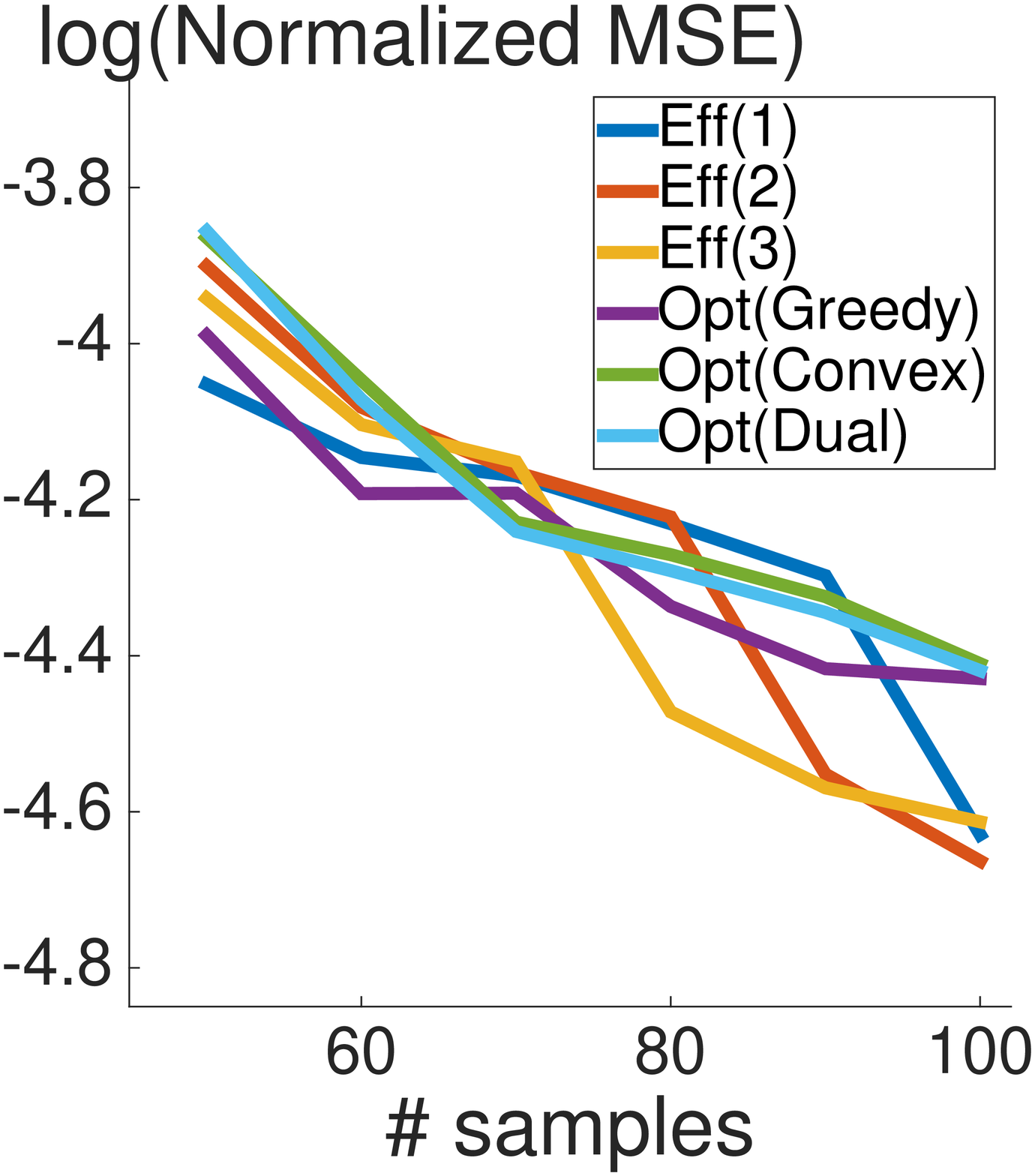}
\\
 {\small (c) TVM. } & {\small (d) LVM(1).} 
  \\
\end{tabular}
  \end{center}
   \caption{\label{fig:recovery_smooth_wind} Sampling and recovery of wind speeds. Lower mean better.  }
\end{figure}

\subsubsection{Case Study: Coauthor Network}
We aim to use the proposed approximation and sampling with recovery techniques to large-scale graphs. Here we present some preliminary results.

We collect the publications on three journals, including IEEE Transactions on Signal Processing (TSP), Image Processing (TIP) and Information Theory (TIT). The dataset includes $10,011$ papers on TSP contributed by $10,569$ authors, $5,304$ papers on TIP contributed by $8,388$ authors,  $13,303$ papers on TIT contributed by $9,533$ authors. We construct a coauthor network where nodes are unique authors and edges indicate the coauthorship. The edge weight $\W_{i,j}= 1$ when author $i$ and $j$ wrote at least one journal together, and $0$, otherwise.  The graph includes $25,282$ unique authors and $46,540$ coauthorships. We emphasize unique authors because some of them may publish papers on more than one journals. 

We form a graph signal by count the total number of papers on TSP for each unique author, which describes a distribution of the contribution to the signal processing community. Intuitively, the graph signal is smooth because people in the same community often write papers together. We then check which graph Fourier basis can well represent this smooth graph signal. Since the graph is huge, the full eigendecomposition is computational inefficient and we only compute $100$ eigenvectors. For the adjacency matrix, we compute the eigenvectors corresponding the largest $100$ eigenvalues; for the graph Laplacian matrix, we compute the eigenvectors corresponding the smallest $100$ eigenvalues; for the transition matrix, we compute the eigenvectors corresponding the first $100$ eigenvalues with largest magnitudes. Since the graph contains hundreds of disconnected components, the eigenvectors of graph Laplacian matrix do not converge all the time. We use the nonlinear approximation~\eqref{eq:nonlinear_approx} to approximate the graph signal. Figure~\ref{fig:tsp_err}(a) shows the approximation errors based on three graph Fourier bases, $\Vm_{\W}$, $\Vm_{\LL}$, $\Vm_{\Pj}$. We see that the graph Fourier basis based on the adjacency matrix provides much better representation for the graph signal of counting the number of papers in TSP; the first 100 eigenvectors of $\Vm_{\LL}$ and $\Vm_{\Pj}$ capture little information in the graph signal. 

When we only have this coauthor network and we aim to have a rough idea of contribution to the signal processing community from each author, it is clear that we do not want to check the publication lists of all the $25,282$ authors. Instead, we can use the sampling and recovery techniques in Section~\ref{sec:smooth_samplingandrecovery} to recover the contribution distribution from a few designed samples. From the aspects of matrix approximation, we aim to select most representative authors and minimize the lost information. Table~\ref{tab:tsp_sample} list the first 10 authors that we want to query. The first column is sorted based on the total number of papers published in all three transactions in a descending order (degree); the rest of the columns are the designed samples provided by the optimal sampling operator implemented by the greedy search algorithm. We try each of the three graph Fourier bases $\Vm_{\W}$, $\Vm_{\LL}$ and $\Vm_{\Pj}$ and the samples based on $\Vm_{\W}$ makes more sense. The intuition behind the designed samples is that we want to sample the hub of the large communities in the graph. For example, when two authors wrote many papers together, even both of them have a large number of papers, we only want to sample one from these two. Note that the optimal sampling operator selects most representative authors and it does not rank the importance of each authors.

\begin{figure}[htb]
  \begin{center}
 \includegraphics[width=0.5\columnwidth]{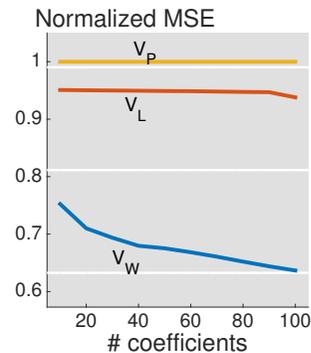} 
  \end{center}
   \caption{\label{fig:tsp_err} Approximation and recovery of contribution to TSP.  }
\end{figure}

\begin{table*}[htbp]
  \footnotesize
  \begin{center}
    \begin{tabular}{@{}llll@{}}
      \toprule
      Number of publications &  $\Psi^*{(\Vm_{\W})}$  &  $\Psi^*{(\Vm_{\LL})}$  &  $\Psi^*{(\Vm_{\Pj})}$  \\
      \midrule \addlinespace[1mm]
H. Vincent Poor  & H. Vincent Poor & Vishnu Naresh Boddeti  & Jang Yi \\

Shlomo Shamai & Aggelos K. Katsaggelos & Kourosh Jafari-Khouzani & Richard B. Wells \\

Dacheng Tao  & Truong Q. Nguyen  & Y.-S. Park  & Akaraphunt Vongkunghae \\

Truong Q. Nguyen & Shuicheng Yan  & Gabriel Tsechpenakis &  Tsung-Hung Tsai \\

A. Robert Calderbank & Dacheng Tao  & Tien C. Hsia &   Chuang-Wen You  \\

Yonina C. Eldar  & Tor Helleseth  & James Gonnella &   Arvin Wen Tsui  \\

Georgios B. Giannakis & Shlomo Shamai  & Clifford J. Nolan &  Min-Chun Hu \\

Shuicheng Yan & Michael Unser  & Ming-Jun Lai &  Wen-Huang Cheng  \\

Aggelos K. Katsaggelos & Petre Stoica & J. Basak &  Heng-Yu Chi \\

Tor Helleseth & Zhi-Quan Luo  & Shengyao Chen &  Chuan Qin \\










\bottomrule
\end{tabular} 
\caption{\label{tab:tsp_sample}. 10 most representative authors in TSP, TIP and TIT suggested by the optimal sampling operators. Suppose a new student wants to study this field by reading papers in TSP, TIP and TIT. Instead of reading papers from related authors, the student should query most representative authors to enrich his or her knowledge.}
\end{center}
\end{table*}

\section{Representations of Piecewise-constant Graph Signals}
\label{sec:R_PC}
In classical signal processing, a piecewise-constant signal means a signal that is locally constant in connected regions separated by lower-dimensional boundaries. It is often related to step functions, square waves and Haar wavelets. It is widely used in image processing~\cite{VetterliKG:12}. Piecewise-constant graph signals have been used in many applications related to graphs without having been explicitly defined; for example, in community detection, community
labels form a piecewise-constant graph signal for a social network; in
semi-supervised learning, classification labels form
a piecewise-constant graph signal for a graph constructed from
the dataset. While smooth graph signals emphasize slow transitions,
piecewise-constant graph signals emphasize fast transitions
(corresponding to boundaries) and localization on the vertex domain (corresponding to
signals being nonzeros in a local neighborhood).

\subsection{Graph Signal Models}
We introduce two definitions for piecewise-constant  graph signals: one comes from the descriptive approach and the other one comes from the generative approach. We also show their connections.

Let $\Delta$ be the~\emph{graph difference
  operator} (the oriented incidence matrix of $G$), whose rows
correspond to edges~\cite{SharpnackRS:12, WangSST:15}. For example, if
$e_i$ is a directed edge that connects the $j$th node to the $k$th
node ($j < k$), the elements of the $i$th row of $\Delta$ are
\begin{equation*}
\label{eq:Delta}
 \Delta_{i, \ell} = 
  \left\{ 
    \begin{array}{rl}
      - {\rm sgn}(\W_{j,k}) \sqrt{ | \W_{j,k} | }, & \ell = j;\\
       {\rm sgn}(\W_{j,k}) \sqrt{ | \W_{j,k} | }, & \ell = k;\\
      0, & \mbox{otherwise},
  \end{array} \right.
\end{equation*}
where $\W$ is the weighted adjacency matrix, sgn$(\cdot)$ denotes the sign, which is used to handle the negative edge weights. For a directed graph, $\W_{j,k}$ and $\W_{k,j}$ corresponds to two directed edges, representing in two rows in $\Delta$; for an undirected graph, since $\W_{j,k} = \W_{k,j}$, they correspond to a same undirected edge, representing in one row in $\Delta$.

To have more insight on the graph difference operator, we compare it with the graph shift operator. The graph shift operator diffuses each signal coefficient to its neighbors based on the edge weights and $\W \x$ is a graph signal representing the shifted version of $\x$. The $i$th element of $\W \x$,
$$
\left(  \W \x  \right)_i = \sum_{j \in {\rm Neighbor}(i)} \W_{i,j} x_j,
$$
assigns the weighted average of its neighbors' signal coefficients to the $i$th node. The difference between the original graph signal and the shift version, $\left( \Id - 1/\lambda_{\max} \W \right) \x$ where $\lambda_{\max}$ is the largest eigenvalue of $\W$, is a graph signal measuring the difference of $\x$. The term $\left\| \left( \Id - 1/\lambda_{\max} \W \right) \x \right\|_p^p$ measure the smoothness of a graph signal as shown in Definitions~\ref{df:local_neighboring_smooth} and~\ref{df:global_neighboring_smooth}.

The graph difference operator compares the signal coefficients of two nodes connected by each edge and $\Delta \x$ is an~\emph{edge signal} representing the difference of $\x$. The $i$th element of $\Delta \x$,
$$
\left(  \Delta \x  \right)_i = {\rm sgn}(\W_{j,k}) \sqrt{ | \W_{j,k} | } \left(  \x_k - \x_j \right),
$$
assigns the difference between two adjacent signal coefficients to the $i$th edge,  where the $i$th edge connects the $j$th node to the $k$th node ($j < k$). The term $\left\| \Delta \x \right\|_p^p$ also measures the smoothness of $\x$. Comparing two measures, $\left\| \left( \Id - 1/\lambda_{\max} \W \right) \x \right\|_p^p$ emphasizes the neighboring difference, which compares each signal coefficient with the weighted average of its neighbors' signal coefficients; and $\left\| \Delta \x \right\|_p^p$  emphasizes the pairwise difference, which compares each pair of adjacent signal coefficients. When $G$ is an undirected graph, $\left\| \Delta \x \right\|_2^2 = \x^T \LL \x$, where $\LL$ is the graph Laplacian matrix, which measures the total Lipschiz smoothness as shown in Definition~\ref{df:global_Lip}. When the graph is unweighted and all the edge weights are nonnegative, the elements of the $i$th row of $\Delta$ are simply
\begin{equation*}
\label{eq:C}
 \Delta_{i, \ell} = 
  \left\{ 
    \begin{array}{rl}
      1, & \ell = k;  \\
      -1, & \ell = j; \\
      0, & \mbox{otherwise}.
  \end{array} \right.
\end{equation*}

The class of piecewise-constant graph signals is a complement of the class of smooth graph signals, because many real-world graph signals contain outliers, which are hardly captured by smooth graph signals. Smooth graph signals emphasize the slow transition over nodes; and piecewise-constant graph signals emphasize fast transition, localization on the vertex domain: fast transition corresponds to the boundary and localization on the vertex domain corresponds to signals being nonzeros in a local neighbors.

We now define the class of piecewise-constant graph~signals.
\begin{defn}
  \label{df:pc_des}
  A graph signal $\x \in \R^N$ is~\emph{piecewise-constant} with $K \in
  \{0, 1, \cdots, N-1\}$ cuts, when it satisfies
  \begin{displaymath}
  \left\| \Delta \x  \right\|_0 \leq K.
\end{displaymath}
\end{defn}

Each element of $\Delta \x$ is the difference between two adjacenct signal coefficient; when $\left( \Delta \x \right)_i \neq 0$, we call the $i$th edge is inconsistent. The class $\PC_{G}(K)$ represents signals that contain at most $K$ inconsistent edges. For example, when $\x = \e_i$, $\left\| \Delta \x  \right\|_0$ is the out degree of the $i$th node; when $\x = {\bf 1}$, $\left\| \Delta \x  \right\|_0 = 0$. From the perspective of graph cuts, we cut inconsistent edges such that a graph is separated into several communities where signal coefficients are the same within each community. Note that Definition~\ref{df:pc_des} includes sparse graph signals. To eliminate sparse graph signals, we can add $\left\| \x \right\|_0 \geq S$ into the definition.

The piecewise-constant graph signal models in Definition~\ref{df:pc_des}  are introduce in a descriptive approach. Following these, we are going to translate the descriptive approach to the generative approach. We construct piecewise-constant graph signals by using local sets, which have
been used previously in graph cuts and graph signal
reconstruction~\cite{Luxburg:07, WangLG:14}.
\begin{defn}
  \label{df:localset}
  Let $\{S_c\}_{c=1}^C$ be the partition of the node set $\V$. We call $\{S_c\}_{c=1}^C$~\emph{local sets} when they satisfy that the subgraph corresponding to each local set is connected, that is, when $G_{S_c}$ is connected for all $c$.
\end{defn}

We can represent a local set $S$ by using a local-set-based graph signal, ${\bf 1}_{S} \in \R^N$, where
\begin{equation*}
\left(  {\bf 1}_{S}  \right)_i = 
  \left\{ 
    \begin{array}{rl}
      1, & v_i \in S;\\
      0, & \mbox{otherwise}.
  \end{array} \right.
\end{equation*}

For a local-set-based graph signal, we measure its smoothness by the normalized variation
$$
\STV_{\Delta, p} (S) = \frac{1}{\left\|  {\bf 1}_S  \right\|_p^p}\left\| \Delta {\bf 1}_S \right\|_p^p.
$$
For unweighted graphs, $\STV_{\Delta, 0} (S) = \STV_{\Delta, 1}(S) = \STV_{\Delta, 2} (S)$. It measures how hard it is to cut the boundary edges and make $G_S$ an isolated subgraph. We normalize the variation by the size of the local set, which implies that given the same cut cost, a larger local set is more smooth than a smaller local set.

\begin{defn}
  \label{df:pc_gen}
A graph signal $\x$ is piecewise-constant based on local sets $\{S_c\}_{c=1}^C$ when
\begin{equation*}
	\x \ = \  \sum_{c=1}^C a_c {\bf 1}_{S_c},
\end{equation*}
where $\{S_c\}_{c=1}^C$ forms a valid series of local sets. Denote this class by $\PC(C)$. 
\end{defn}

When the value of the graph signal on each local set is different, $\left\| \Delta \x  \right\|_0$ counts the total number of edges connecting nodes between local sets. Definition~\ref{df:pc_gen} defines the piecewise-constant graph signal in a generative approach; however, the corresponding representation dictionary has a huge number of atoms, which is impractical to use.

\subsection{Graph Dictionary}
We now discuss representations for piecewise-constant graph signals based on a designed multiresolution local sets. The corresponding representation dictionary has a reasonable size and provides sparse representations for arbitrary piecewise-constant graph signals.

\subsubsection{Design}
\label{sec:Dict_LSPC}
We aim to construct a series of local sets in a multiresolution fashion. We first define the multiresolution analysis on graphs.
\begin{defn}
\label{df:MRA}
A general multiresolution analysis on graphs consists of a sequence of embedded closed subspaces
\begin{equation*}
V_0 \subset  V_{1}  \subset V_{2}  \cdots \subset V_{K}
\end{equation*}
such that
\begin{itemize}
\item upward completeness
  \begin{equation*}
  		\bigcup_{i=0}^K  V_{i} = \R^N;
  \end{equation*}
\item downward completeness
  \begin{equation*}
  		\bigcap_{i=0}^K  V_{i} = \{ c {\bf 1_{\V}}, c \in \R \};
  \end{equation*}
\item existence of basis \  There exists an orthonormal basis $\{\Phi\}_{i}$ for $V_K$.
\end{itemize}
\end{defn}
Compared with the original multiresolution analysis, the complete space here is $\R^N$ instead of $\mathcal{L}_2(\R)$
because of the discrete nature of a graph; we remove scale invariance and translation invariance because the rigorous definitions of scaling and translation for graphs are still unclear. This is the reason we call it~\emph{general multiresolution analysis on graphs}.

\mypar{General Construction}
The intuition behind the proposed construction is to build the connection between the subspaces and local sets: a bigger subspace corresponds to a finer resolution on the graph vertex domain, or more localized local sets. We initialize $S_{0,1} = \V$
to correspond to the $0$th level subspace $V_0$, that is, $V_0 = \{
c_0 {\bf 1}_{S_{0,1}}, c_0 \in \R \}$. We then partition $S_{0,1}$
into two disjoint local sets $S_{1,1}$ and $S_{1,2}$, corresponding to the first level subspace $V_1$, where $V_1 = \{ c_1 {\bf 1}_{S_{1,1}}
+ c_2 {\bf 1}_{S_{1,2}}, c_1, c_2 \in \R \}$. We then
recursively partition each larger local set into two smaller local
sets. For the $i$th level subspace, we have $V_i = \sum_{j=1}^{2^i}
c_j {\bf 1}_{S_{i,j}}$ and then, we partition $S_{i,j}$ into
$S_{i+1,2j-1}, S_{i+1,2j}$ for all $j = 1, 2, \ldots, 2^i$.  We call
$S_{i,j}$ the parent set of $S_{i+1,2j-1}, S_{i+1,2j}$ and
$S_{i+1,2j-1}, S_{i+1,2j}$ are the children sets of $S_{i,j}$. When
$|S_{i,j}| \leq 1$, $S_{i+1,2j-1} = S_{i,j}$ and $S_{i+1,2j} =
\emptyset$. At the finest resolution, each local set corresponds to an
individual node or an empty set. In other words, we build a binary
decomposition tree that partitions a graph structure into multiple
local sets. The $i$th level of the decomposition tree corresponds to
the $i$th level subspace. The depth of the decomposition $T$
depends on how local sets are partitioned; $T$ ranges from
$N$ to $\ceil*{\log N}$, where $N$ corresponds to partitioning one
node at a time and $\ceil*{\log N}$ corresponds to an even partition
at each level.

It is clear that the proposed construction of local sets satisfies three requirements in Definition~\ref{df:MRA}. The initial subspace $V_0$ has the coast resolution. Through partitioning, local sets zoom into
increasingly finer resolutions in the graph vertex domain. The subspace $V_T$ with finest resolution zoom into each individual node and covers the entire $\R^N$.  Classical scale invariance requires that when $f(t) \in V_0$, then $f(2^m t) \in V_m$, which is ill-posed
in the graph domain because graphs are finite and discrete;
the classical translation invariance requires that when $f(t) \in
V_0$, then $f(t-n) \in V_0$, which is again ill-posed, this time
because graphs are irregular. The essence of scaling and translation
invariance, however, is to use the same function and its scales and
translates to span different subspaces, which is what the proposed
construction promotes. The scaling function is ${\bf 1}_{S}$; the hierarchy of partition is similar to the scaling and translation, that is, when ${\bf 1}_{S_{i,j}} \in \V_{i}$, then ${\bf 1}_{S_{i+1,2j-1}}, {\bf 1}_{S_{i+1,2j}} \in \V_{i+1}$, and when ${\bf 1}_{S_{i+1,2j-1}} \in \V_{i+1}$ then ${\bf 1}_{S_{i+1,2j}} \in \V_{i+1}$.

 To summarize the construction, we build a local set decomposition tree by recursively partitioning a local set into two disjoint local sets until that all the local sets are individual nodes. We now show a toy example in Figure~\ref{fig:decomposition}. In Partition 1, we partition the entire node set $S_{0,1} = \V = \{1, 2, 3, 4 \}$ into two disjoint local sets $S_{1,1} =  \{1, 2 \}, S_{1,2} =  \{3, 4 \}$. Thus, $V_1 = \{ c_1  {\bf 1_{S_{1,1}}} +  c_2  {\bf 1_{S_{1,2}}}  , c_1, c_2 \in \R \}$. Similarly, in Partition 2, we partition $S_{1,1}$ into two disjoint connected sets $S_{2,1} =  \{1 \}, S_{2,2}  =  \{ 2 \}$; in Partition 3, we partition $S_{1,2}$ into $S_{2,3}  =  \{3 \}, S_{2,4}  =  \{ 4 \}$. Thus, $V_2 = \{ c_1  {\bf 1_{S_{2,1}}} +  c_2  {\bf 1_{S_{2,2}}} + c_3 {\bf 1_{S_{2,3}}} +  c_4  {\bf 1_{S_{2,4}}} , c_1, c_2, c_3, c_4 \in \R \} = \R^4$.
 
\begin{figure}[t]
  \begin{center}
     \includegraphics[width= 1\columnwidth]{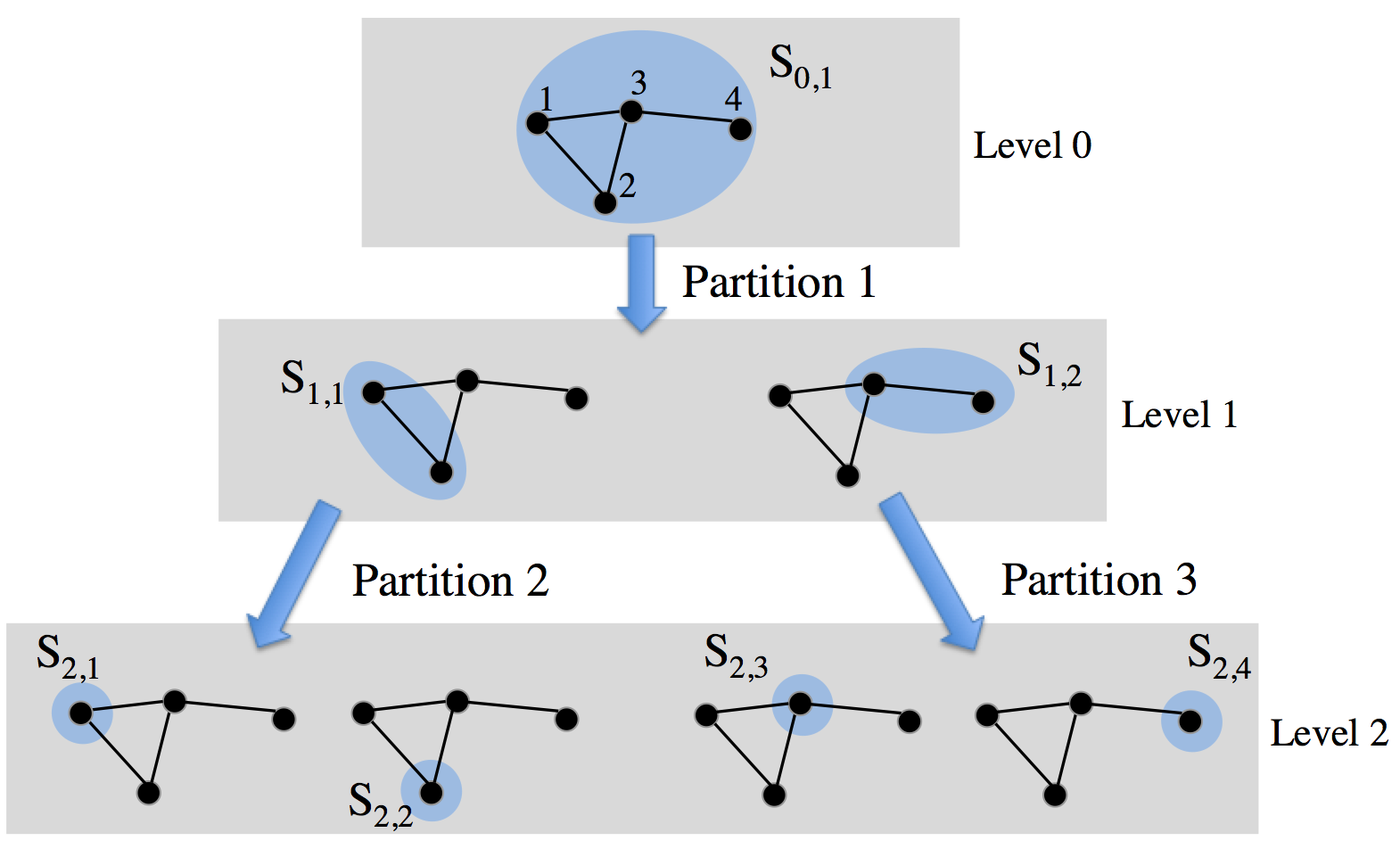}
  \end{center}
  \caption{\label{fig:decomposition} Local set decomposition tree. In each partition, we decompose a node set into two disjoint connected set and generate a basis vector to the wavelet basis. $S_{0,1}$ is in Level 0, $S_{1,1}, S_{1,2}$ are in Level 1, and $S_{2,1}, S_{2,2}, S_{2,3}, S_{2,4}$ are in Level~2.}
\end{figure}
 
\mypar{Graph Partition Algorithm}
 The graph partition is the key step to construct the local sets.  From the perspective of promoting smoothness of graph signals, we partition a local set $S$ into two disjoint local set $S_1, S_2$ by solving the following optimization problem
\begin{eqnarray}
\label{algo:partition}
& \min_{S_1, S_2} & \STV_{\Delta, 0} (S_1) + \STV_{\Delta, 0} (S_2)
\\
\nonumber
& {\rm subject~to:} & S_1 \cap S_2 = \emptyset, S_1 \cup S_2 = S, 
\\
\nonumber
&& G_{S_1}~{\rm and}~G_{S_2}  {\rm~are~connected}.
\end{eqnarray}  
Ideally, we aim to solve~\ref{algo:partition} to obtain two children local sets, however, it is nonconvex and hard to solve. Instead, we consider three relaxed methods to partition a graph.
 
The first one is based on spectral clustering~\cite{Luxburg:07}. We first obtain the graph Laplacian matrix of a local set and compute the eigenvector corresponding to the second smallest eigenvalue of the graph Laplacian matrix. We then set the median number of the eigenvector as the threshold; we put the nodes whose corresponding values in the eigenvector are no smaller than the threshold into a children local set and put the nodes whose corresponding values in the eigenvector are smaller than the threshold into the other children local set. This method approximately solves~\ref{algo:partition} by ignoring the second constraint; it guarantees that two children local sets have the same number of nodes, but does not guarantee that they have the same number of nodes are connected.

The second one is based on spanning tree. To partition a local set, we first obtain the maximum spanning tree of the subgraph and then find a balance node in the spanning tree. The balance node partition the spanning tree into two subtrees with the closet number of nodes~\cite{SharpnackKS:13}. We remove the balance node from the spanning tree, the resulting largest connected component form a children local set and the other nodes including the balance node forms the other children local set. This method approximately solves~\ref{algo:partition} by approximating a subgraph by the corresponding maximum spanning tree; it guarantee that two children local sets are connected, but does not guarantees that they have the same number of nodes. When the original sbugraph is highly connected, the spanning tree loses some connection information and the shape of the local set may not capture the community in the subgraph.

The third one is based on the 2-means clustering. We first randomly select 2 nodes as the community center and assign every other node to its nearest community center based on the geodesic distance. We then recompute the community center for each community by minimizing the summation of the geodesic distances to all the other nodes in the community and assign node to its nearest community center again. We keep doing this until the community centres converge after a few iterations. This method is inspired from the classical $k$-means clustering; it also guarantees that two children local sets are connected, but does not guarantees that they have the same number of nodes.

In general, the proposed construction of local sets does not restrict to any particular graph partition algorithm; depending on the applications, the partition step can also be implemented by many other existing graph partition algorithms.

\mypar{Dictionary Representations}
We collect local sets by level in
ascending order in a dictionary, with atoms corresponding to each
local set, that is, $ \D_{\rm LSPC} = \{ {\bf 1}_{S_{i, j}} \}_{i=0, j = 1}^{i=T, j = 2^i}$.
We call it the~\emph{local-set-based piecewise-constant dictionary}.
After removing empty sets, the dictionary has $2N-1$ atoms, that is,
$\D_{\rm LSPC} \in \R^{N \times (2N-1)}$; each atom is a
piecewise-constant graph signal with various sizes and localizing
various parts of a graph.

\mypar{Wavelet Basis}
We construct a wavelet basis based on the local-set-based piecewise-constant dictionary. We combine two local sets partitioned from the same parent local set to form a basis vector. Let the local sets $S_{i+1,2j-1}, S_{i+1,2j}$ have the same parent local set $S_{i,j}$, the basis vector combing these two local sets is
\begin{eqnarray*}
&& \sqrt{ \frac{|S_{i+1,2j-1}| |S_{i+1,2j}|}{|S_{i+1,2j-1}|+|S_{i+1,2j}|}} \bigg(  \frac{1}{|S_{i+1,2j-1}|} {\bf 1}_{S_{i+1,2j-1}} 
\\
&& - \frac{1}{|S_{i+1,2j}|} {\bf 1}_{S_{i+1,2j}} \bigg).
\end{eqnarray*}
To represent in a matrix form, the wavelet basis is
\begin{eqnarray*}
\W_{\rm LSPC} = \D_{\rm LSPC} \D_2,
\end{eqnarray*}
where the downsampling matrix
\begin{eqnarray*}
 \D_2 & = &
 \begin{bmatrix}
 \frac{1}{\left\| \d_1 \right\|_2} & \phantom{+}0  & \cdots & \phantom{+}0 \\
 0 & \phantom{+} g(\d_2, \d_3)  & \cdots & \phantom{+}0 \\
 0 & -g(\d_3, \d_2)  & \cdots & \phantom{+}0 \\
  0 & \phantom{+}0 & \cdots & \phantom{+}0 \\
 0 & \phantom{+}0   & \cdots & \phantom{+}0 \\
 \vdots & \phantom{+}\vdots & \ddots & \phantom{+}\vdots \\
  0 & \phantom{+}0  & \cdots & \phantom{+}g(\d_{2N-2}, \d_{2N-1} )  \\
 0 & \phantom{+}0  & \cdots & -g(\d_{2N-1}, \d_{2N-2}) \\
 \end{bmatrix}
 \\
 & \in  & \R^{(2N-1) \times N},
\end{eqnarray*}
with $\d_i$ is the $i$th column of $\D_{\rm LSPC}$, and 
 $$
 g(\d_i, \d_j) = \sqrt{ \frac{\left\| \d_j \right\|_0} { ( \left\| \d_i \right\|_0  + \left\| \d_j \right\|_0 ) \left\| \d_i \right\|_0  }}.
 $$
The downsampling matrix $\Um_2$ combines two consecutive column vectors in $\D_{\rm LSPC}$ to form one column vector in $\W_{\rm LSPC}$ and the function $g(\cdot, \cdot)$ reweights the column vectors in $\D_{\rm LSPC}$ to ensure that each column vector in $\W_{\rm LSPC}$ has norm 1 and sums to 0. 

Another explanation is that when we recursively partition a node set into two local sets, each partition generates a wavelet basis vector. We still use Figure~\ref{fig:decomposition} as an example. In Partition 1, we partition the entire node set $S_{0,1} = \{1, 2, 3, 4 \}$ into $S_{1,1} =  \{1, 2 \}, S_{1,2} =  \{3, 4 \}$ and generate a basis vector   
\begin{eqnarray*}
&& \sqrt{ \frac{|S_{1,1}| |S_{1,2}|}{|S_{1,1}|+|S_{1,2}|}} \left(  \frac{1}{|S_{1,1}|} {\bf 1}_{S_{1,1}} - \frac{1}{|S_{1,2}|} {\bf 1}_{S_{1,2}}\right) 
\\
& = & \frac{1}{2}\begin{bmatrix}
1& 1 & -1 & -1
\end{bmatrix};
\end{eqnarray*}
in Partition 2, we partition $S_{1,1}$ into two disjoint connected sets $S_{2,1} =  \{1 \}, S_{2,2} =  \{ 2 \}$ and generate a basis vector   
\begin{eqnarray*}
&& \sqrt{ \frac{|S_{2,1} | |S_{2,2} |}{|S_{2,1} |+|S_{2,2} |}} \left(  \frac{1}{|S_{2,1} |} {\bf 1}_{S_{2,1} } - \frac{1}{|S_{2,2} |} {\bf 1}_{S_{2,2} }\right) 
\\
& = & \frac{1}{\sqrt{2}}\begin{bmatrix}
1& -1 & 0 & 0
\end{bmatrix};
\end{eqnarray*}
in Partition 3, we partition $S_{1,2}$ into $S_{2,3} =  \{3 \}, S_{2,4} =  \{ 4 \}$ and generate a basis vector   
\begin{eqnarray*}
&& \sqrt{ \frac{|S_{2,3}| |S_{2,4}|}{|S_{2,3}|+|S_{2,4}|}} \left(  \frac{1}{|S_{2,3}|} {\bf 1}_{S_{2,3}} - \frac{1}{|S_{2,4}|} {\bf 1}_{S_{2,4}}\right) 
\\
& = & \frac{1}{\sqrt{2}}\begin{bmatrix}
0 & 0 & 1 & -1
\end{bmatrix}.
\end{eqnarray*}

We summarize the construction of  the~\emph{local-set-based wavelet basis} in Algorithm~\ref{alg:wavelet}.
\begin{algorithm}[h]
  \footnotesize
  \caption{\label{alg:wavelet} Local-set-based  Wavelet Basis Construction }
  \begin{tabular}{@{}lll@{}}
    \addlinespace[1mm]
   {\bf Input} 
      & $G(\V, \E, \Adj )$~~graph \\
     {\bf Output}  
      & $\W_{\rm LSPC}$~~~~~~wavelet basis \\
    \addlinespace[2mm]
    {\bf Function} & &\\
    & initialize a stack of node sets $\mathbb{S}$ and a set of vectors $\W$ \\
    & push $S = \V$ into $\mathbb{S}$ \\ 
    & add $\w = \frac{1}{ \sqrt{|S|}} {\bf 1}_S$ into $\W_{\rm LSPC}$ \\
    & while the cardinality of the largest element of $\mathbb{S}$ is bigger than $1$  \\ 
    &~pop up one element from $\mathbb{S}$ as $S$ \\
    &~partition $S$ into two disjoint connected sets $S_1, S_2$ \\
    &~push $S_1, S_2$ into $\mathbb{S}$ \\
    &~add $\w = \sqrt{ \frac{|S_1| |S_2|}{|S_1|+|S_2|}} \left(  \frac{1}{|S_1|} {\bf 1}_{S_1} - \frac{1}{|S_2|} {\bf 1}_{S_2}\right)$ into $\W_{\rm LSPC}$\\
    & end \\
    & {\bf return} $\W_{\rm LSPC}$ \\  
     \addlinespace[1mm]
  \end{tabular}
\end{algorithm}

\subsubsection{Properties}
We now analyze some properties of the proposed construction of the local sets and wavelet basis. The main results are
\begin{itemize}
\item  the local-set-based dictionary provides a multiresolution representation;
\item there exists a tradeoff between smoothness and fine resolution in partitioning the local sets;
\item  the local-set-based wavelet basis is an orthonormal basis;
\item the local-set-based wavelet basis promotes sparsity for piecewise-constant graph signals.
\end{itemize}

\begin{myThm}
\label{thm:MRA}
The proposed construction of local sets satisfies the multiresolution analysis on graphs.
\end{myThm}
We have shown this in the previous section. We list here for completeness. In the original multiresolution analysis, more localization in the  time domain leads to more high-frequency components. Here we show a similar result.
\begin{myThm}
\label{thm:smooth}
A series of local sets with a finer resolution is less smooth, that is, for all $i$,
\begin{equation*}
\sum_{j=1}^{2^i} \STV_{\Delta, 0} (S_{i,j}) \leq  \sum_{j=1}^{2^{i+1}}  \STV_{\Delta, 0}  (S_{i+1,j}) .
\end{equation*}
\end{myThm}
\begin{proof}
We first show that the sum of variations of two children local sets is larger than the variation of the parent local set.
\begin{eqnarray*}
&& \STV_{\Delta, 0} (S_{i+1,2j-1}) + \STV_{\Delta, 0} (S_{i+1,2j})  
\\
& = & \frac{1}{ |S_{i+1,2j-1}| }\left\| \Delta {\bf 1}_{ S_{i+1,2j-1} } \right\|_0 + 
\frac{1}{ |S_{i+1,2j}| }\left\| \Delta {\bf 1}_{ S_{i+1,2j} } \right\|_0
\\
& \stackrel{(a)}{ \geq } &  \frac{1}{ |S_{i,j}| }  \left(  \left\| \Delta {\bf 1}_{ S_{i+1,2j-1} } \right\|_0 + \left\| \Delta {\bf 1}_{ S_{i+1,2j} }  \right\|_0 \right)
\\
& \stackrel{(b)}{\geq} &  \frac{1}{ |S_{i,j}| }\left\| \Delta {\bf 1}_{S_{i,j}} \right\|_0 \ = \ \STV_{\Delta, 0}  (S_{i,j}),
\end{eqnarray*}
where $(a)$ follows from that the cardinality of the the parent local set is larger than either of its children local sets and $(b)$ follows from that we need to cut a boundary to partition two children local sets. Since every local set in the $i$th level has two children local sets in the $i+1$th level and every local set in the $i+1$th level has a parent local set in the $i$th local, we sum them together and obtain Theorem~\ref{thm:smooth}.
\end{proof}

Theorem~\ref{thm:smooth} shows that  by zooming in the graph vertex domain, the partitioned local sets get less smooth; in other words, we have to tradeoff the smoothness to obtain a finer resolution.

We next show that the local-set-based wavelet basis is a valid orthonormal basis.
\begin{myThm}
The local-set-based wavelet basis constructs an orthonormal basis.
\end{myThm}
\begin{proof}
First, we show each vector has norm one.
\begin{eqnarray*}
&& \left\| \sqrt{ \frac{|S_1| |S_2|}{|S_1|+|S_2|}} \left(  \frac{1}{ | S_1 |} {\bf 1}_{S_1} - \frac{1}{| S_2 |} {\bf 1}_{S_2}\right)  \right\|_2^2 
\\
&  \stackrel{(a)}{=} & 
|S_1|  \left( \sqrt{\frac{|S_1| |S_2|}{|S_1|+|S_2|}}  \frac{1}{ | S_1 |}  \right)^2 +  |S_2|  \left( \sqrt{\frac{|S_1| |S_1|}{|S_1|+|S_2|}}  \frac{1}{ | S_2 |}  \right)^2 
\\
& = & 1,
\end{eqnarray*}
where $(a)$ follows from that $S_1 \cap S_2 = \emptyset$.
Second, we show each vector is orthogonal to the other vectors. We have
\begin{eqnarray*}
{\bf 1}^T \w = \sqrt{ \frac{|S_1| |S_2|}{|S_1|+|S_2|}} \left( \sum_{i \in S_1} \frac{1}{ | S_1 |}  -  \sum_{i \in S_2} \frac{1}{| S_2 |}  \right) = 0.
\end{eqnarray*}
Thus, each vector is orthogonal to the first vector, ${\bf 1}_{\V}/\sqrt{|{\V}| }$. Each other individual vector is generated from two node sets. Let $S_1, S_2$ generate $\w_i$ and $S_3, S_4$ generate $\w_j$. Due to the construction, there are only two conditions, two node sets of one vector belongs to one node set of the other vector, and all four node sets do not share element with each other. For the first case, without losing generality, let $\left( S_3 \cup S_4 \right) \cap S_1=  S_3 \cup S_4$, we have
\begin{eqnarray*}
 \w_i^T \w_j 
& = &  \sqrt{ \frac{|S_1| |S_2|}{|S_1|+|S_2|}
 \frac{|S_3| |S_4|}{|S_3|+|S_4|}} \left( \sum_{i \in S_3} \frac{1}{ | S_3 |}  -  \sum_{i \in S_4} \frac{1}{| S_4 |} \right) 
\\ 
& = & 0.
\end{eqnarray*}
For the second case, the inner product between  $\w_i$ and $\w_j $ is zero because their supports do not match.
Third, we show that $\W_{\rm LSPC}$ spans $\R^N$. Since we recursively partition the node set until the cardinalities of all the node sets are smaller than 2, there are $N$ vectors in $\W_{\rm LSPC}$.
\end{proof}

We show that the local-set-based wavelet basis is a good representation for piecewise-constant graph signals through promoting the sparsity.
\begin{myThm}
\label{thm:sparse}
Let $\W$ be the output of Algorithm~\ref{alg:wavelet} and $T$ be the maximum level of the decomposition in Algorithm~\ref{alg:wavelet} . For all $\x \in \R^N$, we have
\begin{eqnarray*}
\left\| \W_{\rm LSPC}^T \x \right\|_0 \leq  \left\|  \Delta \x \right\|_0  T.
\end{eqnarray*}
\end{myThm}

\begin{proof}
When an edge $e \in {\rm Supp}(\Delta \w)$, where Supp denotes the indices of nonzero elements, we call that the edge $\e$ is activated by the vector $\w$. Since each edge is activated at most once in each decomposition level, so each edge is activated by at most $T$ basis elements. Let activations($e$) be the number of basis elements in $\W_{\rm LSPC}$ that activates $e$. 
\begin{eqnarray*}
\left\| \W_{\rm LSPC}^T \x \right\|_0 \leq  \sum_{ e \in {\rm Supp}(\Delta \w) }  {\rm activations}(e) \leq \left\|  \Delta \x \right\|_0  T.
\end{eqnarray*}
\end{proof}
The maximum level of the decomposition is determined by the choice of graph partition algorithm. Theorem~\ref{thm:sparse} shows that what it matters is the cardinality of each local set, instead of the shape of each local set. To achieve the best sparse representation, we should partition each local set as evenly as possible. Note that when the partition is perfectly even, the resulting wavelet basis is the same with the classical Haar wavelet basis.

\begin{myCorollary}
\label{thm:sparse}
For all $\x \in \R^N$,  we have $\left\| \a^* \right\|_0 \leq  2 T \left\|  \Delta \x \right\|_0$, where $T$ is the maximum level of the decomposition and
\begin{eqnarray*}
\a^* & = &   \arg \min_{\a}   \left\| \a \right\|_0, 
\\
   && {\rm subject~to:~}  \x \ = \ \D_{\rm LSPC} \a.
\end{eqnarray*}
\end{myCorollary}
Since $\W$ is obtained from $ \D_{\rm LSPC}$,  the worst case is to represent $\x$ by using $\W$.

\begin{myCorollary}
Let the local-set-based wavelet basis evenly partition the node set each time. We have
\begin{eqnarray*}
\left\| \W_{\rm LSPC}^T \x \right\|_0 \leq  \left\|  \Delta \x \right\|_0 \ceil*{\log N}.
\end{eqnarray*}
\end{myCorollary}
We see that the local-set-based piecewise-constant wavelet basis provides a sparse representation for the piecewise-constant graph signals. We conjecture the proposed construction is the optimal orthonormal basis to promote the sparsity for piecewise-constant graph signals.
\begin{myConj} Let $\W$ be the local-set-based wavelet basis with the even partition. Then,
\begin{eqnarray*}
 \W_{\rm LSPC} & = & \arg \min_{\F}  \max_{\x \in \{ \left\|  \Delta \x \right\|_0 \leq K \} }  \left\|  \F^T \x \right\|_0, 
 \\
 && {\rm subject~to}: \F {\rm~is~orthornormal}.
\end{eqnarray*}
\end{myConj}

In general, the graph difference operator provides more sparse representation than the local-set-based wavelet basis, however, the graph difference operator is not necessarily a one-to-one mapping and is bad at reconstruction; the graph difference operator only focuses on the pairwise relationship. On the other hand, the local-set-based wavelet basis is good at reconstruction and provides multiresolution view in the graph vertex domain.

The even partition minimizes the worst case; it does not necessarily mean that the even partition is good for all the applications. For example, a graph has two communities, a huge one and a tiny one, which hints that a piecewise-constant graph signal sits on a part of either of two communities. In this case, we cut a few edges to partition two communities and assign a local set for each of them, instead of partitioning the huge community to make sure that two local sets have a same cardinality.

\subsection{Graph Signal Processing Tasks}

\subsubsection{Approximation}
\label{sec:pc_app}
Approximation is a standard task to evaluate a representation and it is similar to compression. The goal is to use a few expansion coefficients to approximate a graph signal. We compare the graph Fourier transform~\cite{SandryhailaM:13}, the windowed graph Fourier transform~\cite{ShumanRV:15}, the local-set-based wavelet basis and dictionary. The graph Fourier transform is the eigenvector matrix of the graph shift and the windowed graph Fourier transform provides vertex-frequency analysis on graphs. For the local-set-based piecewise-constant wavelet basis and dictionary, we also consider three graph partition algorithms, including spectral clustering, spanning tree and 2-means.

\mypar{Algorithm}
Since the graph Fourier transform and the local-set-based wavelet bases are orthonormal bases, we consider nonlinear approximation for the graph Fourier bases, that is, after expanding in with a representation, we should choose the $K$ largest-magnitude expansion coefficients so as to minimize the approximation error. Let $\{ \phi_k \in \R^N \}_{k=1}^N$ and  $\{ \widehat{\phi}_k \in \R^N \}_{k=1}^N$ be a pair of biorthonormal basis and $\x \in  \R^N $ be a signal. Here the graph Fourier transform matrix $\Um = \{ \widehat{\phi}_k \}_{k=1}^N$ and the graph Fourier basis $\Vm = \{ \phi_k \}_{k=1}^N$. The nonlinear approximation to $\x$ is 
\begin{eqnarray}
\label{eq:nonlinear_approx}
   \x^* = \sum_{k \in \mathcal{I}_K} \left\langle {\x, \widehat{\phi}_k} \right\rangle \phi_k,
\end{eqnarray}
where $\mathcal{I}_K$ is the index set of the $K$ largest-magnitude expansion coefficients. When a basis promotes sparsity for $\x$, only a few expansion coefficients are needed to obtain a small approximation error. Note that~\eqref{eq:nonlinear_approx} is a special case of~\eqref{eq:sparse_coding} when the distance metric $d(\cdot, \cdot)$ is the $\ell_2$ norm and $\D$ is an orthonormal basis.

Since the the windowed graph dictionary and the local-set-based piecewise-constant dictionaries are redundant, we solve the following sparse coding problem,
\begin{eqnarray}
\label{eq:sparse_coding_pc}
   \x' =  &  \arg \min_{\a} & \left\| \x - \D \a  \right\|_2^2,
   \\
   \nonumber
   & {\rm subject~to:~} &  \left\| \a \right\|_0 \leq K,
\end{eqnarray}
where $\D$ is a redundant dictionary and $\a$ is a sparse code. The idea is to use a linear combination of a few atoms from $\D$ to approximate the original signal. When $\D$ is an orthonormal basis, the closed-form solution is exactly~\eqref{eq:nonlinear_approx}. We solve~\eqref{eq:sparse_coding} by using the orthogonal matching pursuit, which is a greedy algorithm~\cite{PatiRK:93}. Note that~\eqref{eq:sparse_coding_pc} is a special case of~\eqref{eq:sparse_coding} when the distance metric $d(\cdot, \cdot)$ is the $\ell_2$ norm.

\mypar{Experiments}
We test the four representations on two datasets, including the Minnesota road graph~\cite{MinnesotaGraph} and the U.S city graph~\cite{ChenSMK:14}.

For the Minnesota road graph, we simulate a piecewise-constant graph signal by randomly picking 5 nodes as community centers and assigning each other node to its nearest community center based on the geodesic distance. We assign a random integer to each community. The simulated graph signal is shown in Figure~\ref{fig:min_signal}. The signal contains $5$ piecewise constants and $84$ inconsistent edges. The frequency coefficients and the wavelet coefficients obtained by using three graph partition algorithms are shown in Figure~\ref{fig:min_signal}(b), (c), (d) and (e). The sparsities of the wavelet coefficients for spectral clustering, spanning tree, and 2-means are $364$, $254$, and $251$, respectively; the proposed wavelet bases provide much better sparse representations than the graph Fourier transform.

The evaluation metric of the approximation error is the normalized mean square error, that is,
\begin{equation*}
 {\rm Normalized~MSE} = \frac{ \left\|  \x' - \x \right\|_2^2 }{ \left\| \x \right\|_2^2 },
\end{equation*}
where $\x'$ is the approximation signal and $\x$ is the original signal. Figure~\ref{fig:min_signal}(f) shows the approximation errors given by the four representations. The x-axis is the number of coefficients used in approximation, which is $K$ in~\eqref{eq:nonlinear_approx} and~\eqref{eq:sparse_coding} and the y-axis is the approximation error, where lower means better.  We see that the local-set-based wavelet with spectral clustering and local-set-based dictionary with spectral clustering provides much better performances and the windowed graph Fourier transform catches up with graph Fourier transform around $15$ expansion coefficients. Figure~\ref{fig:min_signal}(g) and (h) compares the local-set-based wavelets and dictionaries with three different partition algorithms, respectively. We see that the spanning tree and 2-means have similar performances, which are better than spectral clustering. This is consistent with the sparsities of the wavelet coefficients, where the wavelet coefficients of spanning tree and 2-means are more sparse than those of spectral clustering.

\begin{figure}[htb]
  \begin{center}
    \begin{tabular}{ccc}
\includegraphics[width=0.4\columnwidth]{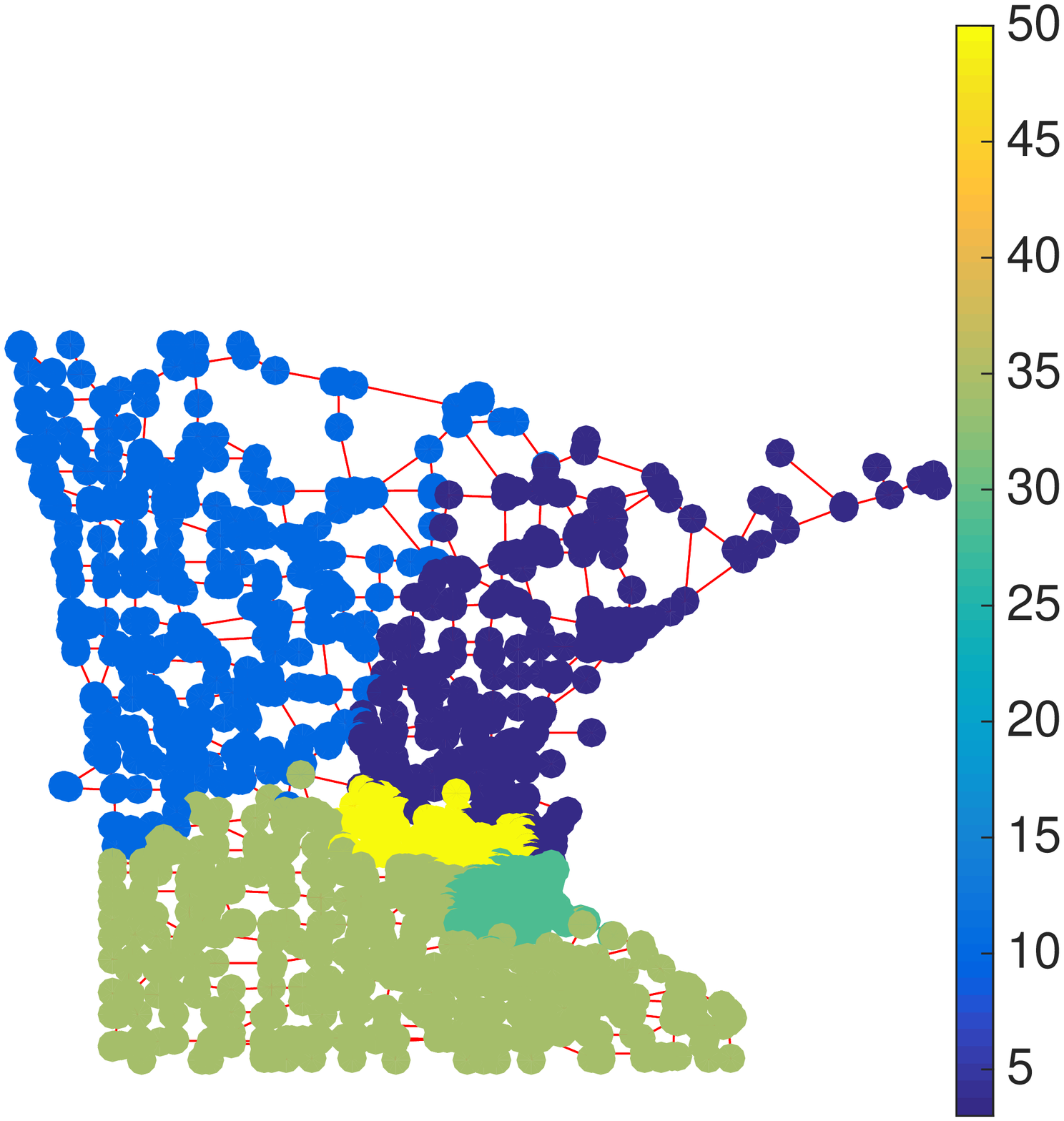}  & \includegraphics[width=0.4\columnwidth]{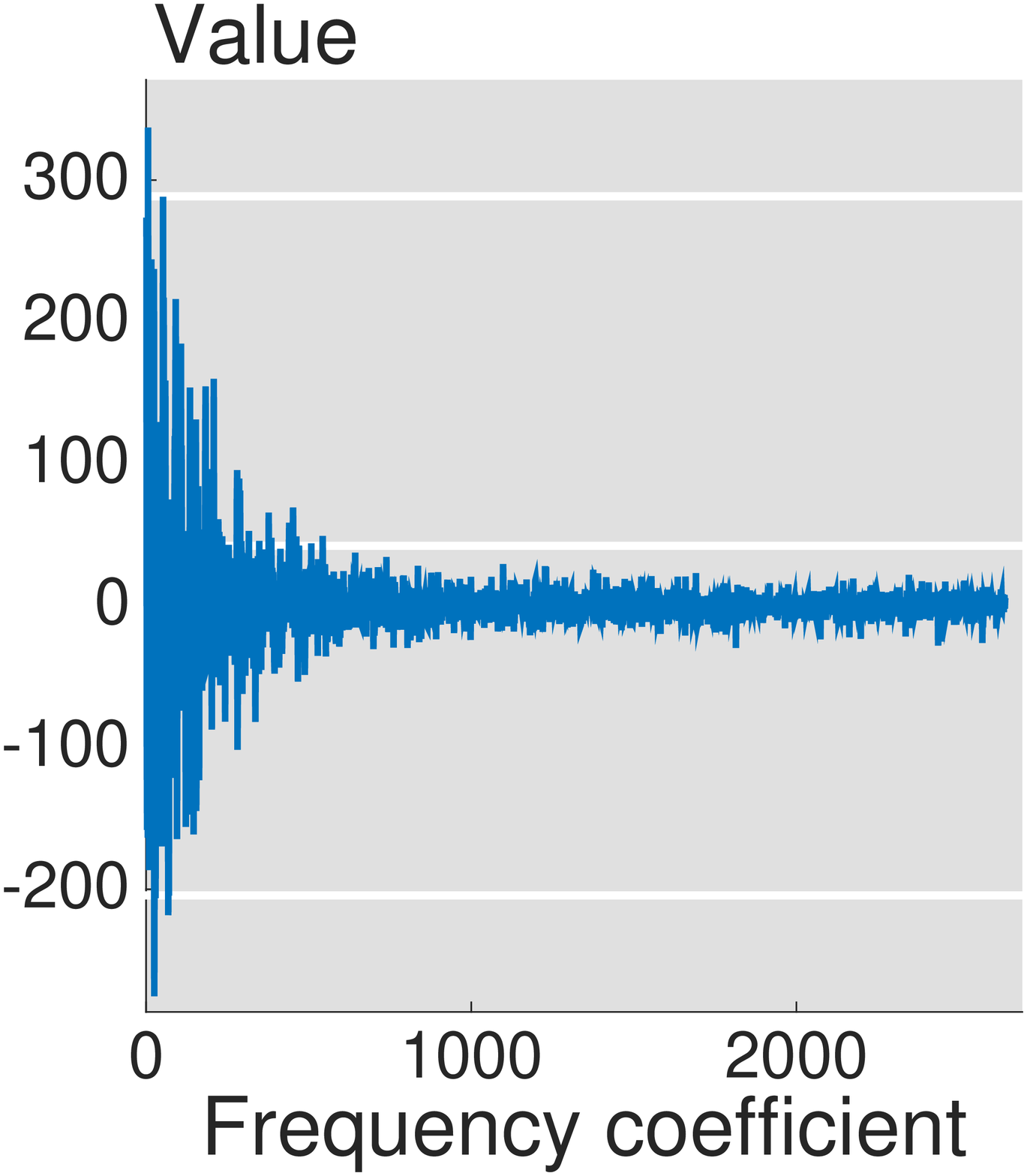} 
\\
 {\small (a) Signal.} & {\small (b) Frequency coefficients.} 
 \\
\includegraphics[width=0.4\columnwidth]{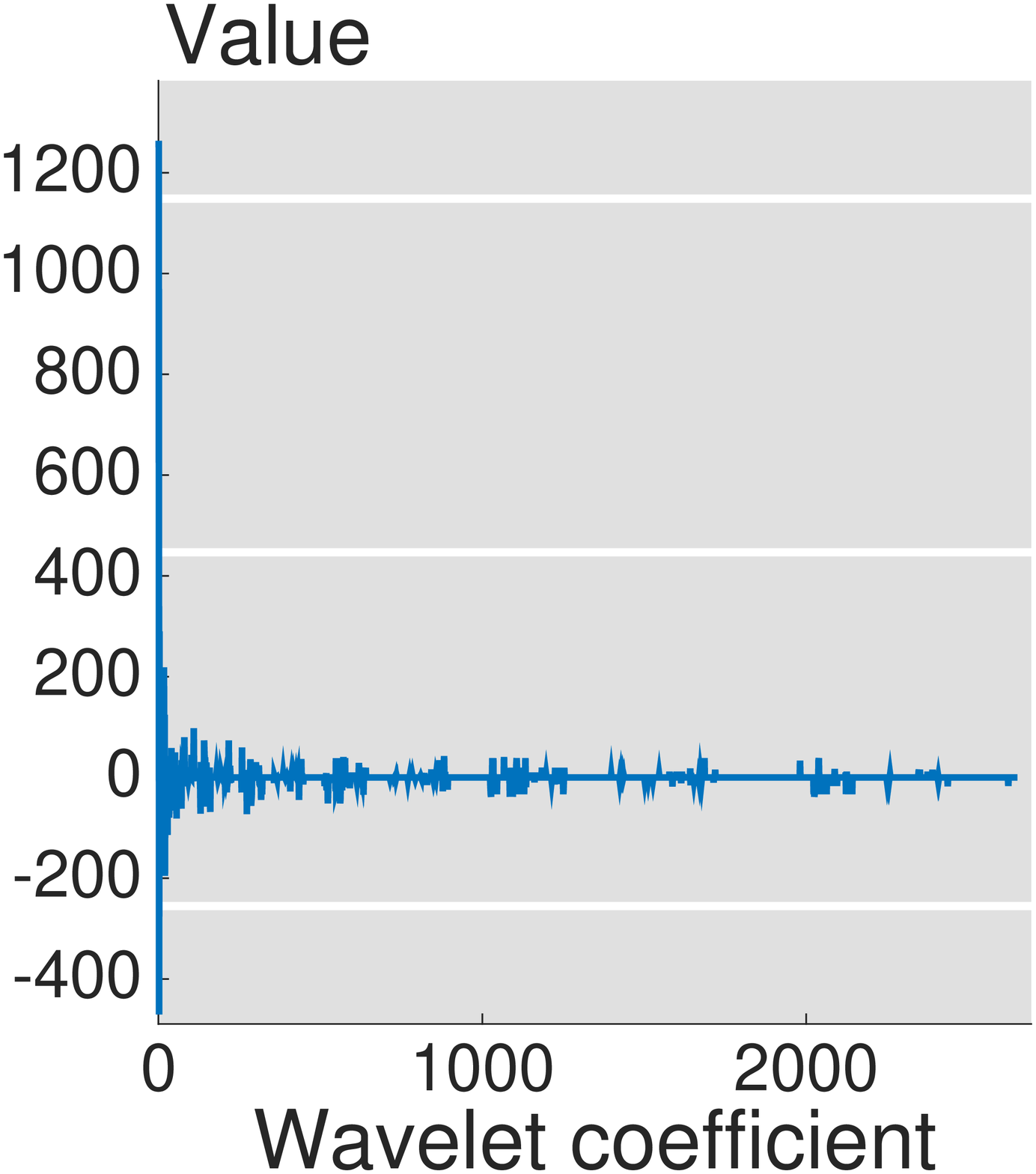}  &
\includegraphics[width=0.4\columnwidth]{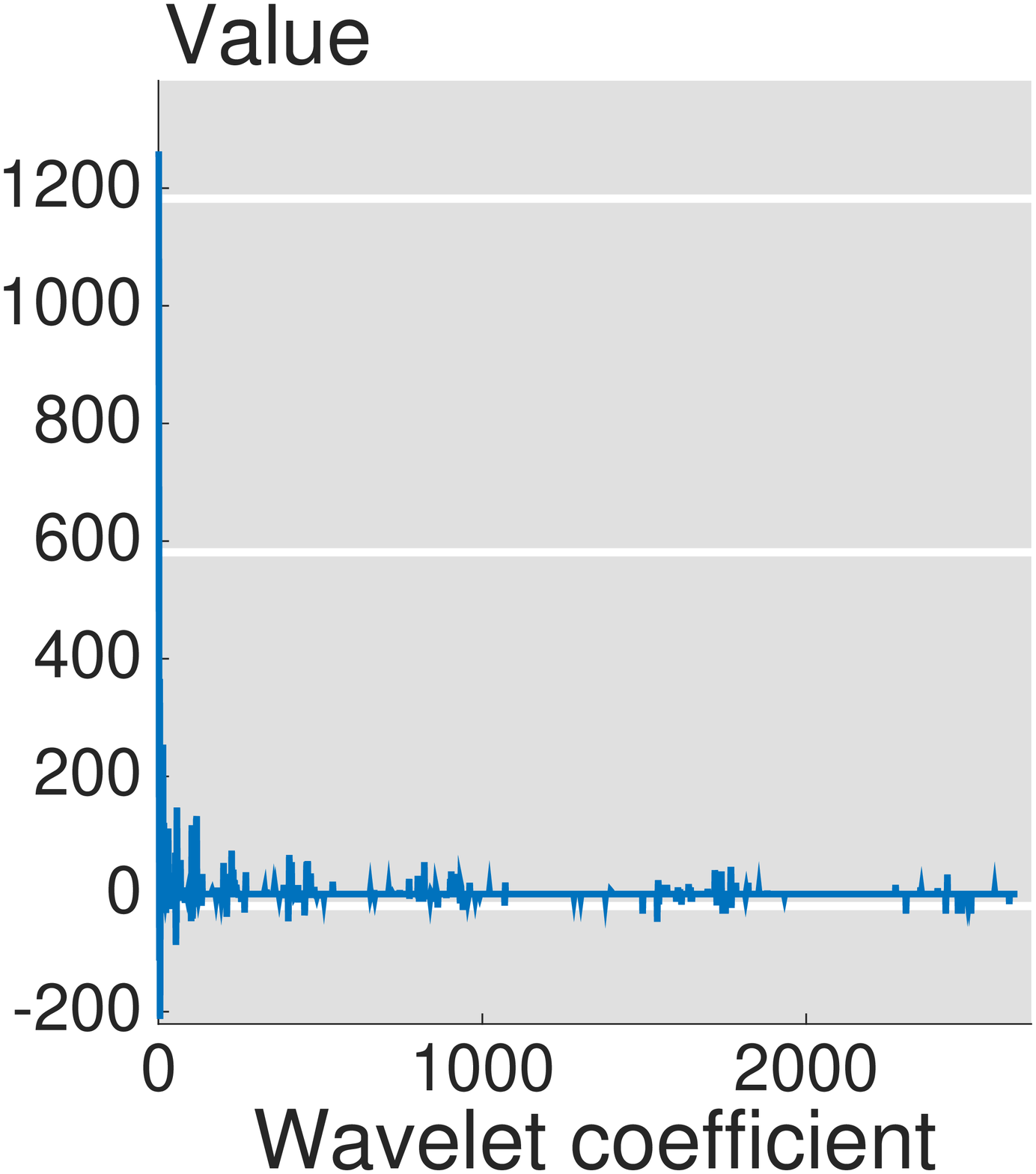}
\\
 {\small (c) Wavelet coefficients} & {\small (d) Wavelet coefficients} 
 \\
 {\small from spectral clustering.} & {\small from spanning tree.} 
 \\
\includegraphics[width=0.4\columnwidth]{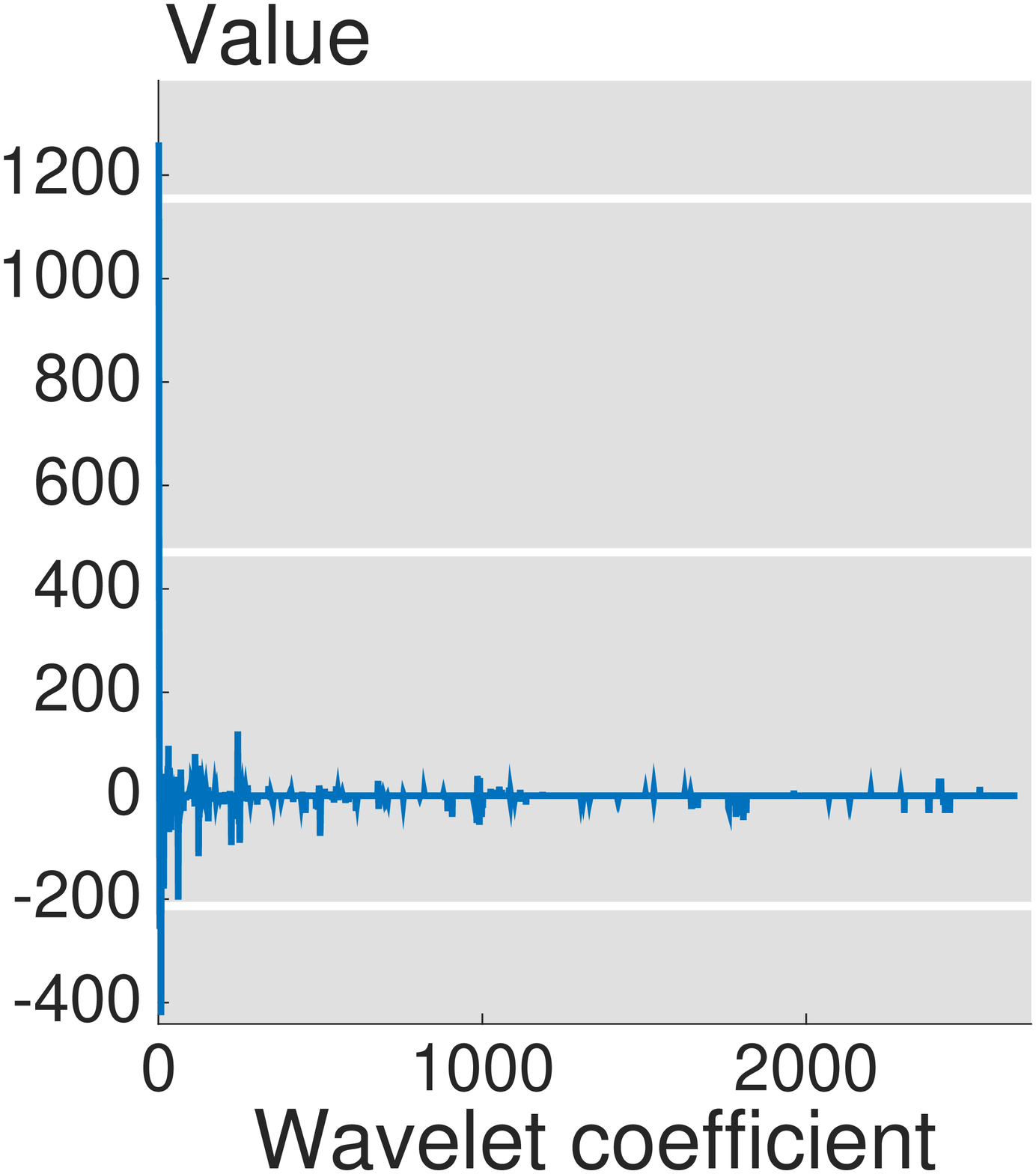}  &
\includegraphics[width=0.4\columnwidth]{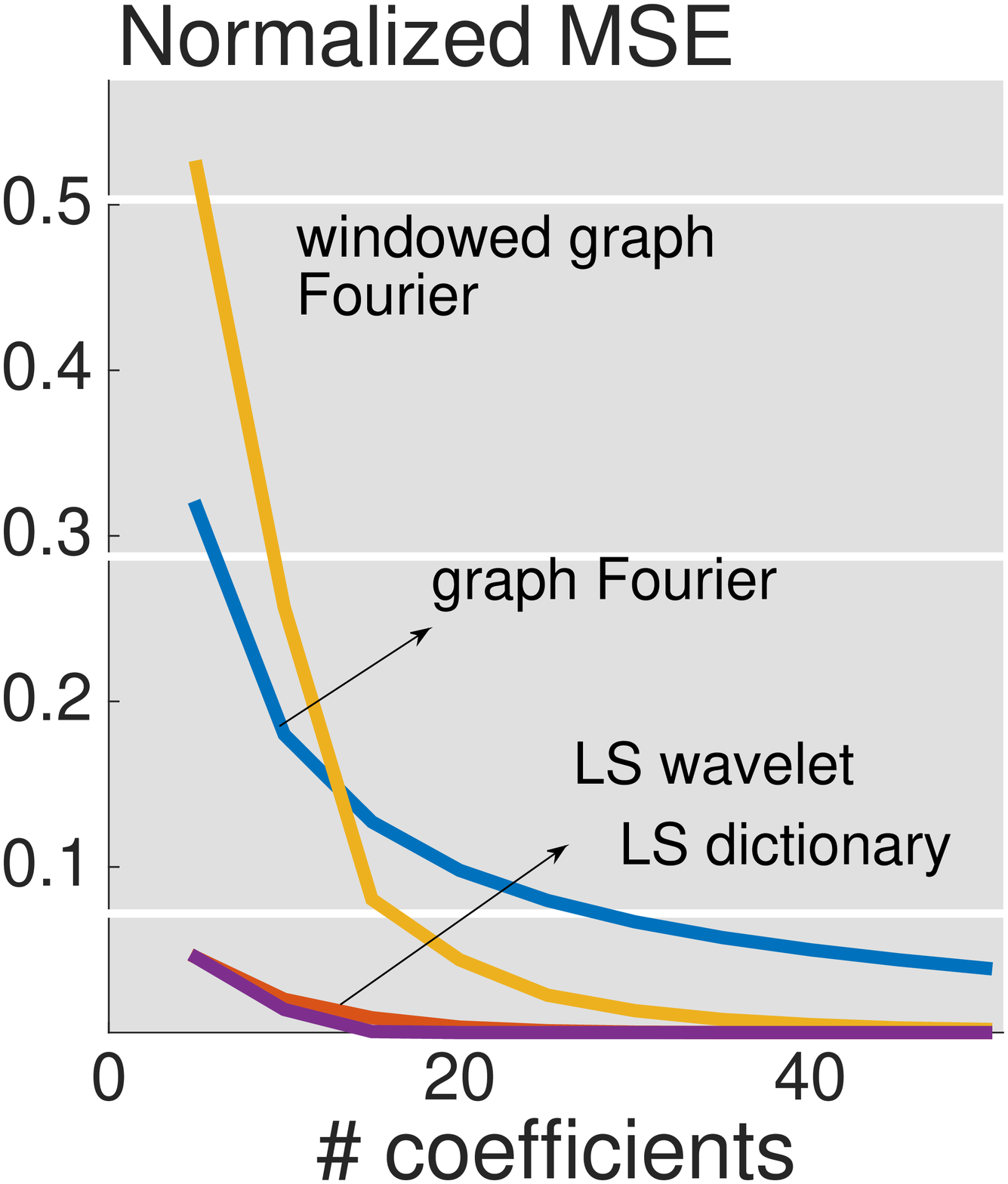}
\\
 {\small (e) Wavelet coefficients} & {\small (f) Error comparison.} 
 \\
 {\small from 2-means.} &
 \\
 \includegraphics[width=0.4\columnwidth]{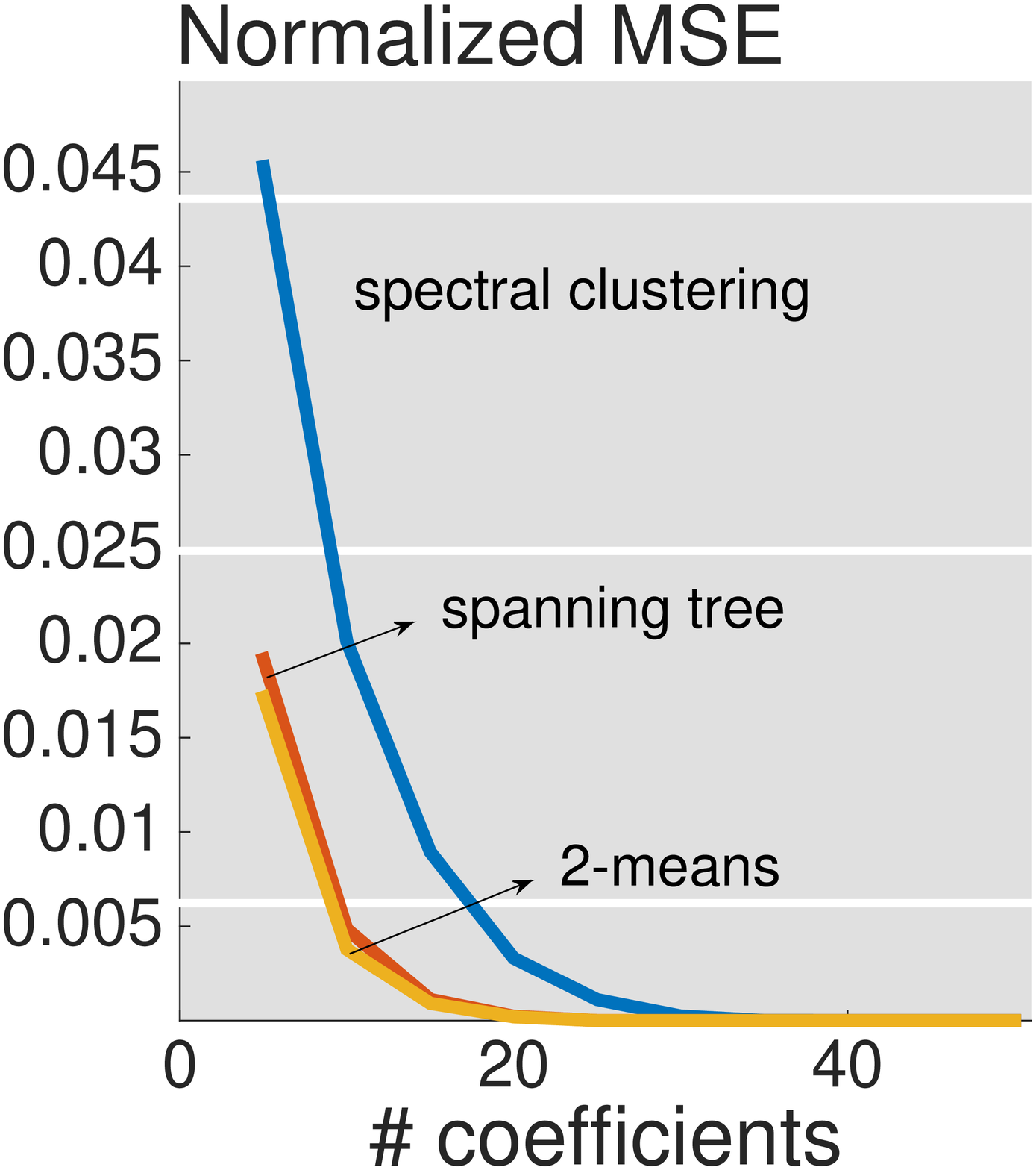} &
\includegraphics[width=0.4\columnwidth]{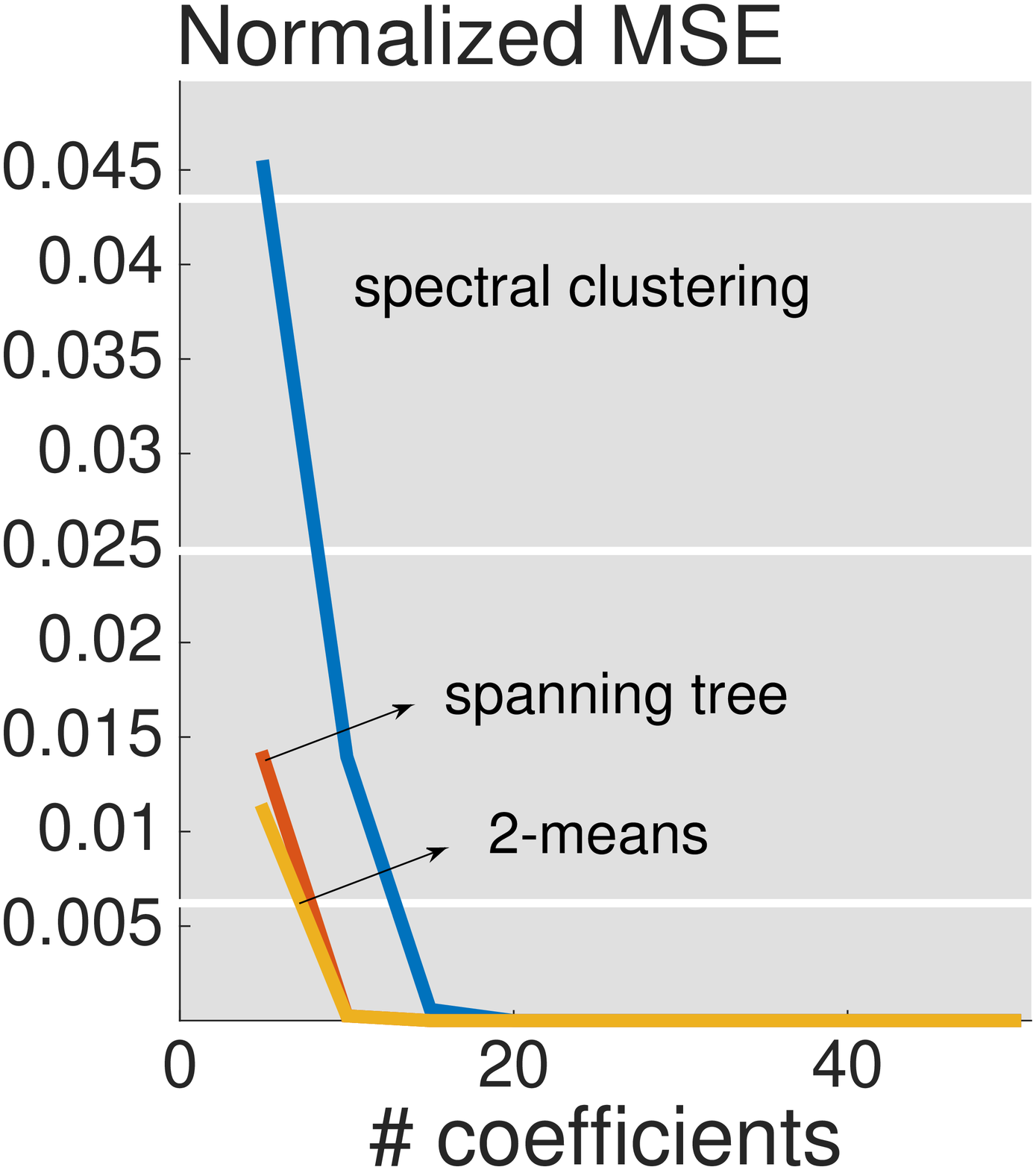}
\\
 {\small (g) Error (Wavelet)} & {\small (h) Error (Dictionary).} 
\end{tabular}
  \end{center}
   \caption{\label{fig:min_signal} Approximation on the Minnesota road graph.}
\end{figure}

The U.S city graph is a network representation of $150$ weather stations across the U.S. We assign an edge when two 
weather stations are within 500 miles. The graph includes $150$ nodes and $1033$ undirected, unweighted edges. Based on the  geographical area, we partition the nodes into four communities, including the north area (N), the middle area (M), the south area (S), and the west area (W). The corresponding piecewise-constant graph signal is
\begin{equation}
\label{eq:pc_us}
  \x =  {\bf 1}_{N} + 2 \cdot {\bf 1}_{M} + 3 \cdot {\bf 1}_{S} + 4 \cdot {\bf 1}_{W}.
\end{equation}
The graph signal is shown in Figure~\ref{fig:us_signal}(a), where dark blue indicates the north area, the light  indicates the middle area, the dark yellow indicates the south area and  the light yellow indicates the west area. The signal contains $4$ piecewise constants and $144$ inconsistent edges.

 The frequency coefficients and the wavelet coefficients obtained by using three graph partition algorithms are shown in Figure~\ref{fig:us_signal}(b), (c), (d) and (e). The sparsities of the wavelet coefficients for spectral clustering, spanning tree, and 2-means are $45$, $56$, and $41$, respectively; the proposed wavelet bases provide much better sparse representations than the graph Fourier transform.

The evaluation metric of the approximation error is also the normalized mean square error.  Figure~\ref{fig:us_signal}(f) shows the approximation errors given by the four representations. Similarly to Figure~\ref{fig:us_signal}(d), the local-set-based wavelet with spectral clustering and local-set-based dictionary with spectral clustering provides much better performances and the windowed graph Fourier transform catches up with graph Fourier transform around $25$ expansion coefficients. Figure~\ref{fig:us_signal}(g) and (h) compares the local-set-based wavelets and dictionaries with three different graph partition algorithms, respectively. We see that the spectral clustering provides the best performance.

\begin{figure}[htb]
  \begin{center}
    \begin{tabular}{ccc}
\includegraphics[width=0.4\columnwidth]{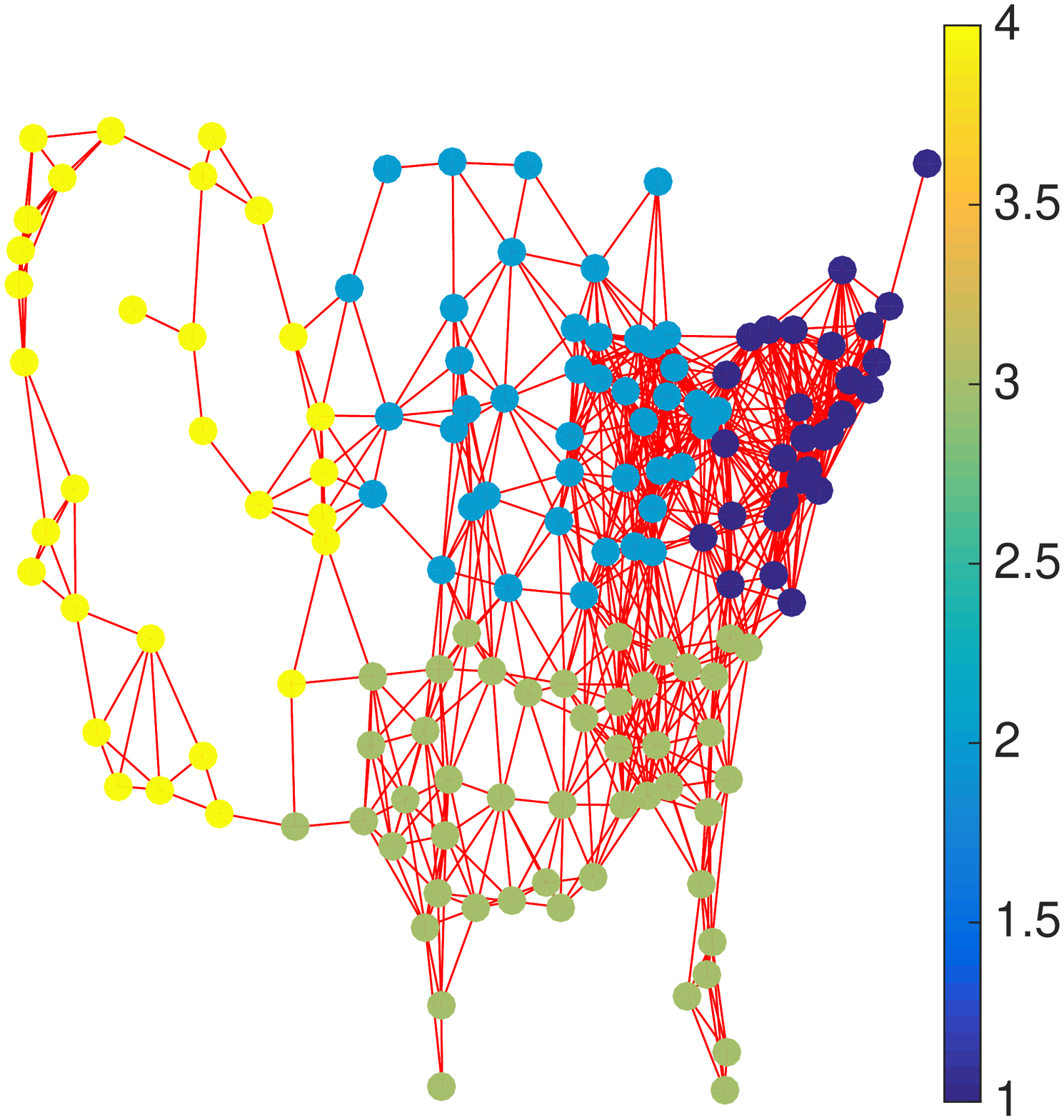}  & \includegraphics[width=0.4\columnwidth]{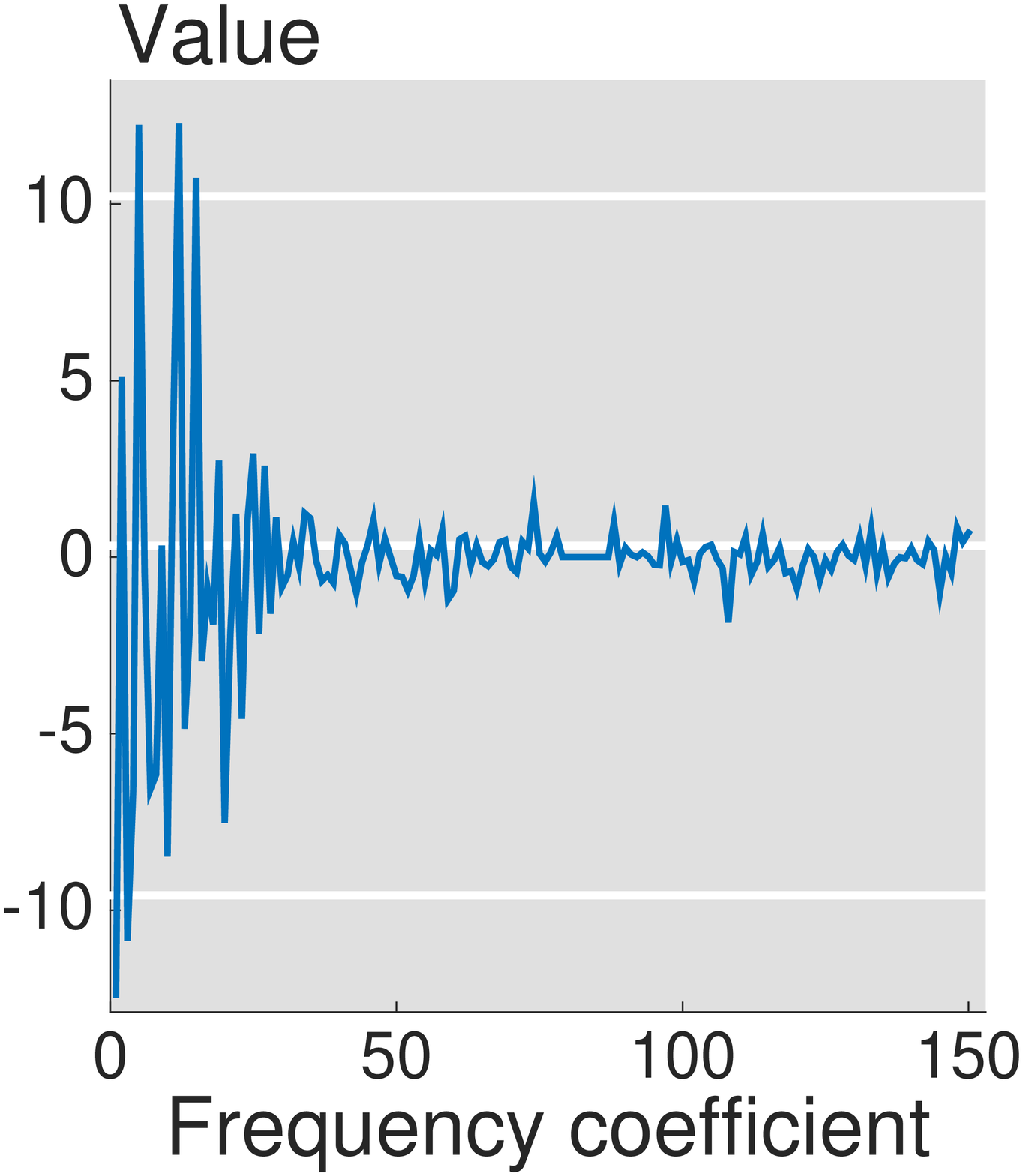} 
\\
 {\small (a) Signal.} & {\small (b) Frequency coefficients.} 
 \\
\includegraphics[width=0.4\columnwidth]{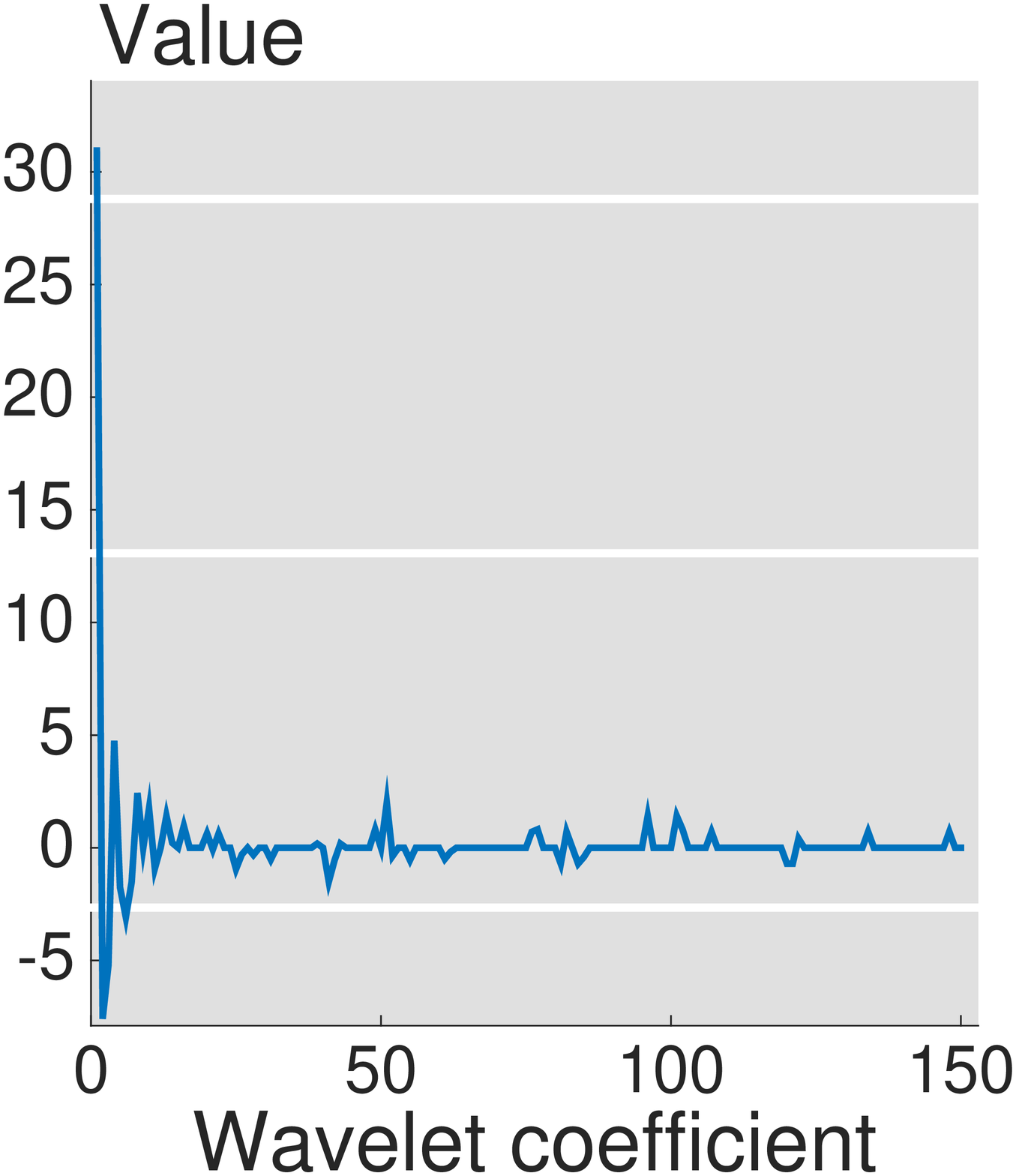}  &
\includegraphics[width=0.4\columnwidth]{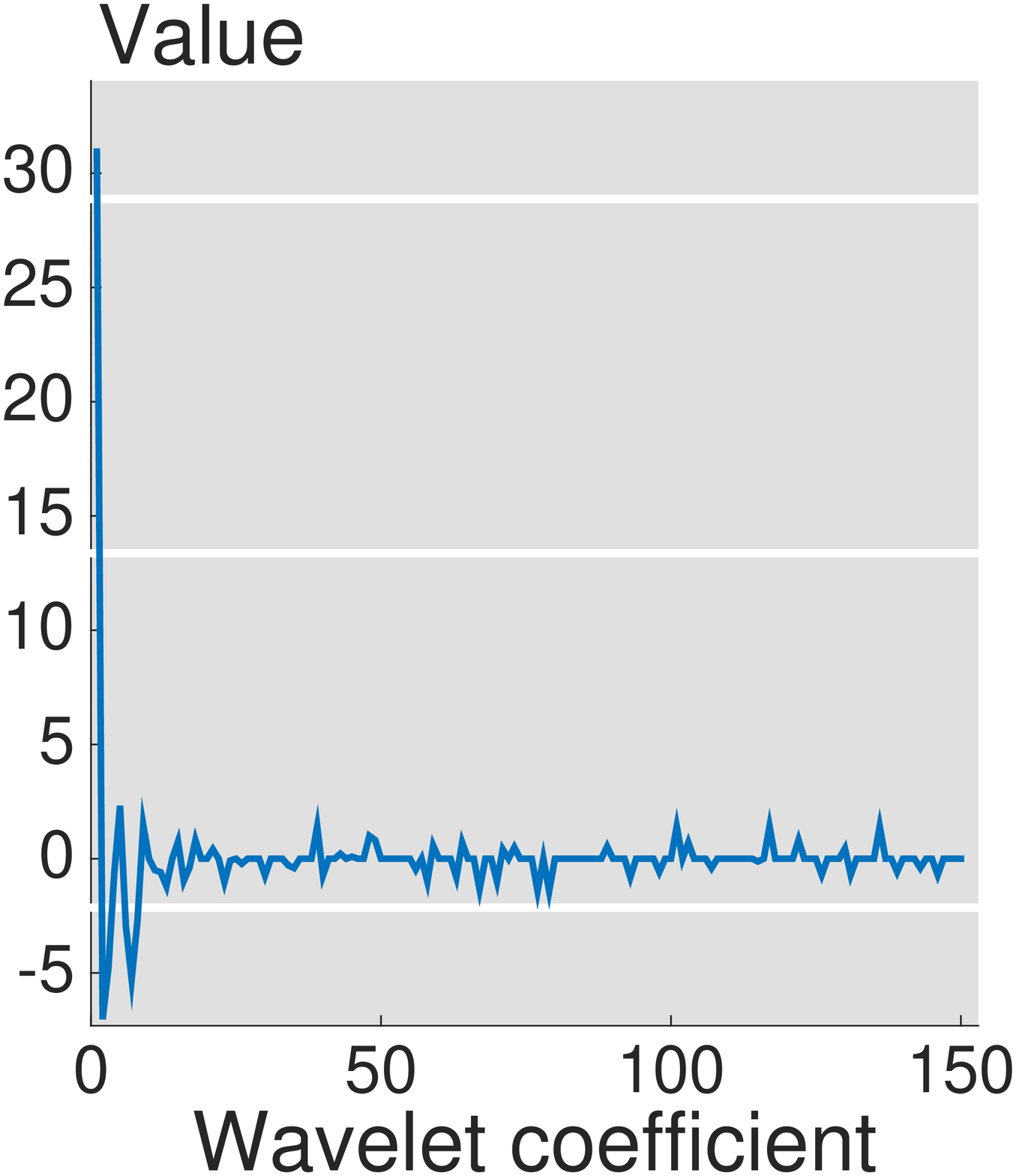}
\\
 {\small (c) Wavelet coefficients} & {\small (d) Wavelet coefficients} 
 \\
 {\small from spectral clustering.} & {\small from spanning tree.} 
 \\
\includegraphics[width=0.4\columnwidth]{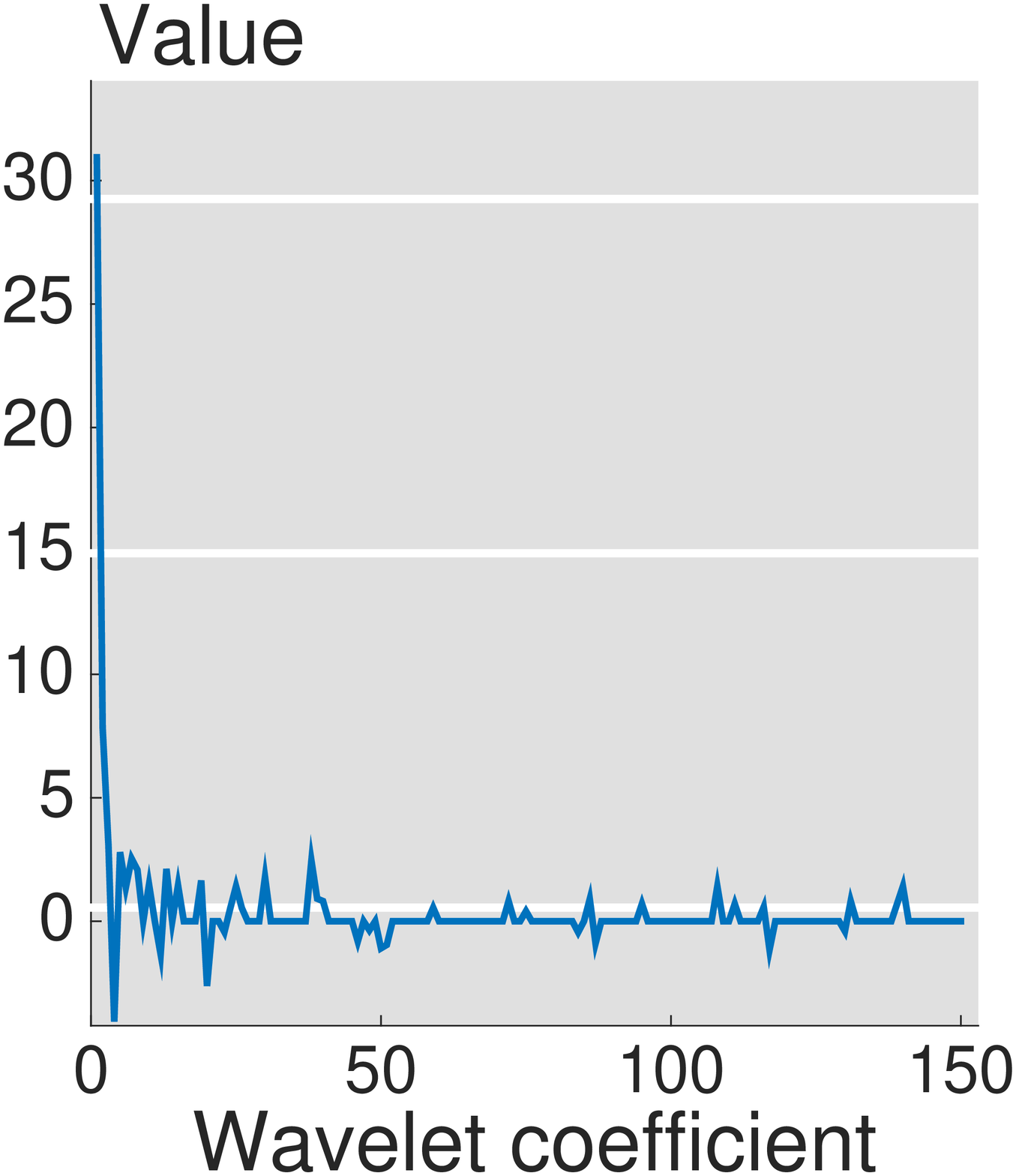}  &
\includegraphics[width=0.4\columnwidth]{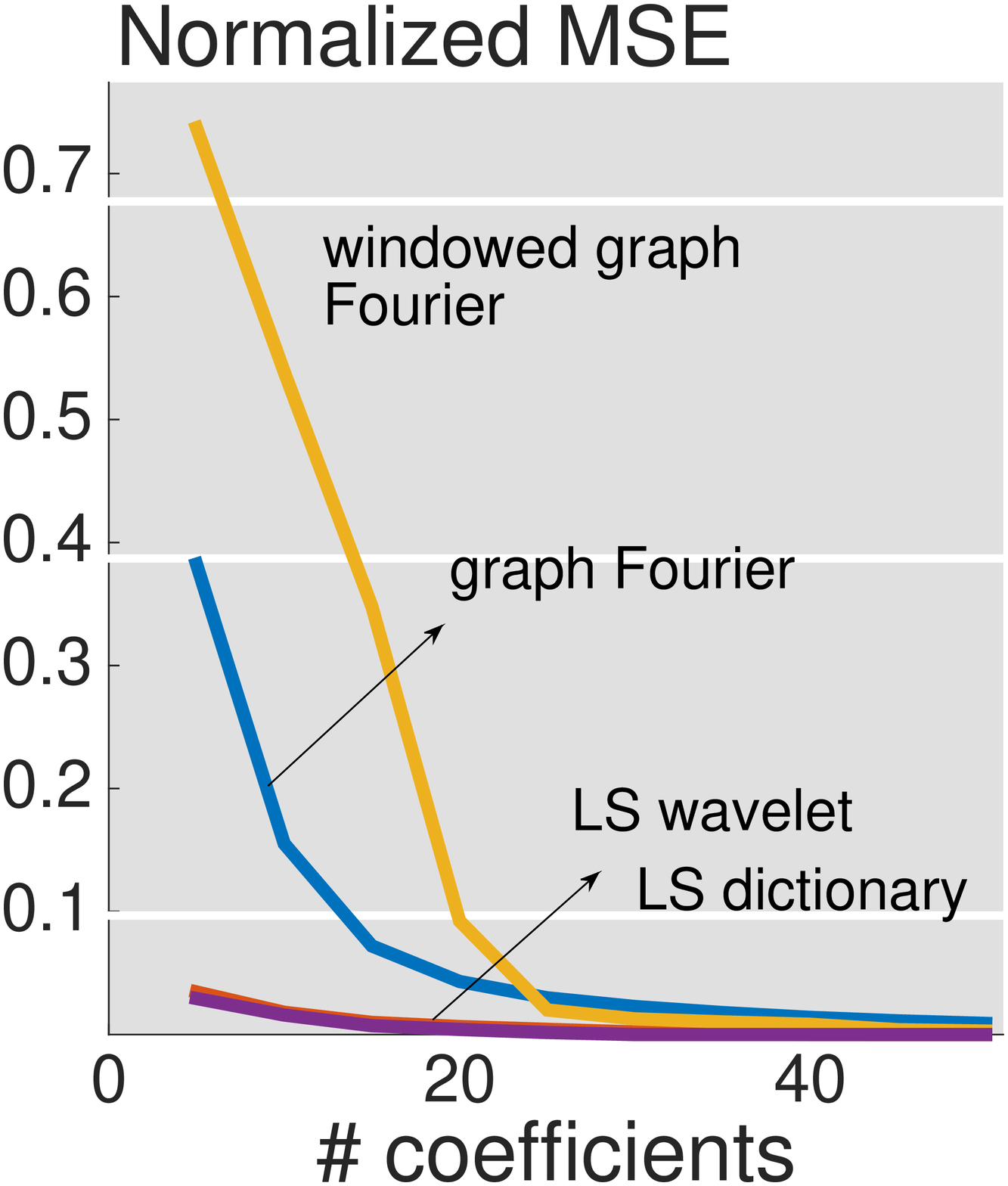} 
 \\
 {\small (e) Wavelet coefficients} & {\small (f) Error comparison.} 
 \\
 {\small from 2-means.} &
 \\
 \includegraphics[width=0.4\columnwidth]{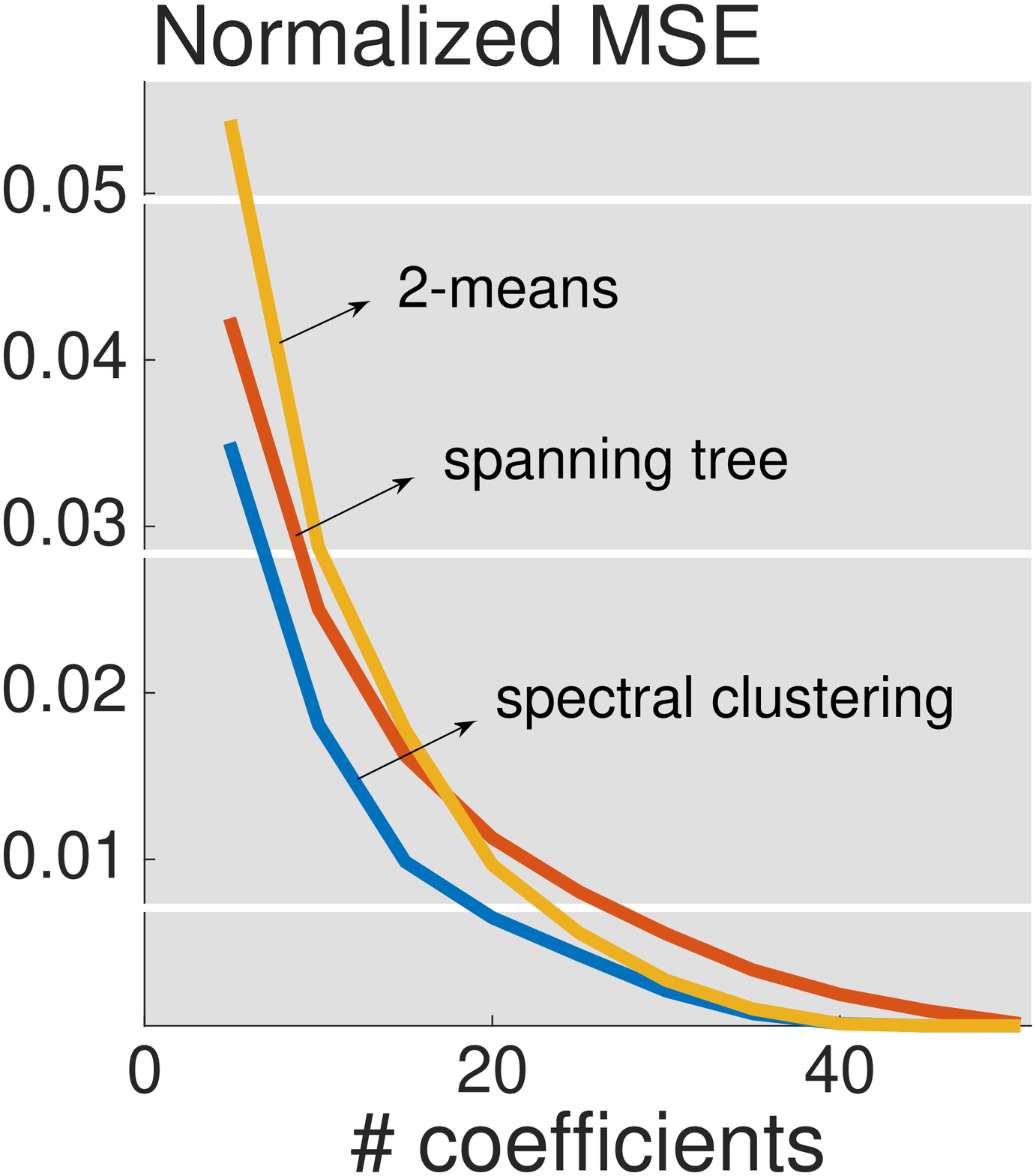} &
\includegraphics[width=0.4\columnwidth]{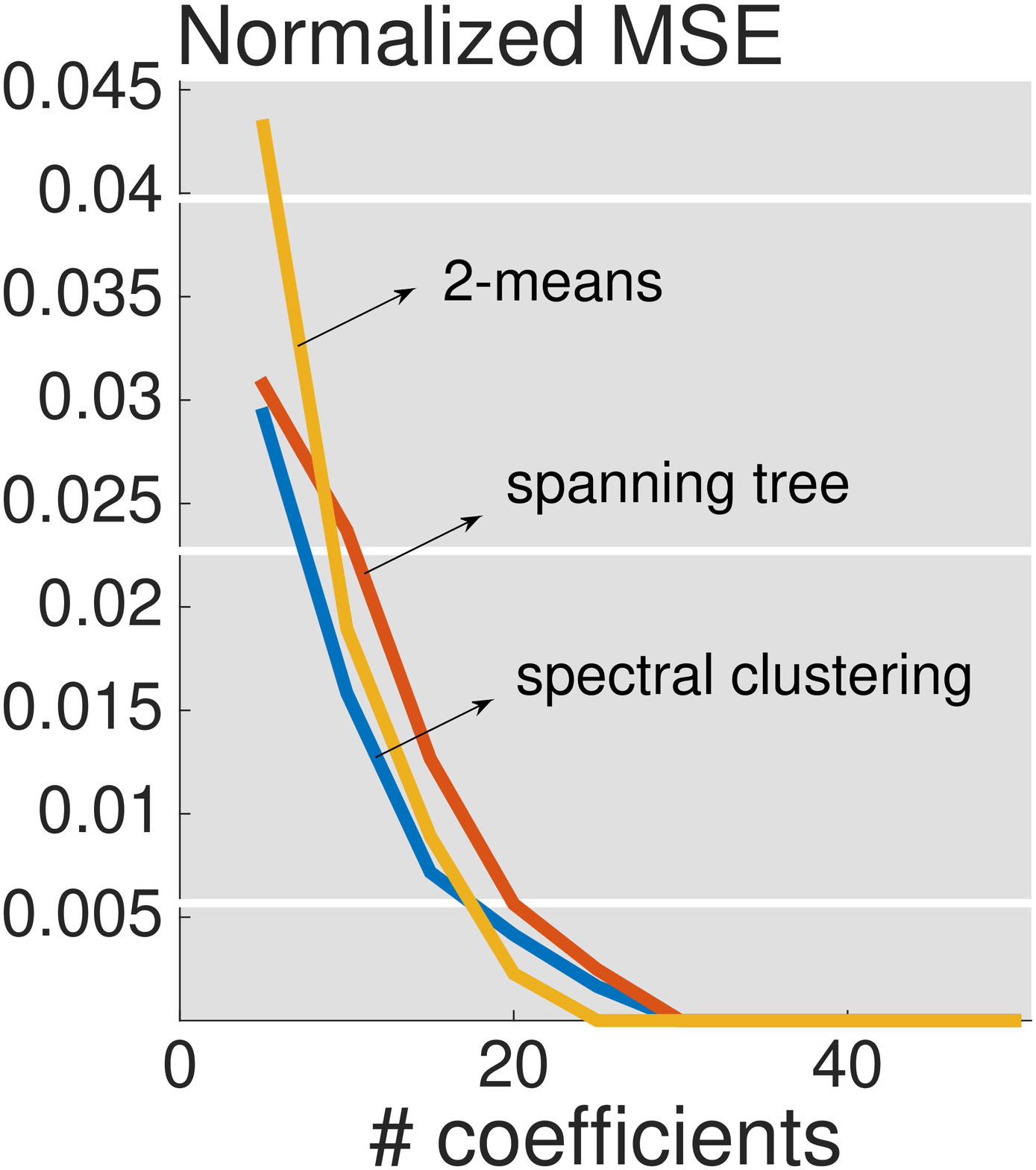}
\\
 {\small (g) Error (wavelet).} & {\small (h) Error (dictionary).} 
\end{tabular}
  \end{center}
   \caption{\label{fig:us_signal} Approximation on the U.S city graph.}
\end{figure}

To summarize the task of approximation, the proposed local set based representations provide a reliable approximation to a piecewise-constant graph signal because of the sparsity promotion.

\subsubsection{Sampling and Recovery}
\label{sec:SR}
The goal of sampling and recovery is to collect a few samples from a graph signal, and then to recover the original graph signal from those samples either exactly or approximately. Similar procedures are widely used in many applications, such as semi-supervised learning and active learning .

\mypar{Algorithm}
We consider the following recovery algorithm based on the multiresolution local sets. Let $m$ be the number of samples.
 We use the multiresolution decomposition of the local sets as shown in Section~\ref{sec:Dict_LSPC}.
 Instead of obtaining a full decomposition tree, we partition the local sets until we obtain $m$ leaf local sets. Those local sets may not be in the same decomposition level , but the union of them still cover the entire space. For the local sets in the same level of the decomposition tree, we first partition the one that has the largest number of nodes. For each leaf local set, we choose a center that has the minimum summation of the geodesic distances to all the other nodes in the leaf local set. We use those centers for the $M$  leaf local sets as the sampled set. Let $\x \in \R^N$ be a piecewise-constant graph signal, $\M = (\M_1, \cdots, \M_M)$ be the designed sampled set, with each sampled node $\M_j$ be the center of the $j$th leaf local set $S_j$. The recovered graph signal is 
\begin{equation*}
\x' \ = \   \sum_{j=1}^{M}  x_{\M_j} {\bf 1}_{S_j}.
\end{equation*}

We obtain a simple upper bound for the recovery error of this algorithm.
\begin{myThm}
\label{thm:recovery_err}
Let the original graph signal $\x \in \PC_G(K)$.
The recovery error is bounded as
\begin{eqnarray*}
\sum_{i=1}^{N} \Id ( x_i \neq x'_i ) \leq K \max_{j = 1, \cdots, M} | S_j |,
\end{eqnarray*}
where $\Id (\cdot)$ is the indicator function.
\end{myThm}
\begin{proof}
The error happens only when there exists at least one inconsistent edge in a community. Since there are $K$ inconsistent edges, we make errors in at most $K$ communities. The worst case is that each error is made in the one of the largest $K$ communities.
\end{proof}
Theorem~\ref{thm:recovery_err} shows that the size of the largest community influences the recovery error.  When we use the even partition, the size of the largest local set is minimized, which minimizes the upper bound. Similar to Theorem~\ref{thm:sparse}, Theorem~\ref{thm:recovery_err} also shows the importance of the even partition again. This algorithm studies the graph structure before taking samples, which belongs to the experimentally designed sampling. In the classical regression for piecewise-constant functions, it is known that experimentally designed sampling has the same performance with random sampling asymptotically~\cite{CastroWN:05}. When we restrict the problem setting to sample only a few nodes from a finite graph, however, random sampling can lead to the uneven partition where some communities are much larger than the others. As a deterministic approach, the experimentally designed sampling minimizes the error bound and is better than random sampling when the sample size is small.

We also consider two other recovery algorithms, including trend filtering on graphs and harmonic functions. For trend filtering on graphs, we consider
\begin{eqnarray*}
\label{eq:trend_filtering}
   \x' =  &  \arg \min_{\t} & \left\| \x_\M - \t_\M  \right\|_2^2 + \mu \left\| \Delta \t \right\|_1,
\end{eqnarray*}
where $\M$ is the sampling node set obtained by random sampling. We want to push the recovered graph signal to be close to the original one at the sampled nodes and to be piecewise constant. For harmonic functions, we consider
\begin{eqnarray*}
\label{eq:trend_filtering}
   \x' =  &  \arg \min_{\t} & \left\| \x_\M - \t_\M  \right\|_2^2 + \mu \left\| \Delta \t \right\|_2^2,
\end{eqnarray*}
where $\left\| \Delta \t \right\|_2^2 = \t^T \LL \t$ and $\LL$ is the graph Laplacian matrix. Harmonic functions are proposed to recover a smooth graph signal which can be treated as an approximation of a piecewise-constant graph signal. When we obtain the solution, we assign each coefficient to its closest constant in $\x$.

\mypar{Experiments}
We test the four representations on two datasets, including the Minnesota road graph~\cite{MinnesotaGraph} and the U.S city graph~\cite{ChenSMK:14}.

For the Minnesota road graph, we simulate a piecewise-constant graph signal by randomly picking 5 nodes as community centers and assigning each other node a community label based on the geodesic distance. We assign a random integer to each community. We still use the simulated graph signal in Figure~\ref{fig:min_signal}(a).

The evaluation metric of the recovery error is the percentage of  mislabel, that is,
\begin{equation*}
 {\rm Error} = \frac{ \sum_{i=1}^N \Id( x_i \neq x_i') }{ N },
\end{equation*}
where $\x'$ is the recovered signal and $\x$ is the ground-truth signal. Figure~\ref{fig:recovery}(a) shows the recovery errors given by three algorithms. The x-axis is the ratio between the number of samples and the total number of nodes and the y-axis is the recovery error, where lower means better. 
Since harmonic functions and trend filtering are based on random sampling, the results are averaged over 50 runs. SC indicates spectral clustering, ST indicates spanning tree and 2M indicates 2-means. We see that the local-set-based recovery algorithms are better than harmonic functions and trend filtering, especially when the sample ratio is small.

For the U.S city graph, we use the same piecewise-constant graph signal~\eqref{eq:pc_us} as the ground truth. The evaluation metric of the recovery error is the percentage of  mislabel. Figure~\ref{fig:recovery}(b) shows the recovery errors given by three recovery strategies with two different sampling strategies. Similarly to the recovery of the Minnesota road graph, we see that the local-set-based recovery algorithms are better than harmonic functions and trend filtering, especially when the sample ratio is small.

\begin{figure}[htb]
  \begin{center}
    \begin{tabular}{cc}
\includegraphics[width=0.45\columnwidth]{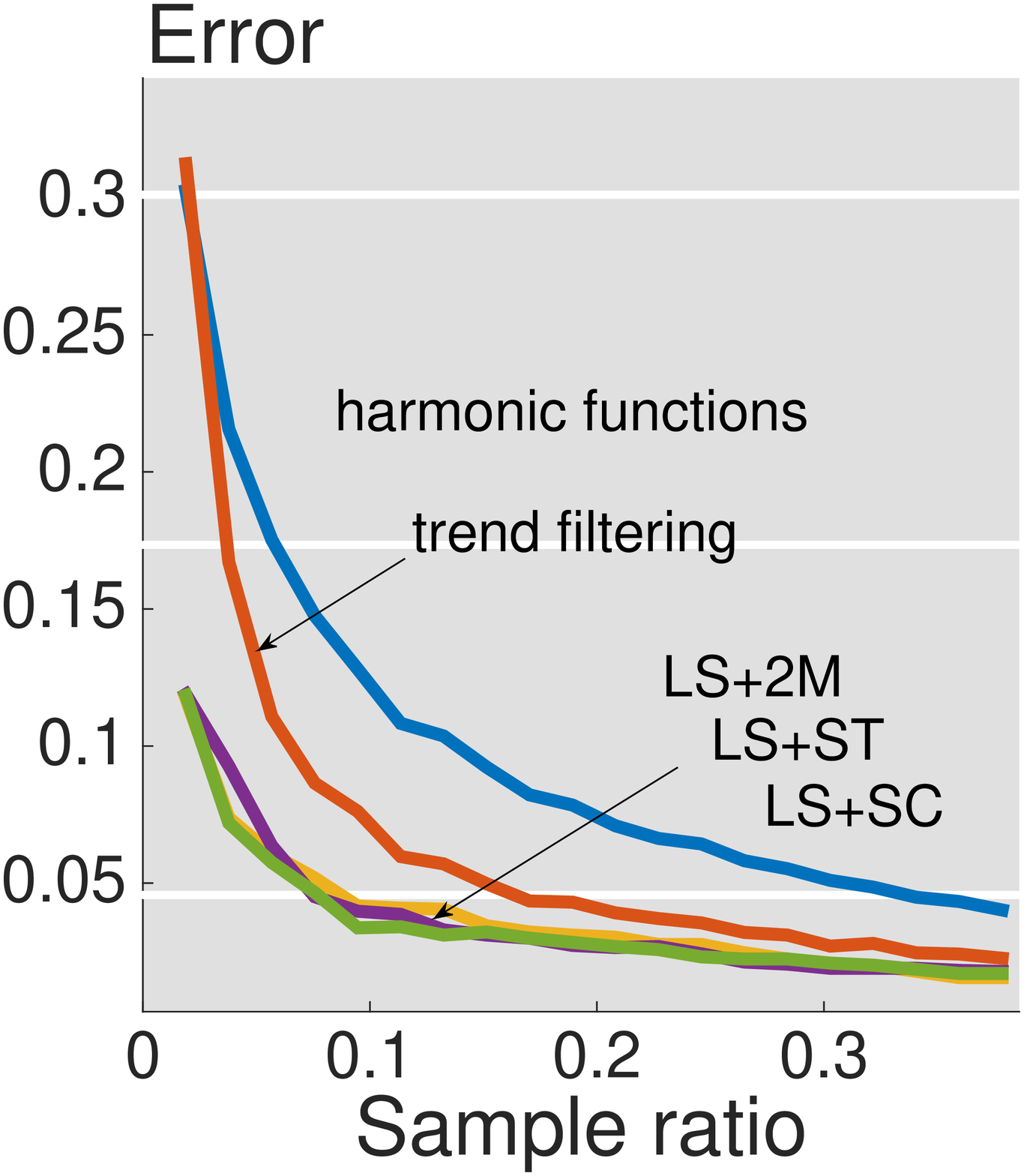}  & \includegraphics[width=0.45\columnwidth]{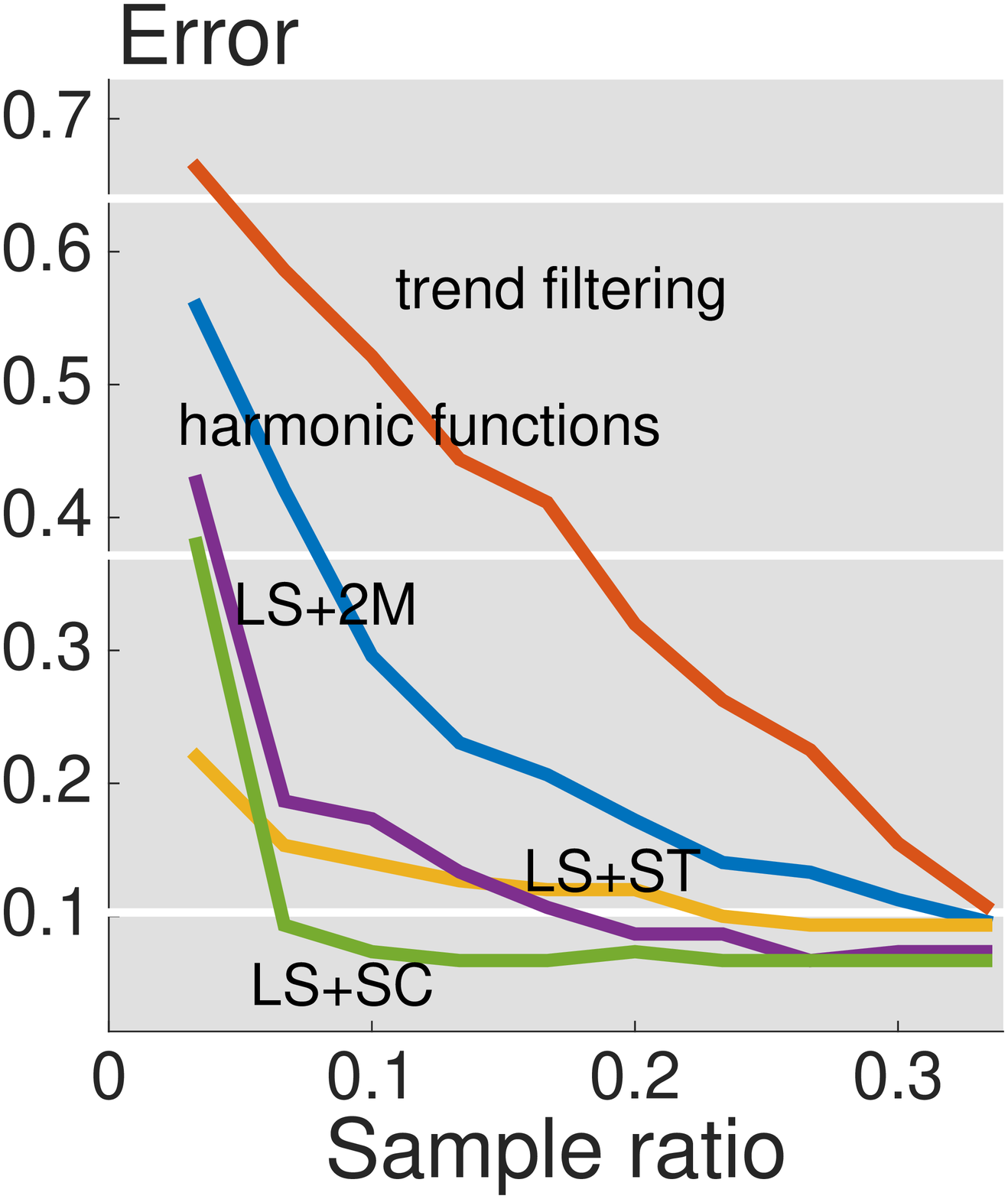}
\\
 {\small (a) Minnesota road graph.} & {\small (b) U.S. city graph} 
\end{tabular}
  \end{center}
   \caption{\label{fig:recovery} Comparison of recovery errors. LS+SC represents the local-set-based recovery algorithm with spectral clustering partition; LS+ST represents the local-set-based recovery algorithm with spanning tree partition; LS+2M represents the local-set-based recovery algorithm with 2-means partition. }
\end{figure}

To summarize the task of sampling and recovery, the proposed center-assign algorithm is simple and useful in the recovery. The experimentally designed sampling based on local sets tries to minimizes the upper bound in Theorem~\ref{thm:recovery_err} and make each local set have similar sizes. It provides a deterministic approach to choose sampled nodes; it works better than random sampling when the sample ratio is small and have a similar performance with random sampling asymptotically.

\subsubsection{Case Study: Epidemics Process}
Epidemics process has been modeled as the diffusion process of ideas/opinions/beliefs/innovations over a finite-sized, static,
connected social network~\cite{ZhangM:14}. In the terminology of epidemics, if the state of each node is either susceptible or infected, it is usually model by the susceptible-infected-susceptible (SIS) model. Nodes that are infected have a certain rate ($\gamma$) to recover and return to be susceptible; nodes that are susceptible can be contagious if infected by its neighboring infected nodes under certain rate ($\beta$).

Here we adopt the SIS model on network, which takes the network structure into account and help us estimate the macroscopic behavior of an epidemic outbreak~\cite{Newman:10}. In SIS model on network, $\beta$ is the infection rate quantify the probability per unit time that the infection will be transmitted from an infective individual to a susceptible one, $\gamma$ is the recovery (or healing) probability per unit time that an infective individual recovers and becomes susceptible again. To be more accurate, the infection rate studied here is a part of endogenous infection rate, which has the form of $\beta d$, where $d$ is the number of infected neighbors of the susceptible node~\cite{ZhangM:14, ZhangM:12}.  Since $\beta d$ dependents on the structure of the network, $\beta$ is referred as topology dependent infection rate, and since recovery is a spontaneous action and the recovery probability is identical for all the infective nodes, $\gamma$ is considered to be network topology independent~\cite{ZhangM:14}.

We consider a task to estimate the disease incidence, or the percentage of the infected nodes at each time. A simple method is that, in each time, we randomly sample some nodes, query their states, and calculate the percentage of the infected nodes. This method provides an unbiased estimator to estimate the disease incidence. However, this method has two shortcomings: first, it loses information on graphs and cannot tell which nodes are infected; second, since it is a random approach, it needs a huge number of samples to ensure a reliable estimation. 

We can model the states of nodes as a graph signal where $1$ represents infective and $0$ represents susceptible. When the topology dependent infection rate is high and the healing probability is low, the 
infection spreads locally; that is, nodes in a same community get infected in a same time and the corresponding graph signal is piecewise constant. We can use the sampling and recovery algorithm in Section~\ref{sec:SR} and then calculate the percentage of the infected nodes in the recovered graph signal. In this way, we can visualize the graph and tell which nodes may be infected because we recover the states of all the nodes; we also avoid the randomness effect because the algorithm is based on the experimentally designed sampling.

We simulate an epidemics process over the Minnesota road graph by using the SIS model. We set $\gamma$ be $0.1$,  and $\beta$ be $0.6$. In the first day, we randomly select three nodes to be infected and diffuses it for $49$ days. Figure~\ref{fig:min_estimation} shows the states of nodes in the $10$th day and the $20$th day. We see that three small communities are infected in the $10$th day; these communities are growing bigger in the $20$th day. Since the healing probability is nonzero, a few susceptible nodes still exist within the communities.
 
\begin{figure}[htb]
  \begin{center}
    \begin{tabular}{cc}
 \includegraphics[width=0.45\columnwidth]{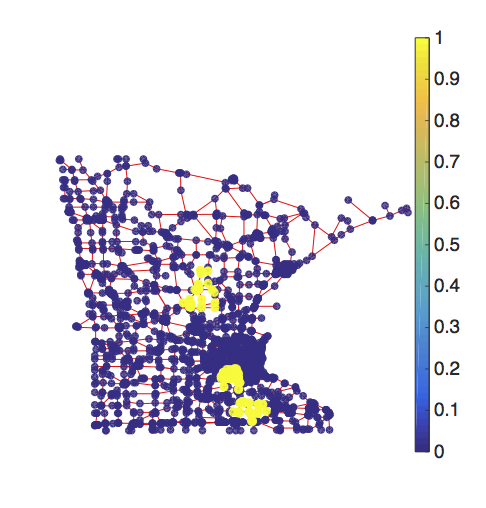}  & \includegraphics[width=0.45\columnwidth]{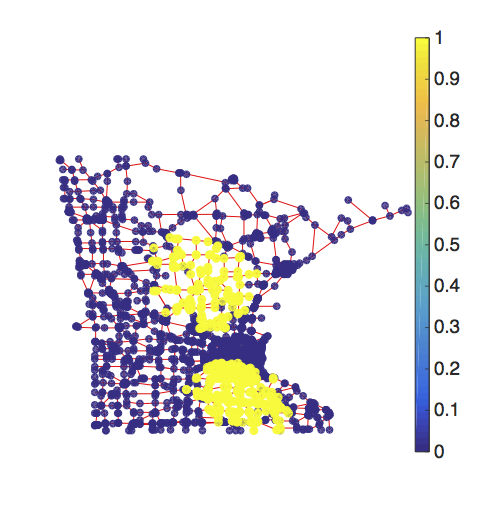}
\\
 {\small (a) $10$th day.} & {\small (b) $20$th day.} 
\end{tabular}
  \end{center}
   \caption{\label{fig:min_estimation} Epidemics process over the Minnesota road graph. Yellow indicates infection and blue indicate susceptible.}
\end{figure}

We compare the results of two algorithms: one is based on random sampling following with calculating the percentage of infection within the sampled nodes; the other one is based on the local-set-based recovery algorithm following with calculating the percentage of infection within the recovered graph signal. The evaluation metric is the frequency that the result of the local-set-based recovery algorithm is closer to the groundtruth, that is,
\begin{equation*}
{\rm Success~rate}  \ = \ \frac{1}{M}  \sum_{i=1}^M  \Id ( | \hat{x}^{(2)} - x_0 | <  | \hat{x}^{(1)}_{i} - x_0 | ),
\end{equation*} 
where $ x_0$ is the ground truth of the percentage of infection,
$ \hat{x}^{(1)}_{i}$ is the estimation of the random algorithm in the $i$th trials, $\hat{x}^{(2)}$ is the estimation of the local-set-based recovery algorithm, and $M$ is the total number of random trials; we choose $M= 1000$ here.  The success rate measures the frequency when the local-set-based recovery algorithm has a better performance. When the success rate is bigger than $0.5$, the local-set-based recovery algorithm is better;  When the success rate is smaller than $0.5$, the random algorithm is better. Figure~\ref{fig:Success_rate} shows the success rates given by the local-set-based recovery with three different graph partition algorithms. In each figure, the x-axis is the day ($50$ days in total); the y-axis is the success rate; the darker region means that local-set-based recovery algorithm fails and the lighter  region means that local-set-based recovery algorithm successes; and the number shows the percentage of success or fail within 50 days. We see that given 100 samples, the local-set-based recovery algorithms are slightly worse than the random algorithm; given 1000 samples, the local-set-based recovery algorithms are slightly better than the random algorithm.

\begin{figure}[htb]
  \begin{center}
    \begin{tabular}{cc}
 \includegraphics[width=0.4\columnwidth]{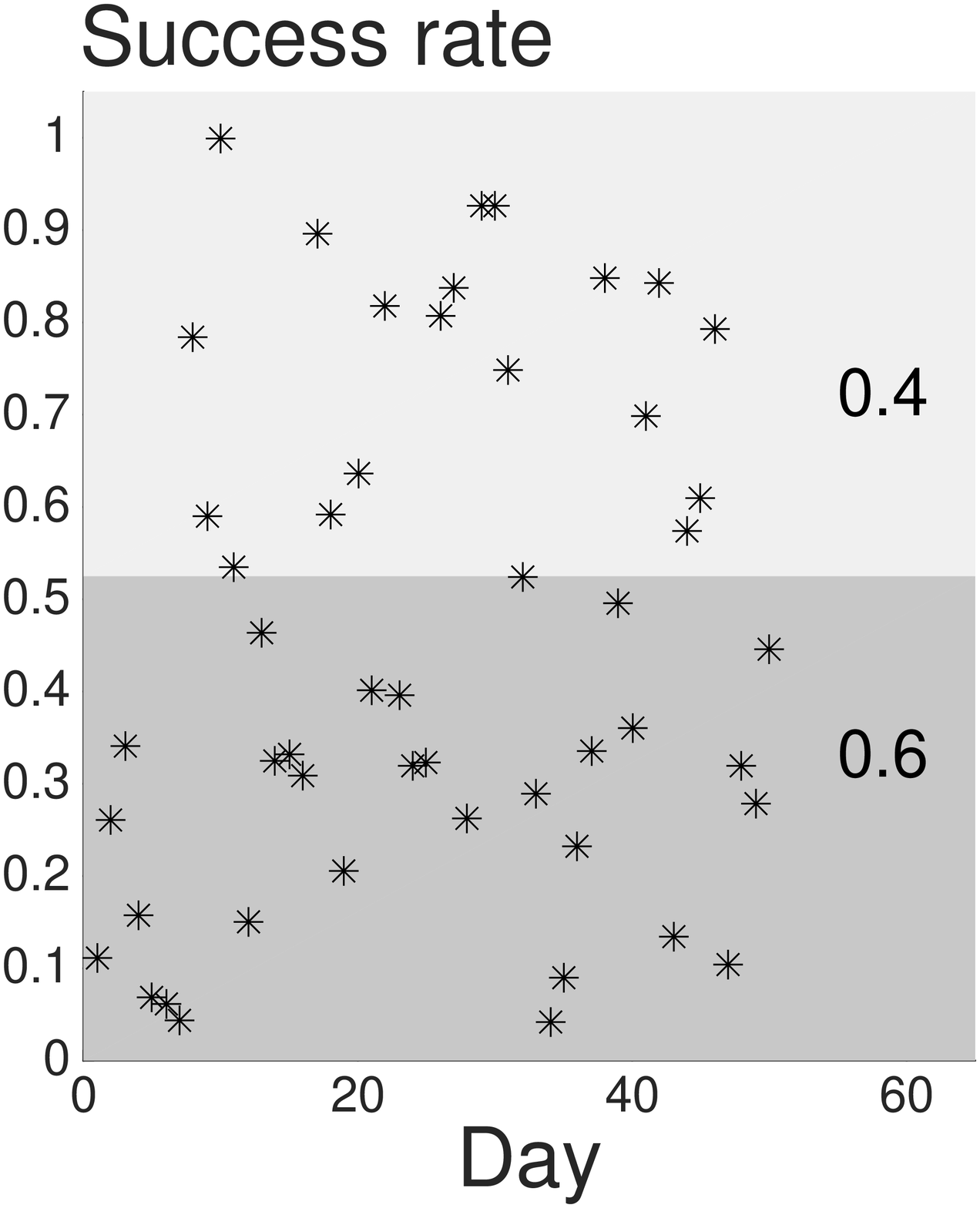}  & \includegraphics[width=0.4\columnwidth]{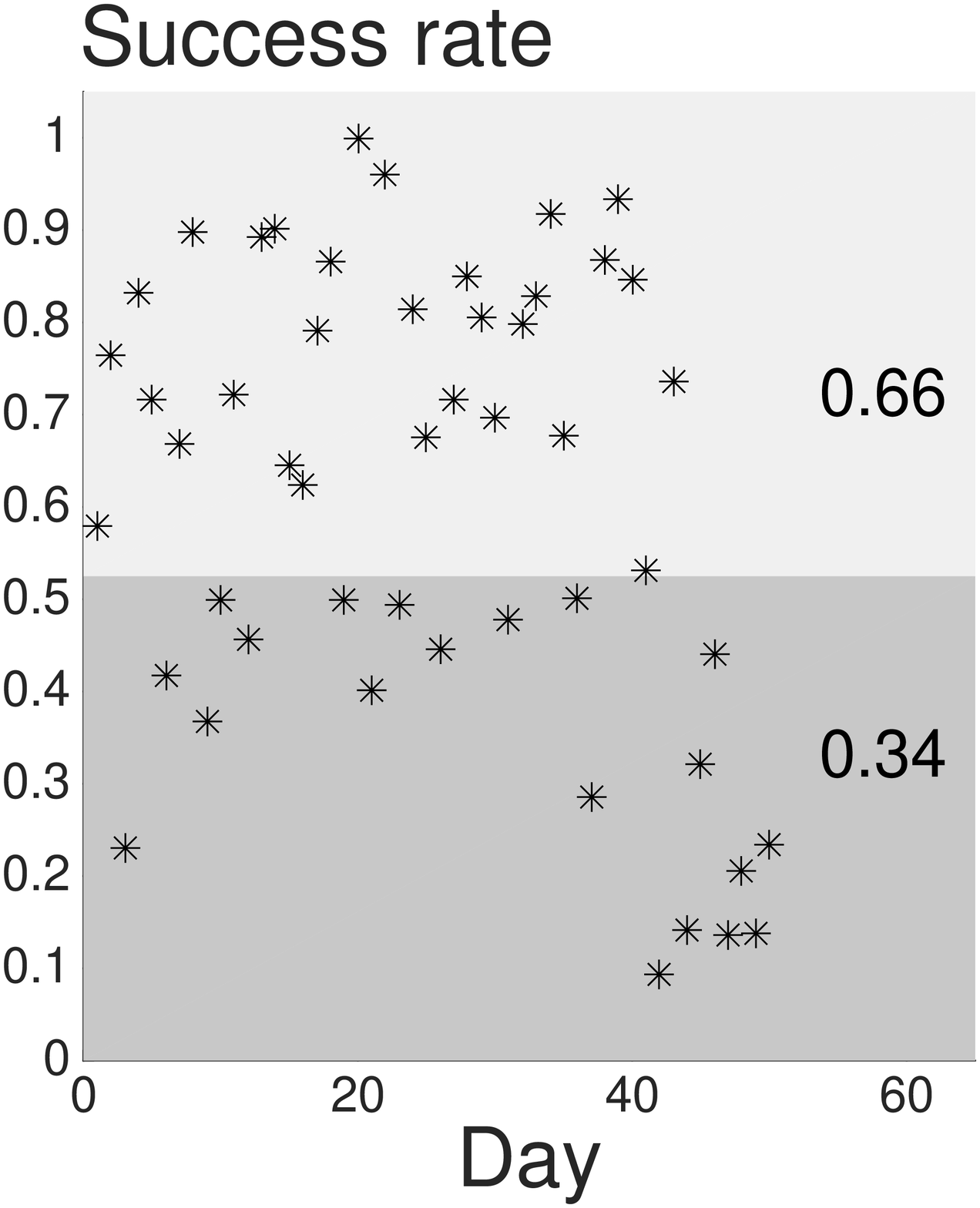}
\\
 {\small (a) Spectral clustering} & {\small (b) Spectral clustering} 
 \\
 { with $100$ samples.} & { with $1000$ samples.} 
  \\
 \includegraphics[width=0.4\columnwidth]{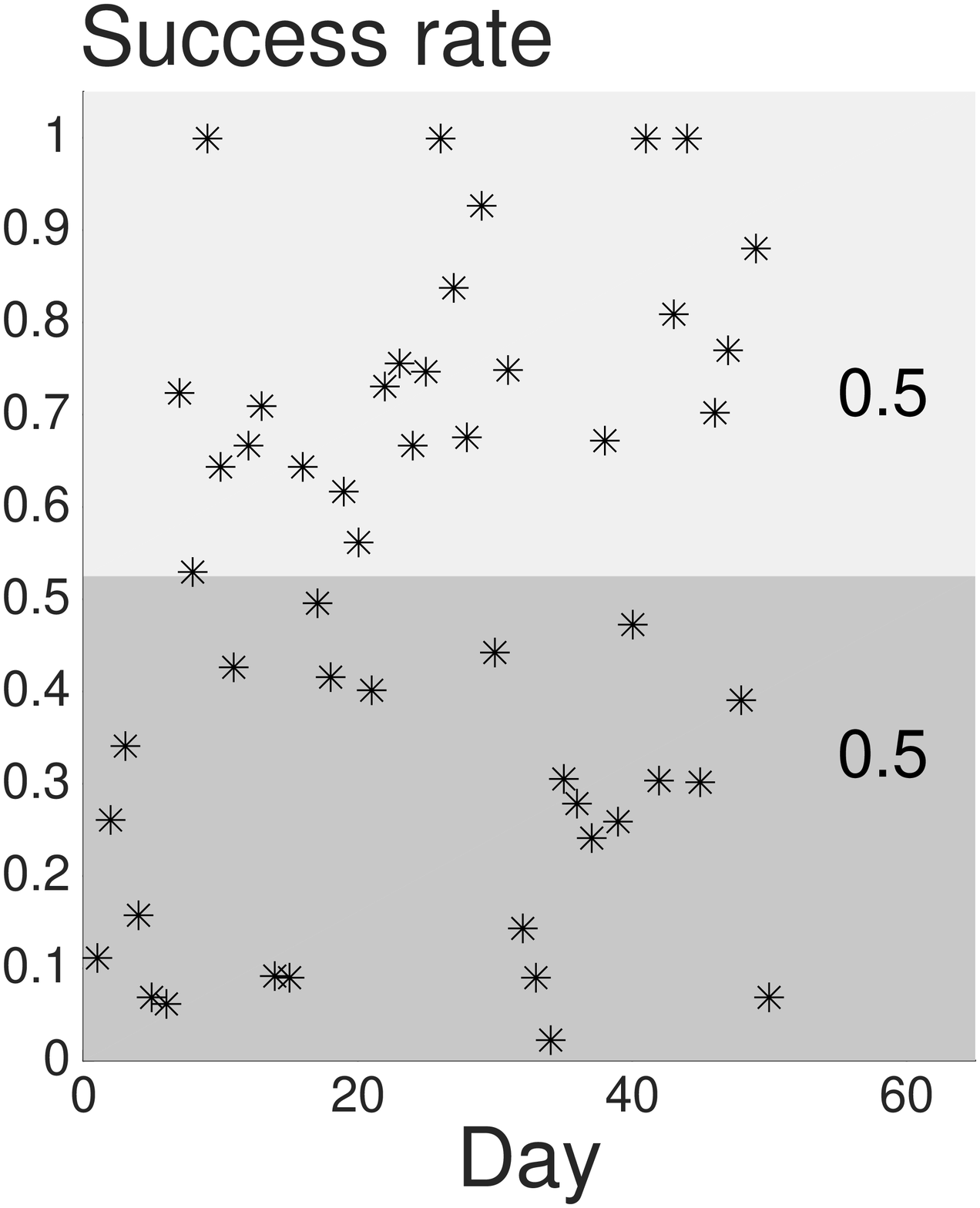}  & \includegraphics[width=0.4\columnwidth]{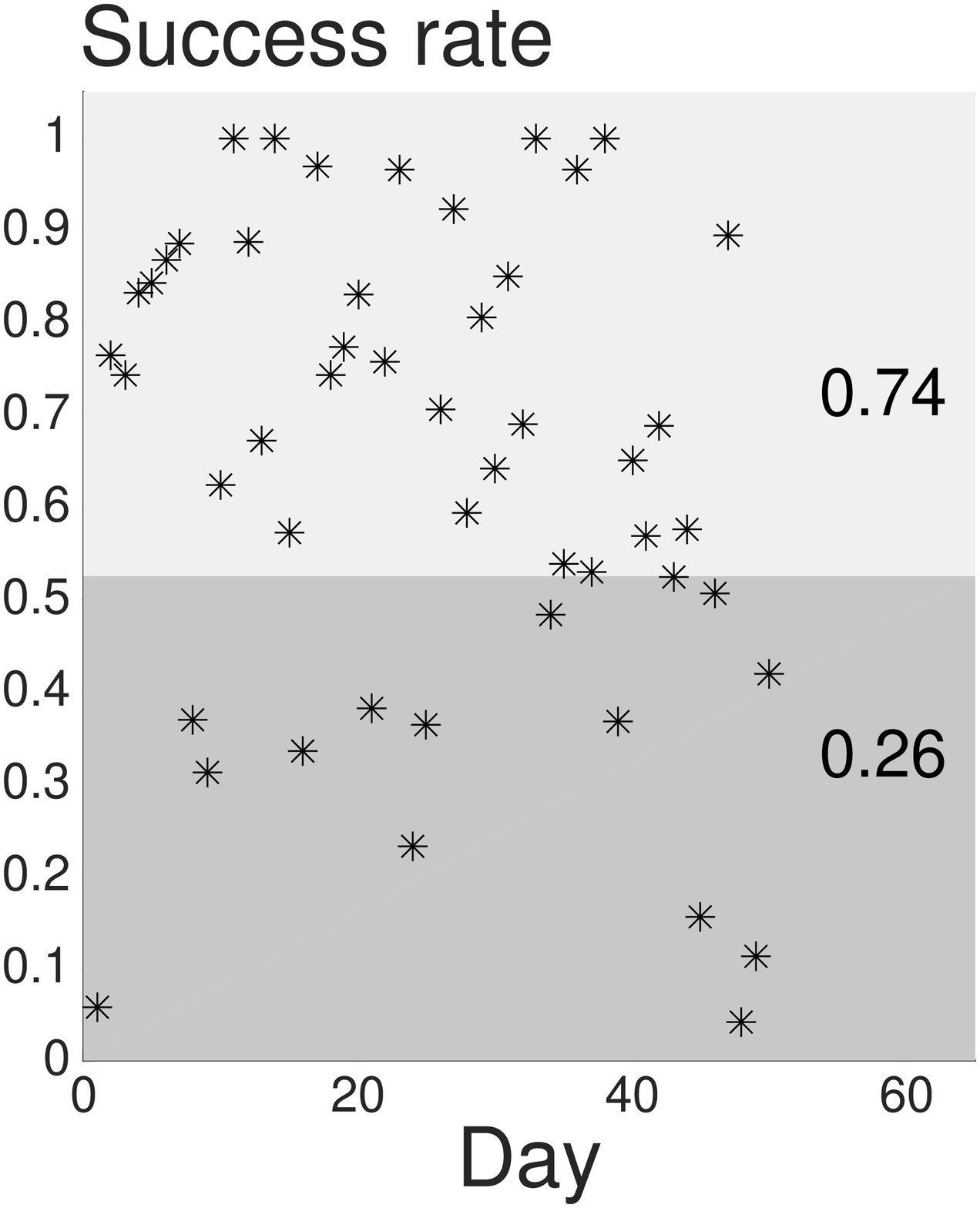}
\\
 {\small (c) Spanning tree} & {\small (d) Spanning tree} 
 \\
 { with $100$ samples.} & { with $1000$ samples.} 
  \\
 \includegraphics[width=0.4\columnwidth]{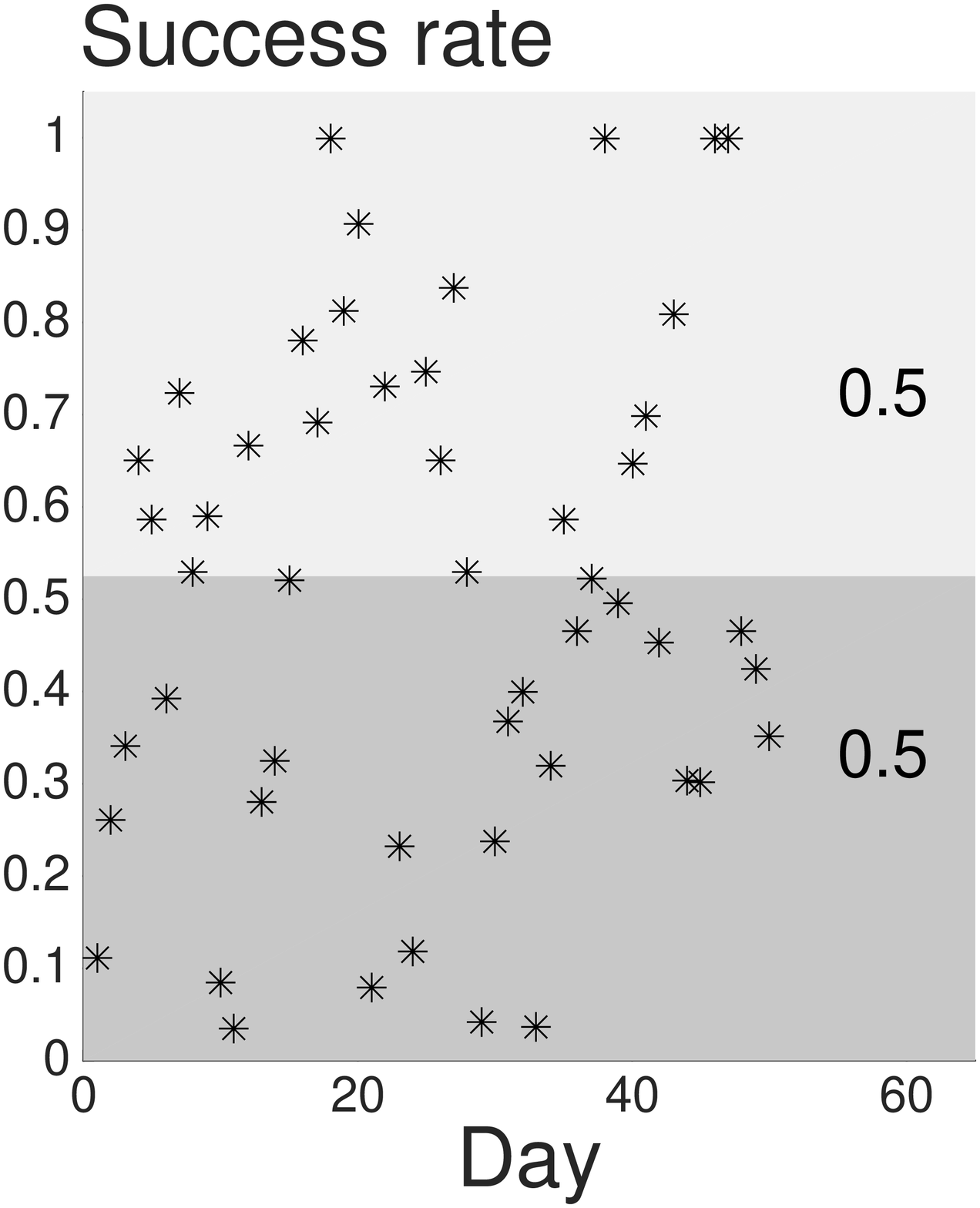}  & \includegraphics[width=0.4\columnwidth]{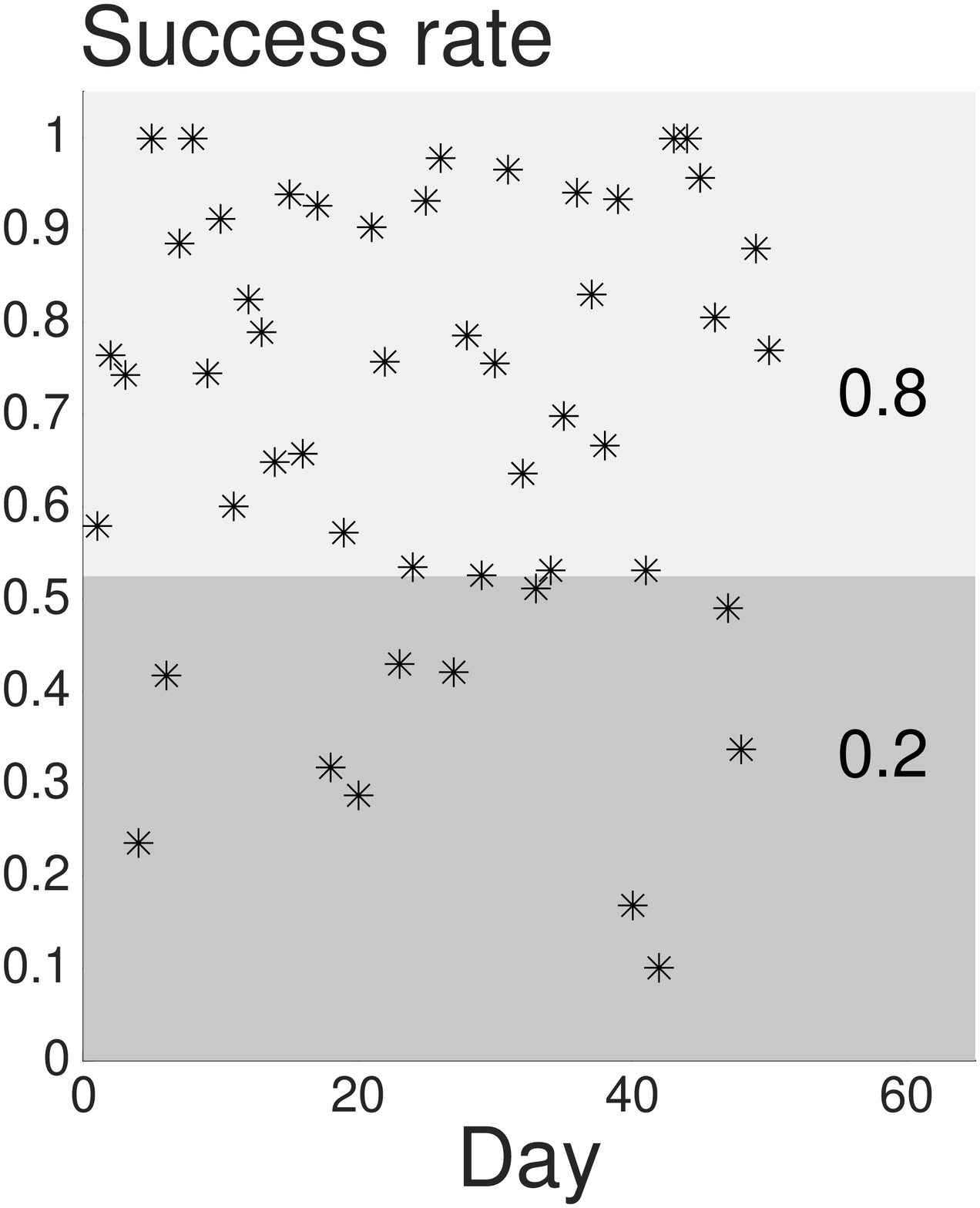}
\\
 {\small (e) 2-means} & {\small (f)  2-means} 
 \\
 { with $100$ samples.} & { with $1000$ samples.} 
  \\
\end{tabular}
  \end{center}
   \caption{\label{fig:Success_rate} Success rate of estimating the disease incidence.  }
\end{figure}

In Figure~\ref{fig:recovery_20}, we show the recovered states by the local-set-based recovery algorithm with 2-means partition on the $20$th day. When having a few samples, the local-set-based recovery algorithms can recover the states in general, but cannot zoom into details and provide accurate estimations; when taking more samples, the local-set-based recovery algorithms recover the states better and provide better estimations. We see that the local-set-based recovery algorithm with 2-means partition provides both good estimation and good visualization.

\begin{figure}[htb]
  \begin{center}
    \begin{tabular}{cc}
 \includegraphics[width=0.45\columnwidth]{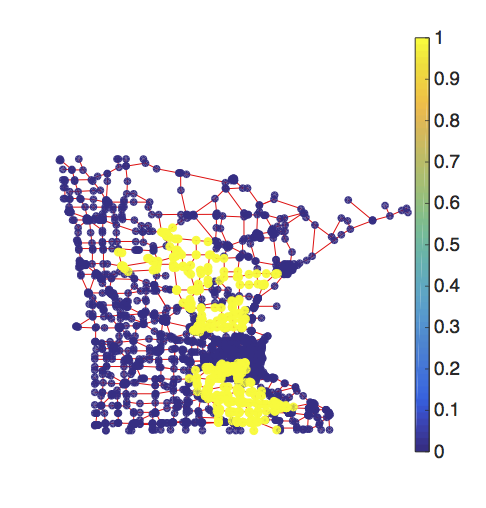}  & \includegraphics[width=0.45\columnwidth]{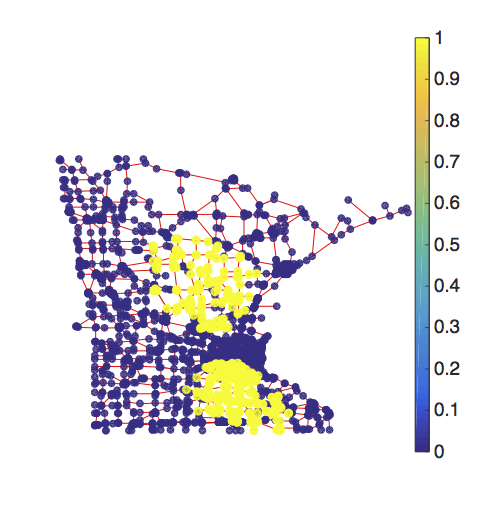}
\\
 {\small (e) 2-means} & {\small (f)  2-means} 
 \\
 { with $100$ samples.} & { wiith $1000$ samples.} 
  \\
\end{tabular}
  \end{center}
   \caption{\label{fig:recovery_20} Recovery of the node state on the $20$th day.  }
\end{figure}

\section{Representations of Piecewise-smooth Graph Signals}
\label{sec:R_PS}
To be able to deal with as wide a class of real-world graphs signals as possible, we combine smooth and piecewise-constant graph signals into piecewise-smooth graph signals.

\subsection{Graph Signal Models}
Based on smooth graph signal models, we have two types of piecewise-smooth signal models, including the piecewise-polynomial class and the piecewise-bandlimited class.

\begin{defn}
A graph signal $\x$ is piecewise-polynomial with $C$ pieces and degree $K$ when
\begin{equation*}
	\x \ = \  \sum_{c=1}^C  \x^{(c)}  {\bf 1}_{ S_c},
\end{equation*}
where $\x^{(c)}$ is a $k$th order polynomial signal on the subgraph $G_{S_c}$ with $x^{(c)}_i  \ = \   a_{c} + \sum_{j \in S_c} \sum_{k=1}^K a_{k, j, c} d^k(v_i, v_j). $
Denote this class by $\PPL(C, K)$. 
\end{defn}
\noindent $\PPL(1, K)$ is the polynomial class with degree $K$, $\PL(K)$ from Definition~\ref{df:polynomial_sig}, and $\PPL(C, 0)$ is the piecewise-constant class with $C$ pieces, $\PC(C)$ from Definition~\ref{df:pc_gen}. The degrees of freedom for a local set $S_c$ at the polynomial degree $k$ is the number of origins, that is,
$
\left\| \begin{bmatrix}
a_{k, 1, c} & a_{k, 2, c} &  \ldots & a_{k, |S_c|, c}
\end{bmatrix} \right\|_0.$

\begin{defn}
A graph signal $\x$ is piecewise-bandlimited with $C$ pieces and bandwidth $K$ when
\begin{equation*}
	\x \ = \  \sum_{c=1}^C  \x^{(c)}  {\bf 1}_{ S_c},
\end{equation*}
where $\x^{(c)}$ is a bandlimited signal on the subgraph $G_{S_c}$ with $x^{(c)}_i = \sum_{k=0}^K a_{k, c} \Vm^{(c)}_{i,k}$, and $\Vm^{(c)}$ is a graph Fourier basis of $G_{S_c}$.
Denote this class by $\PBL_{\Vm}(C, K)$. 
\end{defn}
\noindent We use zero padding to ensure $\Vm^{(c)} \in \R^{N \times N}$ for each~$G_{S_c}$. Still, $\Vm^{(c)}$ can be the eigenvector matrix of  the adjacency matrix, graph Laplacian matrix or the transition matrix.

\subsection{Graph Dictionary}
\label{ssec:Graph_D}
The representations of piecewise-smooth graph signals is based on the local-set piecewise-constant dictionary. To represent piecewise-smooth graph signals, we use multiple atoms for each local set. We take the piecewise-polynomial signals as an example. For each local set, 
\begin{equation*}
	\D_{S_{i, j}}   \ = \ \begin{bmatrix}
	{\bf 1}  &  \D^{(1)}_{S_{i, j}}  &  \D^{(2)}_{S_{i, j}}   &  \ldots &  \D^{(K)}_{S_{i, j}} 
	\end{bmatrix},
\end{equation*}
where  $(  \D_{S_{i, j}}^{(k)} )_{m,n}  = d^k(v_m, v_n)$, when $v_m, v_n \in S_{i,j}$; and 0, otherwise. The number of atoms in $\D_{S_{i, j}}^{(k)}$ is $1+K |S_{i,j}|$. We collect the sub-dictionaries for all the multiresolution local sets to form the~\emph{local-set-based piecewise-smooth dictionary}, that is,
$ \D_{\rm LSPS}~=~\{ \D_{S_{i, j}}  \}_{i=0, j = 1}^{T,  2^i} $. The number of atoms in $\D_{\rm LSPS}$ is $O(K N T)$, where $K$ is the maximum degree of polynomial, $N$ is the size of the graph and $T$ is the maximum level of the decomposition. When we use  even partitioning, the total number of atoms is $O(K N \log N)$.

Similarly, to model piecewise-bandlimited signals, we replace $\D_{S_{i, j}}$ by the graph Fourier basis of each subgraph $G_{S_{i,j}}$. The total number of atoms of the corresponding $\D_{\rm LSPS}$ is then $O(N T)$. For piecewise-smooth signals, we cannot use the sparse coding to do exact approximation. To minimize the approximation error, both the sizes and the shapes of the local sets matter for piecewise-smooth graph signals.

\subsubsection{Properties}
\begin{myThm}
\label{thm:PPL}
For all $\x \in \PBL_{\Vm_L}(C,K)$, where $\Vm_L$ is the graph Fourier basis of the graph Laplacian matrix $\LL$,   we have $\left\| \a^* \right\|_0 \leq  2 K T\left\|  \Delta \x_{\PC} \right\|_0$, where $T$ is the maximum level of the decomposition, $\x_{\PC}$ is a piecewise-constant signal that corresponds the same local sets with $\x$ and
\begin{eqnarray*}
   \a^* &  = &   \arg \min_{\a}   \left\| \a \right\|_0,
   \\
   && {\rm subject~to~}  \left\| \x - \D_{\rm LSPS} \a \right\|_2^2 \leq \epsilon_{\rm par}  \left\| \x \right\|_2^2,
\end{eqnarray*}
where $\epsilon_{\rm par} $ is a constant determined by graph partitioning.
\end{myThm}
\begin{proof}
The main idea is that we approximate a bandlimited signal in the original graph by using bandlimited signals in subgraphs. Based on the eigenvectors of graph Laplacian matrix, we define the bandlimited space, where each signal can be represented as $\x = \Vm_{(K)} \a,$ where $\Vm_{(K)}$ is the submatrix of $\Vm$ containing the first $K$ columns in $\Vm$. We can show that this bandlimited space is a subspace of the total pairwise smooth space $\{\x: \x^T \LL \x \leq  \lambda_K  \x^T \x \}$. 

\begin{eqnarray*}
&& \x^T \LL \x 
\  = \  \sum_{i,j \in \E} \W_{i,j} (x_i - x_j)^2
\\
& = & \sum_{S_c} \sum_{i,j \in \E_{S_c}} \W_{i,j} (x_i - x_j)^2  + \sum_{i,j \in (\E / \cup_{c} \E_{S_c}) } \W_{i,j} (x_i - x_j)^2
\\
& = &  \sum_{S_c} \x_{S_c}^T \LL_{S_c} \x_{S_c} + \x^T \LL_{\rm cut} \x \ \leq \    \lambda_K  \x^T \x,
\end{eqnarray*}
where $\LL_{S_c}$ is the graph Laplacian matrix of the subgraph $G_{S_c}$ and  $\LL_{S_c}$ stores the residual edges, which are cut in the graph partition algorithm.

Thus, $\{\x: \x^T \LL \x \leq  \lambda_K  \x^T \x \}$ is a subset of $\bigcup_{S_c}  \{\x_{S_c}: \x_{S_c}^T \LL_{S_c} \x_{S_c} \leq  \lambda_K  \x^T \x - \x^T \LL_{\rm cut} \x  \}$; that is, any  total pairwise smooth graph signal in the whole graph can be precisely represented by total pairwise smooth graph signals in the subgraphs.

In each local set, when we use the bandlimited space 
$\{ \x: \x = {\Vm_{S_c}}_{(K)} \a \}$ to approximate the space $\{\x_{S_c}: \x_{S_c}^T \LL_{S_c} \x_{S_c} \leq   c \x_{S_c}^T \x_{S_c} \}$, the maximum error we suffer from is $c \x_{S_c}^T \x_{S_c} / \lambda^{(S_c)}_{K+1}$, which is solved by the following optimization problem,
\begin{eqnarray*}
&&	\max_{\x}  \left\|  \x - \Vm_{S_c} \Vm_{S_c}^T \x \right\|_2^2
	\\
&&	{\rm subject~to:}~   \x^T \LL_{S_c} \x \leq   c  \x^T \x.
\end{eqnarray*}
In other words, in each local set, the maximum error to represent $ \{\x_{S_c}: \x_{S_c}^T \LL_{S_c} \x_{S_c} \leq  \lambda_K  \x^T \x - \x^T \LL_{\rm cut} \x   \}$ is 
$(\lambda_K  \x^T \x - \x^T \LL_{\rm cut} \x)/\lambda^{(S_c)}_{K+1}$. Since all the local sets share the variation budget of $\lambda_K \x^T \x$ together, the maximum error we suffer from is $\epsilon_{\rm par} \left\| \x \right\|_2^2 = (\lambda_K  \x^T \x - \x^T \LL_{\rm cut} \x)/ \min_{S_c} \lambda^{(S_c)}_{K+1}$.

In Theorem~\ref{thm:sparse}, we have shown that we need at most $2 T \left\| \Delta \x_{\PC} \right\|_0$ local sets to represent the piecewise-constant template of $\x$. Since we use at most $K$ eigenvectors in each local set, we obtain the results in Theorem~\ref{thm:PPL}.
\end{proof}

We further use the local-set piecewise-smooth dictionary to detect piecewise-smooth signals from random noises.
\begin{myThm}
\label{thm:PPL_Detect}
We consider statistically testing the null and alternative hypothesis,
\begin{eqnarray*}
H_0 &:& \y \sim \N(0, \sigma^2 \Id), 
\\
H_1 &:&  \y \sim \N(\x, \sigma^2 \Id),~\x \in \PBL_{\Vm_L}(C,K).
\end{eqnarray*}
We solve the following optimization problem
\begin{eqnarray*}
   \a_{\y}^* &  = &   \arg \min_{\a}  \left\| \y - \D_{\rm LSPS} \a \right\|_2^2   
   \\
   && {\rm subject~to~}  \left\| \a \right\|_0 \leq 2KT  \left\| \Delta \x_{\rm PC} \right\|_0,
\end{eqnarray*}
by using the matching pursuit algorithm. We reject the null if $\left\| \a_{\y}^* \right\|_{\infty} > \tau$. Set $\tau = \sigma \sqrt{2 \log(KNT/\delta)}$. If
$$
\frac{\left\| \x \right\|_2}{\sigma} \geq C \sqrt{2KT \left\| \Delta \x_{\rm PC} \right\|_0 }  \sqrt{8 \log(\frac{KNT}{\delta})},
$$
where $C = \max_{ |\Omega| \leq 2KT \left\| \Delta \x_{\rm PC} \right\|_0  } \left\| (\D_{\rm LSPS})_{\Omega} \right\|_2 /(1- \sqrt{\epsilon_{\rm par}}) $. Then under $H_0$, $\mathbb{P}({\rm Reject}) \leq \delta$, and under $H_1$, $\mathbb{P} ({\rm Reject}) \leq \delta$.
\end{myThm}

\begin{proof}
Under the null, we have $\left\| \a_{\e}^* \right\|_{\infty} < \sigma \sqrt{2 \log( KNT /\delta)}$ with probability at least $1-\delta$. We set $\tau = \sigma \sqrt{2 \log(KNT/\delta)}$ to control the type 1 error.

For the alternative, $\left\| \a_{\y}^* \right\|_{\infty} \geq \left\| \a_{\x}^* \right\|_{\infty} - \sigma \sqrt{2 \log(KNT/\delta)}$ with probability $1-\delta$. In Theorem~\ref{thm:PPL}, we see that it is feasible to have $\left\| \x - \D_{\rm LSPS} \a_\x \right\|_2^2 \leq \epsilon_{\rm par} \left\| \x \right\|_2^2$ with $\left\| \a_\x \right\|_0 = 2KT  \left\| \Delta \x_{\rm PC} \right\|_0$. We then have
\begin{eqnarray*}
 \left\| \a_{\x}^* \right\|_{\infty} \geq  \frac{ \left\| \a_{\x}^* \right\|_{2}  }{ \sqrt{ \left\| \a_{\x}^* \right\|_0 } } 
\geq 
  \frac{ 1 - \sqrt{ \epsilon_{\rm par}}}{ \left\| (\D_{\rm LSPS})_{\Omega} \right\|_2 \sqrt{ 2KT  \left\| \Delta \x_{\rm PC} \right\|_0  } }   \left\| \x \right\|_2,
\end{eqnarray*}
where $\Omega$ is the support of zero elements in $\a_{\x}^*$. Then, we force the lower bound of $\left\| \a_{\y}^* \right\|_{\infty} $ to control the type 2 error.
\end{proof}

\subsection{Graph Signal Processing Tasks}

\subsubsection{Approximation}
Good approximation for piecewise-smooth graph signals are potentially useful in many applications, such as visualization, denoising, active sampling~\cite{CastroWN:05,ChenVSK:15a} and semi-supervised learning~\cite{Zhu:05}. Similarly to Section~\ref{sec:pc_app}, we compare the windowed graph Fourier transform~\cite{ShumanRV:15} with the local-set-based dictionaries. We use the balance cut of the spanning tree to obtain the local sets~\cite{SharpnackKS:13}. The approximation is implemented by solving the sparse coding problem~\eqref{eq:sparse_coding_pc}.

\mypar{Experiments}
We still test the representations on two datasets, the Minnesota road graph~\cite{MinnesotaGraph} and the U.S city graph~\cite{ChenSMK:14}. On the Minnesota road graph, we simulate $100$ piecewise-constant graph signals as follows: we random choose three nodes as cluster centers and assign all other nodes to their nearest cluster centers based on the geodesic distance. We assign a random integer to each cluster. We further obtain $100$ piecewise-polynomial graph signals by element-wise multiplying a polynomial function, $- d^2(v_0, v) + 12 d(v_0, v)$, where $v_0$ is a reference node that assigns randomly. As an example, see Figure~\ref{fig:Signal}(a). 

On the U.S city graph, we use the real temperature measurements. The graph includes 150 weather stations and each weather station has $365$ days of recordings (one recording per day), for a total of 365 graph signals. As an example, see Figure~\ref{fig:Signal}(b).

\begin{figure}[htb]
  \begin{center}
    \begin{tabular}{cc}
 \includegraphics[width=0.45\columnwidth]{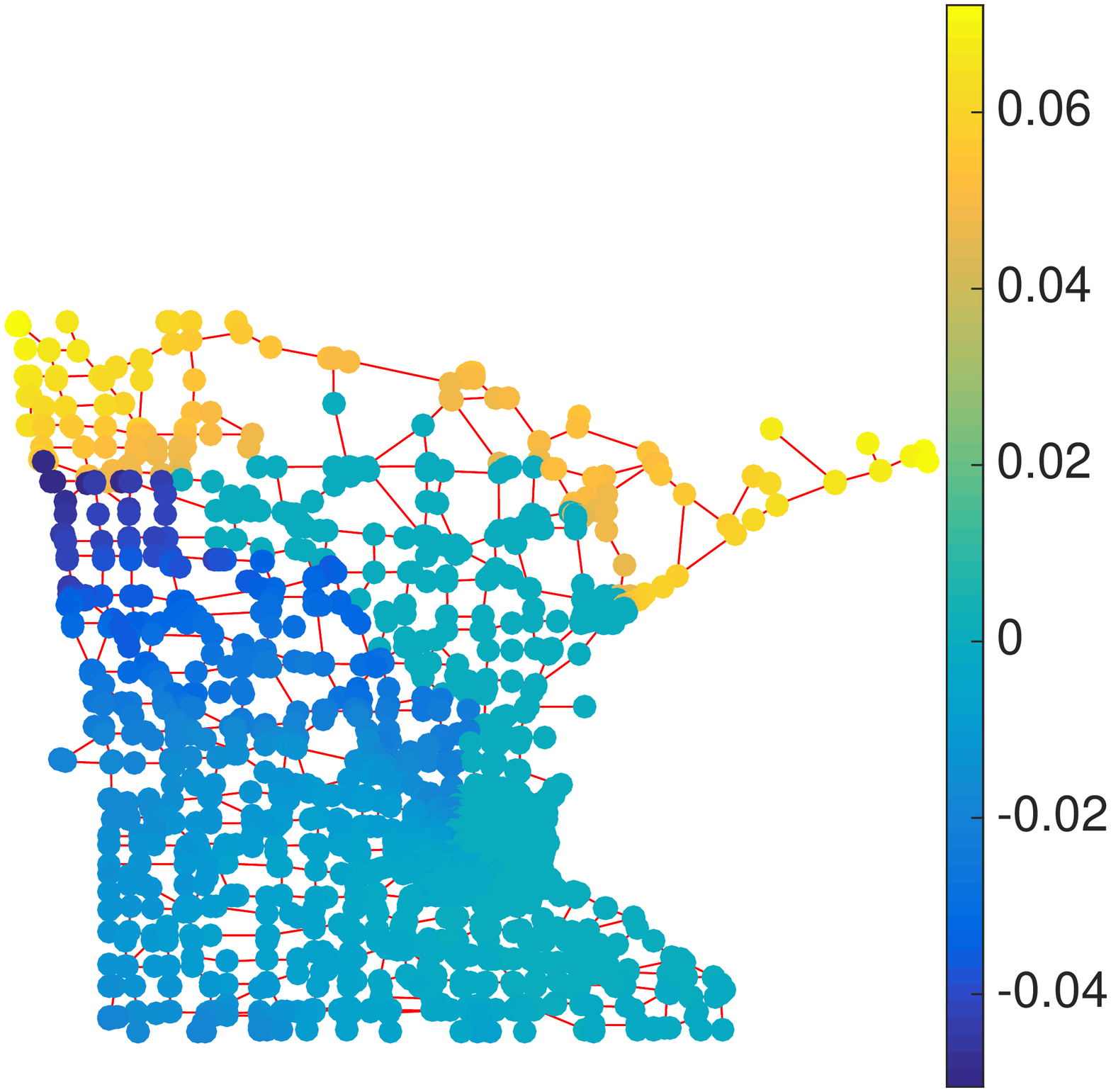} &
 \includegraphics[width=0.4\columnwidth]{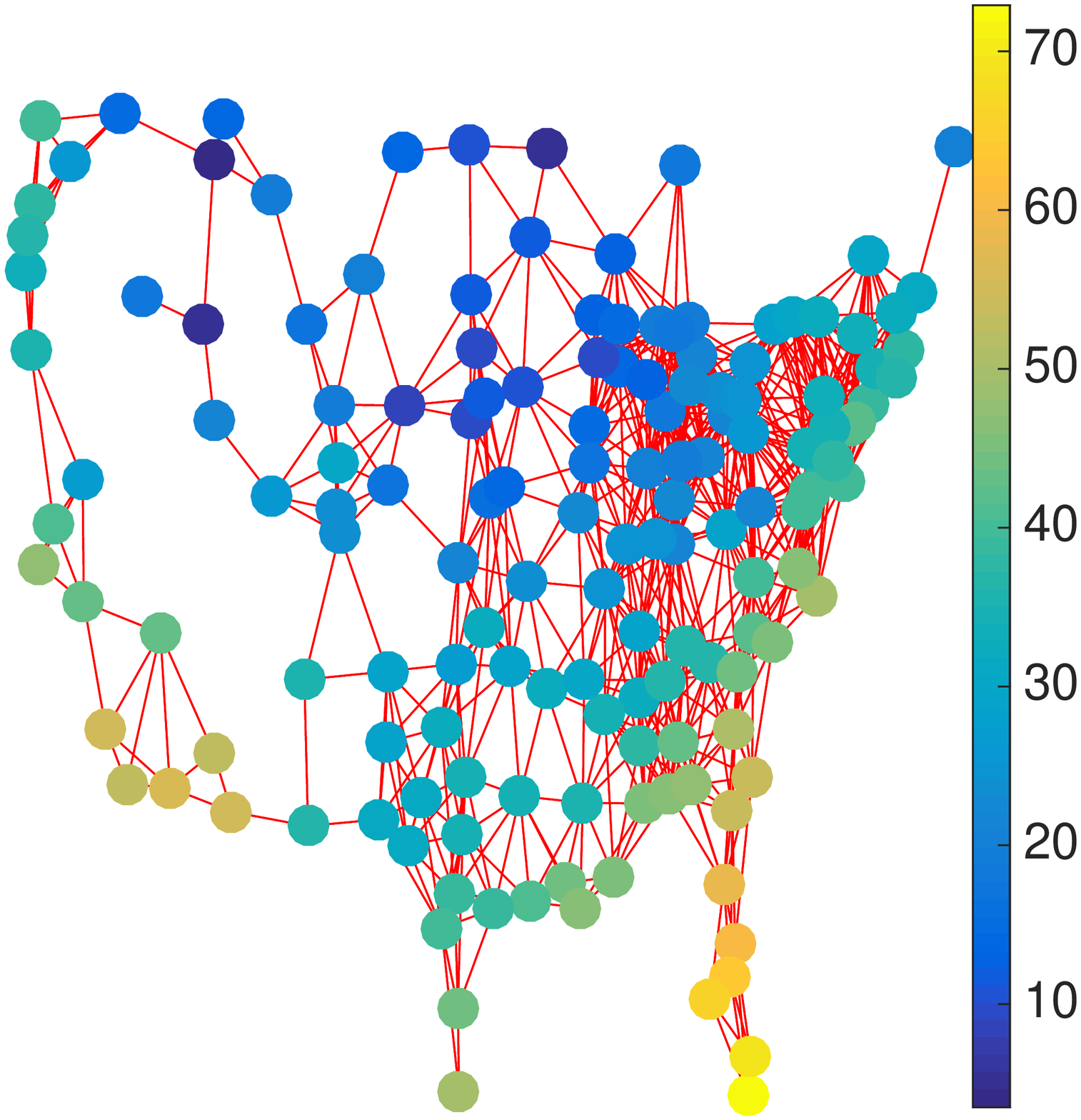}
\\
 {\small (a) Minnesota.} & {\small (b) U.S weather station.} 
\end{tabular}
  \end{center}
   \caption{\label{fig:Signal} Graph signal.}
\end{figure}

The approximation error is measured by the normalized mean square error. Figure~\ref{fig:Approximation} shows the averaged approximation errors.  LSPC denotes local-set-based piecewise-constant dictionary and LSPS denotes local-set-based piecewise-smooth dictionary.  For the windowed graph Fourier transform, we use $15$ filters; for LSPS, three piecewise-smooth models provide tight performances; here we show the results of the piecewise-polynomial smooth model with degree $K = 2$. We see that the local-set-based dictionaries perform better than the windowed graph Fourier transform; local-set-based piecewise-smooth dictionary is slightly better than local-set-based piecewise-constant dictionary; even though the windowed graph Fourier transform is solid in theory, provides highly redundant representations and is useful for visualization, it does not well approximate complex  graph signals.
\begin{figure}[htb]
  \begin{center}
    \begin{tabular}{cc}
 \includegraphics[width=0.4\columnwidth]{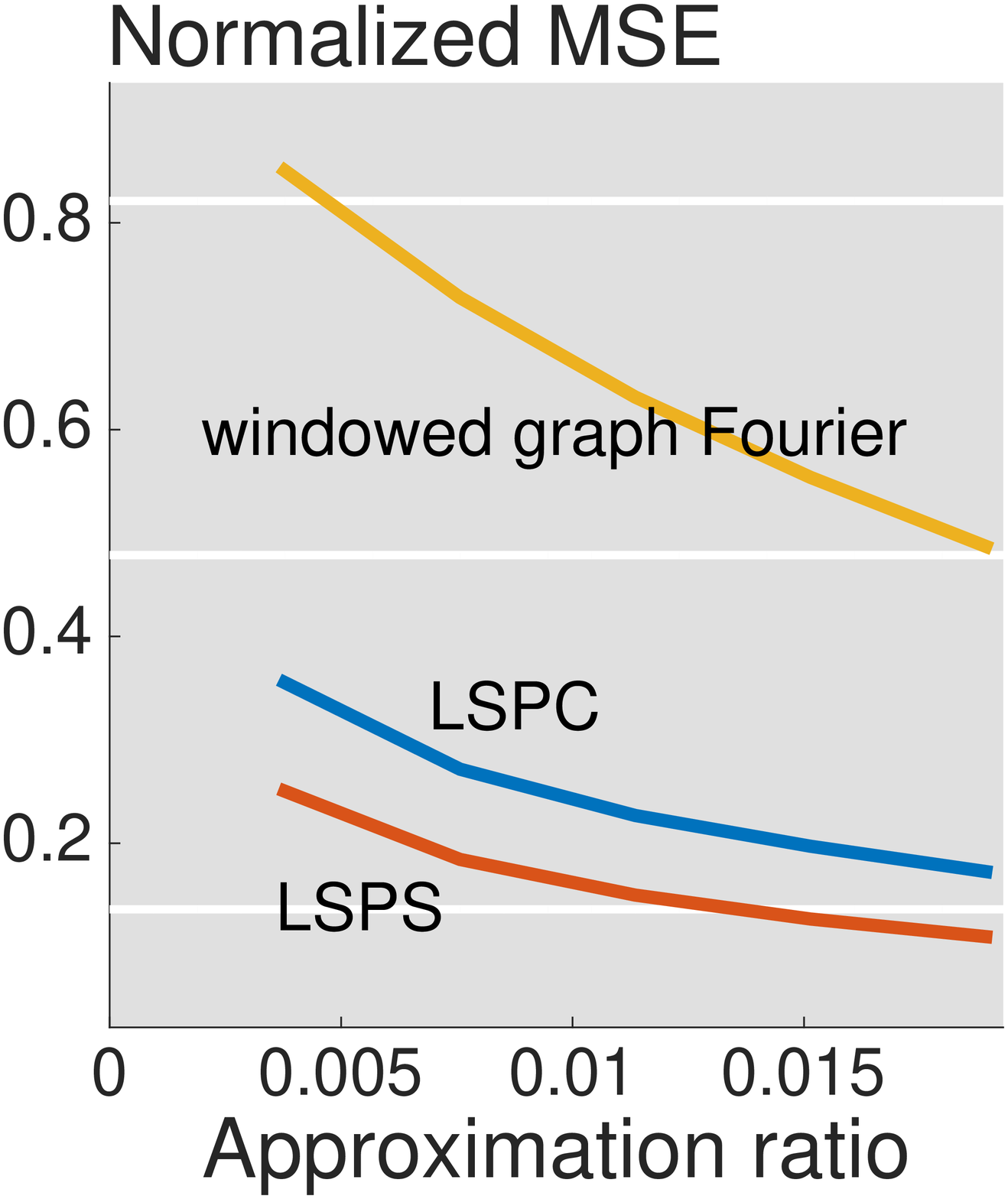} &
 \includegraphics[width=0.4\columnwidth]{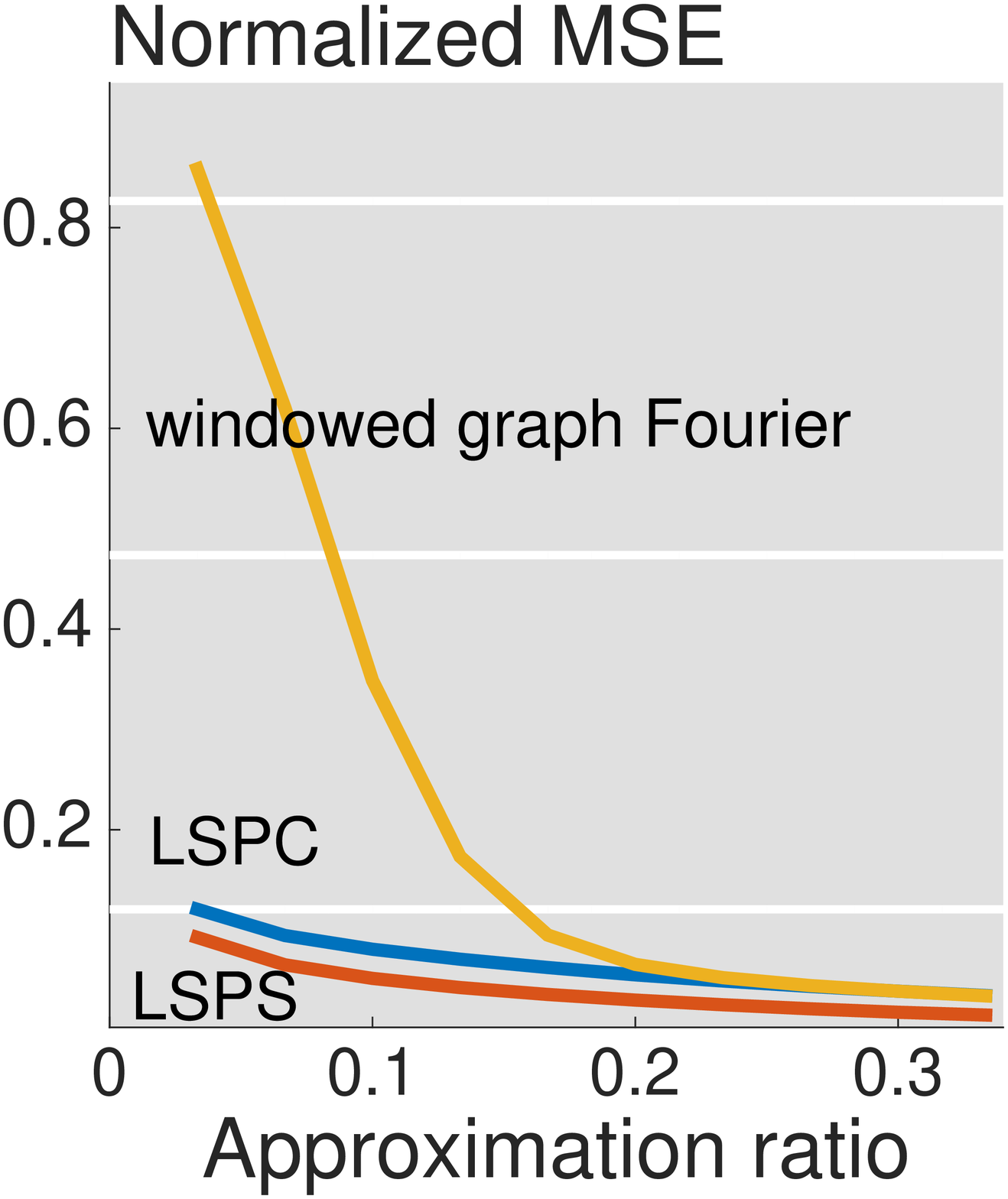}
\\
 {\small (a) Minnesota.} & {\small (b) U.S weather station.} 
\end{tabular}
  \end{center}
   \caption{\label{fig:Approximation} Approximation error. Approximation ratio is the percentage of used coefficients ( $s$ in~\eqref{eq:sparse_coding_pc}). }
\end{figure}

\subsubsection{Sampling and Recovery}
The sampling and recovery techniques of piecewise-smooth graph signals are similar to those of smooth graph signals. The basic idea is to assume the energy of a graph signal concentrates on a few expansion coefficients. For smooth graph signals, it is clear that the energy concentrates on the low-pass frequency band; however, for piecewise-smooth graph signals, we do not have any prior knowledge. We leave the  blind recovery as the future work. Some related works about the blind recovery are discussed in~\cite{SegarraMLR:15, VarmaCK:15}.

\subsection{Case Study: Environmental Change Detection}
This case study is motivated to providing an intuitive and reliable approach to detect environmental change using sparse coding. The detected change has many potential applications, such as offering suggestion to the authorities, or monitoring the variation trend in a local or global view.  Besides, most environmental data is collected by sensors, detecting anomaly records can give prior knowledge on the status of sensors and functioned as a pre-processing procedure for subsequent sensor data analysis. 
 
We study the daily distribution of \ce{SO2} during ${\rm 2014}$ in mainland U.S., the underlying graph is composed by ${\rm 444}$ operating sensors, which record the daily average at various locations. Since the adjacent sensors have similar records, we model the daily graph signal as a piecewise smooth signal, and build local-set-based piecewise-smooth dictionary ($ \D_{\rm LSPS}$), introduced in \ref{ssec:Graph_D} to represent it. In the experiment $ \D_{\rm LSPS}$ is considered as a piecewise polynomial dictionary. 

In this case study, our task is that for arbitrary successive two days, we use the proposed dictionary to detect the areas that whose recorded data changed the most saliently between the two days. The intuition behind this is that the difference of the recorded \ce{SO2} will have bigger magnitude in the targeted areas compared with their neighborhood, and the proposed multi-resolution dictionary can zoom into that certain areas so to detect them. Mathematically, the areas are encoded in the activated atoms corresponding to the top sparse coding coefficient in terms of magnitude. We use the matching pursuit to get the sparse coding $\a$ of $\x$ with respect to the built dictionary. 

In Figure~\ref{fig:detectChange}, we aim to illustrate the benefits of the proposed graph dictionary on detecting the very area whose recorded data changed the most from ${\rm May\ 26}$ to ${\rm May\ 27}$. Figure~\ref{fig:detectChange} (a) (b) show the snapshot of \ce{SO2} distributed on ${\rm May\ 26}$ and ${\rm May\ 27}$, (c) illustrates the snapshot of data records difference between these two days. (d), (e) show the detected area by the activated atom corresponding to the top $1$ and $6$ sparse coding coefficient, and they contain $57$ and $62$ nodes, respectively. As comparison, in (f) we rank the records difference in terms of magnitude and highlight the top $57$ nodes. Since (f) is generated without taking the graph structure into account, the highlighted nodes spread out; while the graph dictionary has the graph information built in, the highlighted nodes in (d) and (e) cluster into group, and give us a better geographical knowledge on which part of U.S. having \ce{SO2} distribution changes the most during these two days. In this experiment, $ \D_{\rm LSPS}$ is built by 2-means graph partition algorithm, other graph partition algorithms, spectral clustering and spanning tree, can also build graph dictionary that are good at detect records data change, and can provide better geographical knowledge than just simply rank the difference coefficient.
\begin{figure}[htb]
  \begin{center}
    \begin{tabular}{cc}
 \includegraphics[width=0.4\columnwidth]{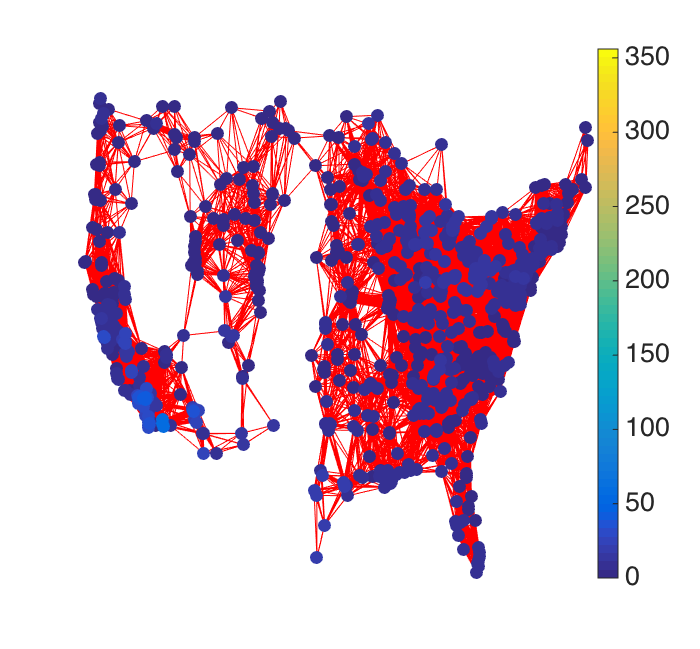}  & \includegraphics[width=0.4\columnwidth]{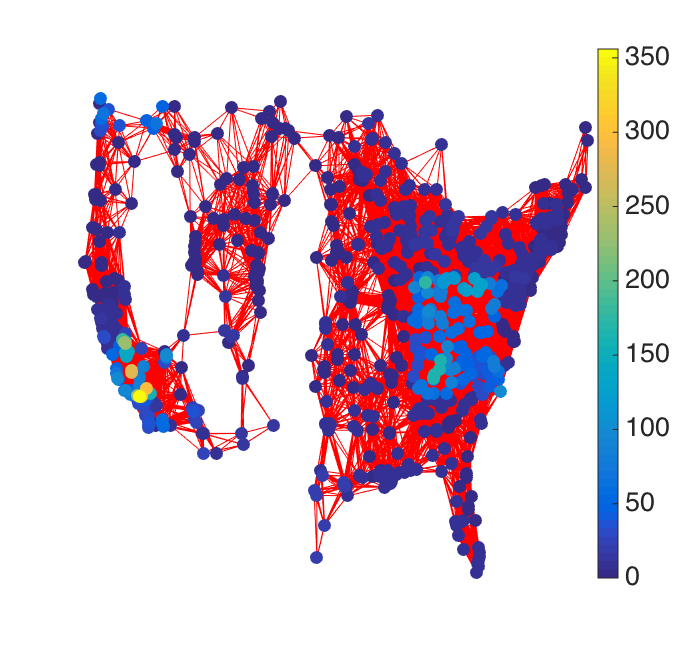}
\\
 {\small (a) Snapshot of \ce{PM_{2.5}}} & {\small (b) Snapshot of corrupted} 
 \\
 { on ${\rm Jan.\ 12}$.} & { \ce{PM_{2.5}} on ${\rm Jan.\ 12}$.} 
\\
\includegraphics[width=0.4\columnwidth]{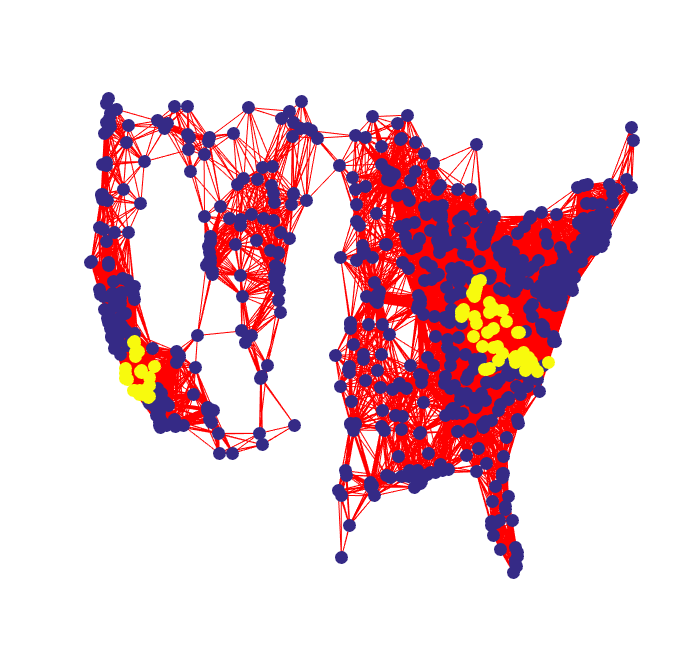}  & \includegraphics[width=0.4\columnwidth]{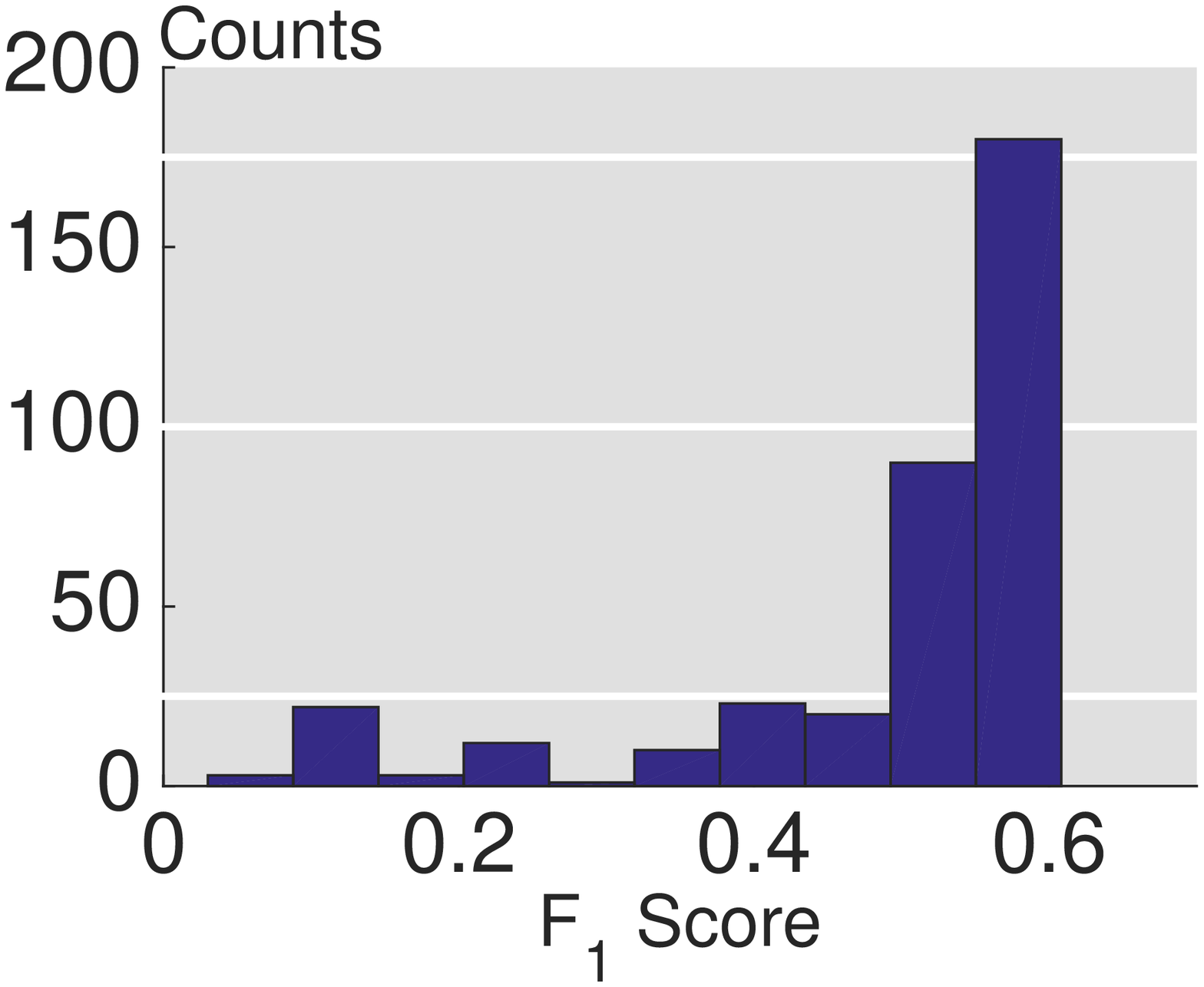}
\\
 {\small (c)  Detected area by atoms} & {\small (d) Histogram of $\rm F_1$ Score} 
 \\
 {corresponding to top $2$ $|{\a}_i|$} & {of detected areas} 
  \\
  {using spectral clustering.}&{using spectral clustering.}
  \\
 \includegraphics[width=0.4\columnwidth]{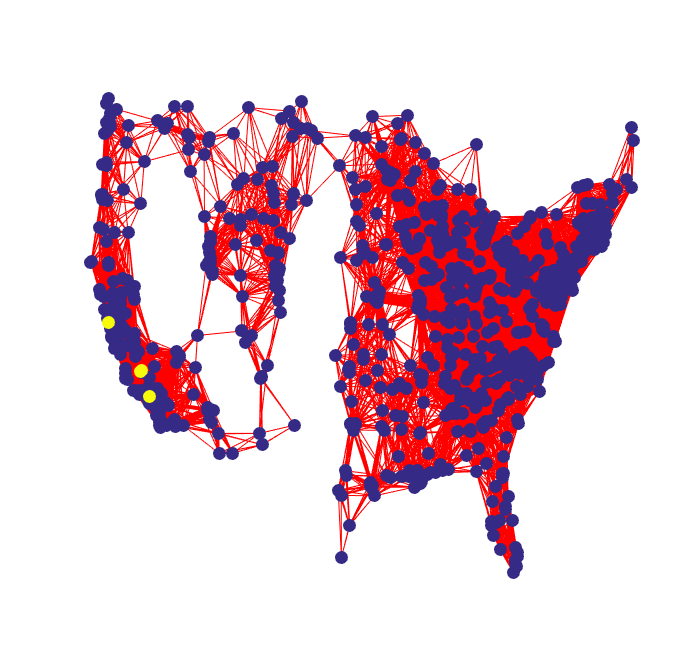}  & \includegraphics[width=0.4\columnwidth]{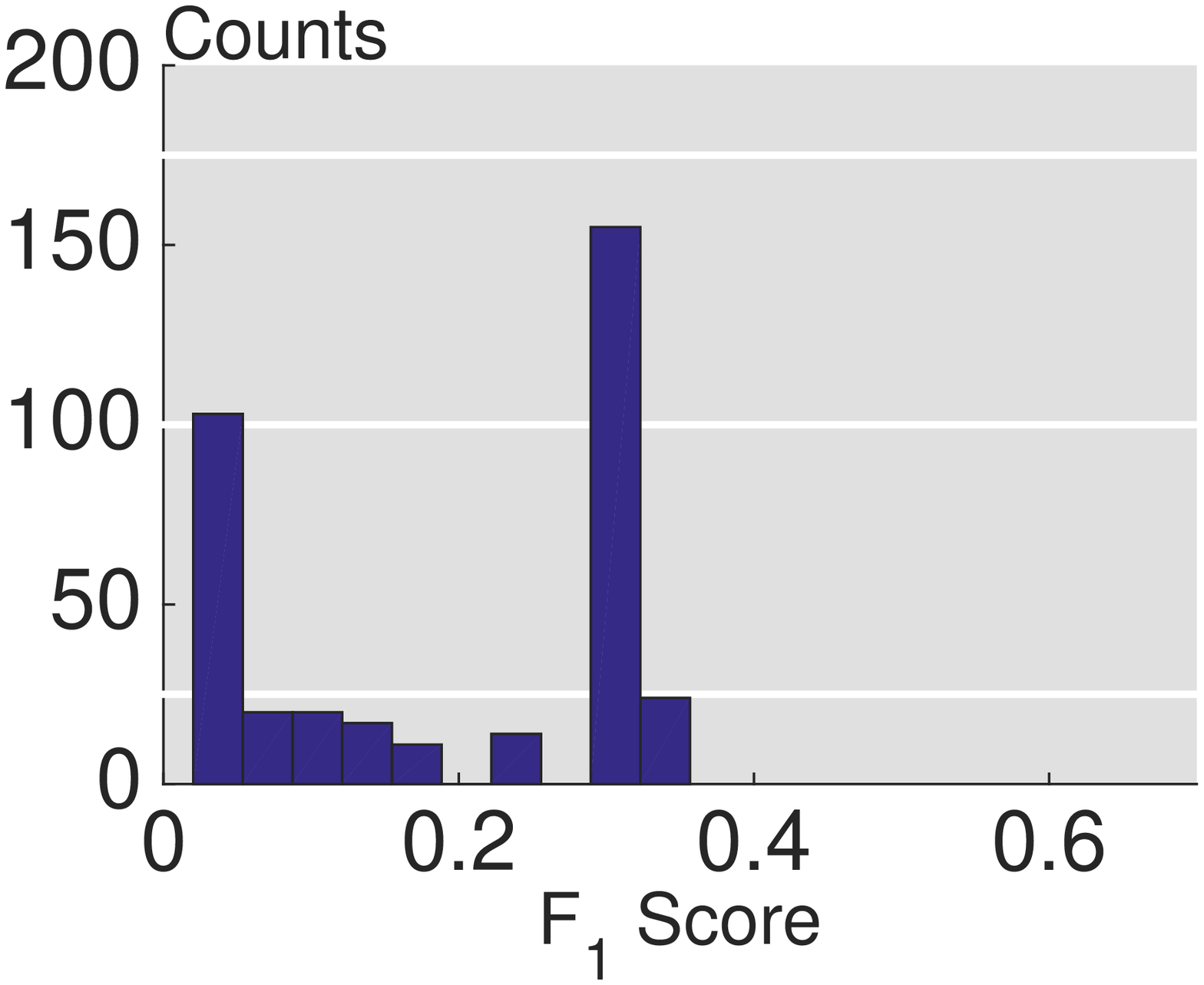}
 \\
 {\small (e) Detected area by atoms} & {\small (f)  Histogram of $\rm F_1$ Score} 
 \\
  {corresponding to top $2$ $|{\a}_i|$} & {of detected areas} 
  \\
  {using MST.} & {using MST.}
  \\
  \includegraphics[width=0.4\columnwidth]{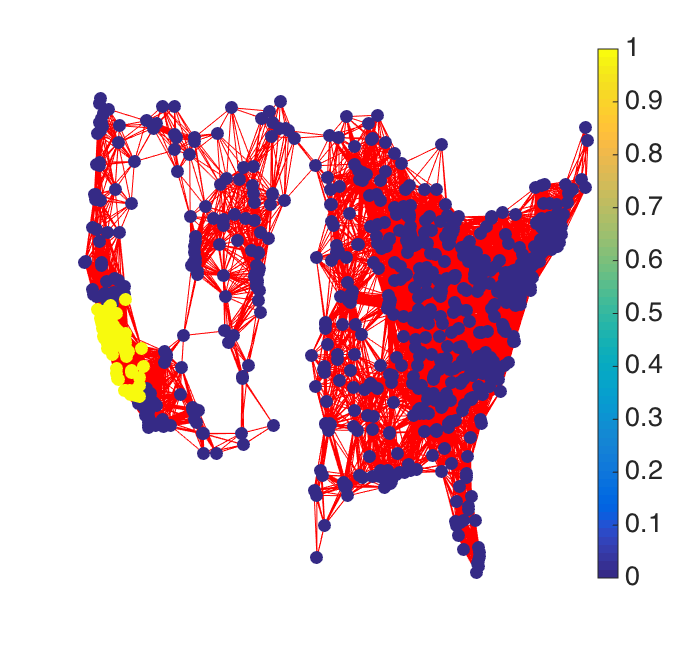}  & \includegraphics[width=0.4\columnwidth]{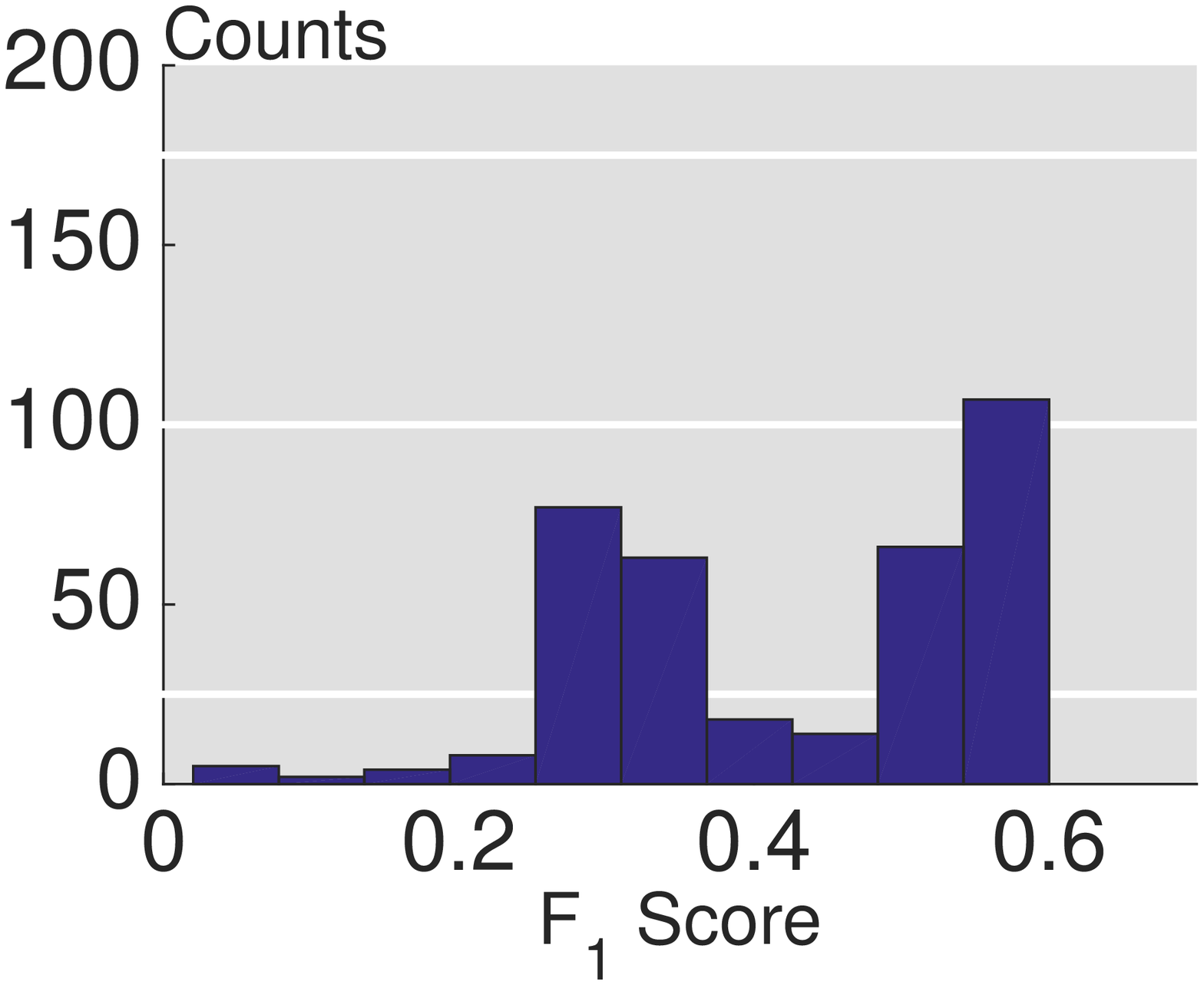}
 \\
 {\small (g) Detected areas by atoms} & {\small (h)  Histogram of $\rm F_1$ Score} 
 \\
  {corresponding to top $2$ $|{\a}_i|$.} & {of detected areas}
  \\
  {using 2-means.}&{using 2-means.}
\end{tabular}
  \end{center}
   \caption{\label{fig:detectChange} Detecting area with most salient data change for ${\rm May\ 26}$ and ${\rm May\ 27}$ by using $ \D_{\rm LSPS}$ built by 2-means graph partition algorithms.}
\end{figure}

\section{Conclusions}
\label{sec:conclusions}
Graph signal representation has been considered in many previous literature~\cite{CrovellaK:03, HammondVG:11, ShumanNFOV:13, ThanouSF:14}. There are mainly two approaches to design a representation for graph signals: one is based on the graph Fourier domain and the other one is based on the graph vertex domain. 

The representations based on the graph Fourier domain are based on the spectral properties of the graph. The most fundamental representation based on the graph Fourier domain is the graph Fourier transform, which is the eigenvectors of a matrix that represents a graph structure~\cite{ShumanNFOV:13, SandryhailaM:13}. Based on the graph Fourier transform, people propose various versions of multiresolution transforms on graphs, including diffusion wavelets~\cite{CoifmanM:06},  spectral graph wavelets~\cite{HammondVG:11}, 
graph quadrature mirror filter banks~\cite{NarangO:12}, windowed graph Fourier transform~\cite{ShumanRV:15}, polynomial graph dictionary~\cite{ThanouSF:14}. The main idea is to construct a series of graph filters on the graph Fourier domain, which are localized on both the vertex and graph Fourier domains. The advantages of the representations on the graph Fourier domain are: first, it avoids the complex and irregular connectivity on the graph vertex domain because each frequency is independent; second, it is efficient, because the construction is simply to determine a series of filter coefficients, where the computation is often accelerated by the polynomial approximation; third, it is similar to the design of the classical wavelets. However, there are two shortcomings: first, it loses the discrete nature of a graph. That is, the construction is not directly based on the graph frequencies; instead, it proposes a continuous kernel and then we samples the values from the continuous kernel;  second, the locality on the graph vertex domain is worse than the representations based on the graph vertex domain. It is true that this construction provides the better locality on the graph Fourier domain. However, the locality on the graph Fourier domain is vague, abstract, and is often less important in most real-world applications.

The representations based on the graph vertex domain are based on the connectivity properties of the graph. The advantages are: first, it provides better locality on the graph vertex domain and is easier to visualize; second, it provides a better understanding on the connectivity of a graph, which is avoided by the graph Fourier transform for the representations on the graph Fourier domain. Some examples of the representations include multiscale wavelets on trees~\cite{GavishNC:10}, graph wavelets for spatial analysis~\cite{CrovellaK:03}, spanning tree wavelet basis~\cite{SharpnackKS:13}.

\begin{itemize}
\item Spanning tree wavelet basis proposes a localized basis on a spanning tree. The proposed local-set-based piecewise-constant wavelet basis is mainly inspired from this work and the proposed representations generalize the results by using more general graph partition algorithms;

\item Multiscale wavelets on trees provides a hierarchy tree representation for a dataset. It proposes a wavelet-like orthonormal basis based on a balanced binary tree, which is similar to the proposed local-set-based piecewise-constant wavelet basis. The previous work focuses on high dimensional data and constructs a decomposition tree bottom up; the proposed local-set representations focus on a graph structure and and construct a decomposition tree top down, which is
useful for capturing clusters;

\item Graph wavelets for spatial traffic analysis proposes a general wavelet representation on graphs. The wavelet basis vectors are not generated from a single function, that is, the wavelet coefficients at different scales and locations are different; the proposed representations resemble the Haar wavelet basis in spirit and are generated from a single indicator function.
\end{itemize}

In this paper, we proposed a unified framework for representing and modeling data on graphs. Based on this framework, we study three typical classes of graph signals: smooth graph signals, piecewise-constant graph signals, and piecewise-smooth graph signals. For each class, we provide an explicit definition of the graph signals and construct a corresponding graph dictionary with desirable properties. We then study how such graph dictionary works in two standard tasks: approximation and sampling followed with recovery, both from theoretical as well as algorithmic perspectives. Finally, for each class, we present a case study of a real-world problem by using the proposed methodology.

\bibliographystyle{IEEEbib}
\bibliography{bibl_jelena}

\end{document}